\documentclass[journal]{IEEEtran}
\usepackage{amsmath,amsfonts}
\usepackage{amssymb}
\usepackage{algorithmic}
\usepackage{algorithm}
\usepackage{array}
\usepackage[caption=false,font=normalsize,labelfont=sf,textfont=sf]{subfig}
\usepackage{textcomp}
\usepackage{stfloats}
\usepackage{url}
\usepackage{verbatim}
\usepackage{graphicx}
\usepackage{bm}
\usepackage{booktabs} 
\usepackage{hyperref}
\usepackage{color}
\newtheorem{theorem}{Theorem}[section]
\newtheorem{lemma}{Lemma}[section]
\newtheorem{definition}{Definition}[section]
\newtheorem{proposition}{Proposition}[section]
\newtheorem{remark}{Remark}[section]
\newtheorem{corollary}{Corollary}[section]

\newtheorem{assumption}{Assumption}[section]
\begin{document}


\title{Multi-Source Wasserstein Distributionally Robust Graph Learning}        

\author{Chuansen Peng, Yifan Xia, Jinshan Zhong, and Xiaojing Shen
\thanks{The work was supported by the Sichuan Provincial Natural Science Foundation Project (Innovative Research Group) 2026NSFSCZY0056. \textit{(Corresponding author: Xiaojing Shen.)}\\
\indent Chuansen Peng, Jinshan Zhong and Xiaojing Shen are with School of Mathematics, Sichuan University, Chengdu, Sichuan, 610064, China. (e-mail: pengchuansen@stu.scu.edu.cn; 2023322010061@stu.scu.edu.cn; shenxj@scu.edu.cn). \\
\indent Yifan Xia is with School of Statistics and Data Science, Southwestern University of Finance and Economics, Chengdu, Sichuan, 611130, China. (e-mail: xiayifan@swufe.edu.cn).}}



\maketitle

\begin{abstract}
Reconstructing complex network topologies from data is a fundamental challenge in cybernetics and  graph
signal processing, with applications in brain connectivity, sensor
networks, and social networks. In practice, target-domain samples are
scarce, while heterogeneous source-domain data are abundant. Fusing
these sources into a reliable nominal distribution is challenging:
Euclidean averaging works when sources are homogeneous but degrades
sharply as inter-source divergence grows, collapsing distinct
geometries into an inflated, biased consensus. We exploit the
Wasserstein metric's distribution-preserving fusion to counter this
heterogeneity while preserving each source's intrinsic geometry. We
propose MS-WDRO, a multi-source Wasserstein distributionally robust
graph learning framework that fuses heterogeneous sources via their
weighted Wasserstein barycenter, a geometrically principled nominal
distribution, then builds an ambiguity ball around it to hedge
residual uncertainty. Minimizing worst-case risk over this ball yields
a tractable regularized Laplacian estimator, solved efficiently via a
provably convergent ADMM scheme. We establish non-asymptotic
guarantees: a finite-sample concentration bound for the empirical
barycenter, a pooling bias lower bound proving naive aggregation is
suboptimal, and an out-of-sample excess risk bound decaying at a
parametric rate with only logarithmic dependence on source count. To
calibrate the four coupled hyperparameters governing robustness,
sparsity, and source fusion, we unroll the solver into a
differentiable architecture trained end-to-end, achieving
data-adaptive calibration beyond cross-validation while retaining
interpretability. Experiments on synthetic benchmarks and the
multi-site ABIDE~I neuroimaging dataset show MS-WDRO consistently
outperforms seven baselines in graph recovery accuracy, sample
efficiency, and downstream diagnostic utility, with the largest gains
in the sample-scarce regime motivating this work.
\end{abstract}

\begin{IEEEkeywords}
Graph signal processing, network topology inference, distributionally robust optimization, Wasserstein barycenter, algorithm unrolling.
heterogeneous multi-source learning
\end{IEEEkeywords}

\section{Introduction}
\IEEEPARstart{C}{omplex} network-structured data pervade modern science and
engineering: gene regulatory networks govern cellular function,
functional brain connectomes encode cognitive processes, power grids
exchange energy across interconnected nodes, and social platforms
generate opinion dynamics over relational graphs~\cite{zhang2019joint,song2025multisensor,qian2025survey,joseph2021controllability,lee2025sfgcn,ma2018reliable,wu2020online}.  Graph signal
processing (GSP) has emerged as a principled mathematical framework
for analyzing signals defined on such irregular domains, extending
classical tools from Fourier analysis, filtering, and sampling theory
to data indexed by graph vertices~\cite{shuman2013emerging,ortega2018graph,li2023unscented}.  A central and often foundational challenge in this
paradigm is \emph{network topology inference}, the problem of
recovering the underlying graph structure from observed nodal
signals.  The inferred topology is not
merely an end in itself; it serves as the substrate for graph-based
filtering, signal interpolation, anomaly detection, and disease
biomarker discovery in neuroimaging studies of conditions such as
autism spectrum disorder.  In all of these
settings, the quality of downstream analysis depends critically on the
fidelity of the inferred graph, placing a premium on topology
estimators that are both statistically principled and computationally
tractable.
 
 
The vast majority of graph learning methods are built on a generative
signal model that captures the relationship between the observed signals
and the latent graph.  Under the widely adopted \emph{smooth graph
signal model}, nodal observations are assumed to vary slowly across
connected vertices, and the combinatorial graph Laplacian naturally
emerges as the precision matrix of the corresponding degenerate
Gaussian distribution.  Building on this model, Dong
et al.~\cite{dong2016learning} formulated graph learning as a
constrained maximum-likelihood estimation problem and derived efficient
block coordinate descent algorithms.  Kalofolias~\cite{kalofolias2016learn}
recast a closely related objective as a smooth convex program amenable
to proximal gradient solvers, while Egilmez
et al.~\cite{egilmez2017graph} incorporated structural constraints
on the Laplacian and developed specialized ADMM procedures with
convergence guarantees.  Several subsequent works extended these
foundations: Kumar et al.~\cite{kumar2020unified} unified a broad
class of structured graph learning problems through spectral
constraints; Sardellitti et al.~\cite{sardellitti2019graph} leveraged
transform learning to infer graph topology without specifying a
parametric signal model; and accelerated first-order methods have been
developed to handle large-scale instances~\cite{saboksayr2021accelerated}.  Beyond the smoothness paradigm, complementary
lines of work infer topology from stationary graph signals whose power
spectrum is aligned with the graph Fourier
basis~\cite{marques2017stationary},
from diffusion-process observations~\cite{thanou2017learning}, from
spectral templates derived from signal second-order
statistics~\cite{segarra2017network}, and from joint network topology inference across
multiple related graphs~\cite{yuan2023joint}.  In the probabilistic
graphical model literature, sparse precision matrix
estimation~\cite{friedman2008sparse} provides an alternative route to
graph structure recovery when Gaussian assumptions hold.  Despite their
considerable diversity in signal model and algorithmic strategy, these
approaches share a common requirement: they rely on a sufficiently
large and representative collection of signal observations drawn from
the \emph{same} target distribution.  Their statistical performance
degrades severely when samples are scarce or when the observations do
not faithfully reflect the distribution of interest.
 
 
In practice, graph signal observations are finite in number,
contaminated by noise, and subject to distributional shift arising from
sensor drift, protocol variation, or environmental change. Under these
conditions, the empirical risk minimization (ERM) principle may lead to overfitting: the learned graph
Laplacian is calibrated to the finite observed samples but generalizes
poorly to signals drawn from the true underlying distribution \cite{duchi2021learning}.
Distributionally robust optimization (DRO) offers a principled remedy
by replacing the nominal empirical distribution with a worst-case
distribution drawn from an ambiguity set, thereby immunizing the
estimator against deviations between the training and true
distributions~\cite{delage2010distributionally, wiesemann2014distributionally}.
Among the many choices of ambiguity set, those defined through the
Wasserstein metric have attracted particular attention owing to their
geometric interpretability, their tight connection to adversarial
perturbations, and their ability to yield tractable convex
reformulations of otherwise intractable minimax
problems~\cite{esfahani2018data, blanchet2019quantifying,
gao2023distributionally}.  A Wasserstein ball of prescribed radius
centered at the empirical distribution certifies that the learned
estimator performs well for any distribution that lies within a
controlled optimal transport distance from the training
data, a guarantee that purely moment-based ambiguity
sets~\cite{delage2010distributionally, wiesemann2014distributionally}
cannot provide in the absence of shape assumptions.  The foundations
of Wasserstein DRO, duality theory, tractable reformulations, and
finite-sample performance guarantees, have been developed in
considerable depth~\cite{esfahani2018data}, and the framework has been applied to a broad array
of statistical learning problems including logistic
regression~\cite{shafieezadeh2015distributionally} and adversarially
robust training~\cite{sinha2018certifying}.  Most recently, Wasserstein
DRO has been specialized to graph Laplacian estimation from smooth
signals~\cite{zhang2025wasserstein}, demonstrating that distributional
robustness yields measurable improvements in graph recovery accuracy
when target-domain samples are limited. These developments concentrate on the single-source
setting, in which all observations are assumed to arise from one
underlying distribution. Motivated by this body of work, the present
paper extends Wasserstein distributional robustness from single source to the multi-source
regime, in which signals collected across several related yet
heterogeneous domains must be reconciled into a single robust
estimator. This extension raises challenges of its own: heterogeneous
sources cannot be pooled directly without incurring a mixing bias, and
the ambiguity set itself must be constructed jointly from multiple
empirical distributions rather than centered on a single one, calling
for a principled mechanism to aggregate heterogeneous source
information before robustness can be meaningfully enforced.
 
 
Many high-impact applications of graph learning feature a target domain
from which signal samples are difficult or impossible to collect in
sufficient quantity, while multiple related but heterogeneous source
domains provide abundant data.  In functional neuroimaging, a rare
clinical cohort at a small acquisition site, such as the CMU site in
the ABIDE~I consortium~\cite{dimartino2014autism}, constitutes the
target domain, while data from larger studies conducted at different
scanning sites under distinct imaging protocols serve as source domains,
inducing site-specific distributional shift that is well documented in
the brain connectivity literature~\cite{bullmore2009complex}.  Similar
multi-source configurations arise in distributed sensor networks, where
sensors with different hardware characteristics monitor a shared
physical phenomenon; in clinical federated learning, where data
heterogeneity across hospitals precludes simple aggregation; and in
social network analysis, where platform-specific behavioral norms
differentiate the signal statistics across source and target domains.
The transfer learning literature has long recognized that naive
aggregation of heterogeneous source data by pooling all observations
into a single empirical distribution introduces a \emph{mixing
bias}~\cite{pan2010survey, ben2010theory} whose magnitude grows with
inter-source divergence, and that constructing a shared model from
such a biased nominal distribution leads to systematically degraded
target-domain performance~\cite{mohri2019agnostic}.  A principled
approach to multi-source distribution aggregation is provided by
\emph{optimal transport} theory~\cite{villani2009optimal}: the \emph{Wasserstein barycenter}~\cite{agueh2011barycenters}
defines the Fr\'{e}chet mean of a family of distributions in the
Wasserstein metric space and produces a consensus representation that
preserves the intrinsic geometric structure of each source, in sharp
contrast to the covariance inflation inherent to the mixture
distribution.  Efficient algorithms for Wasserstein barycenter
computation~\cite{cuturi2014fast, alvarez2016fixed} and the Gelbrich
lower bound relating the barycenter to source second-order
moments~\cite{gelbrich1990formula} have rendered this construction
computationally tractable at the scale of practical graph learning
problems.  Optimal transport has also proved fruitful for domain
adaptation by geometrically aligning source and target
distributions~\cite{courty2017optimal}; the present work leverages its
complementary role of fusing source distributions into a structurally
sound nominal distribution for robust graph estimation. The specific idea of centering a Wasserstein ambiguity set at a
barycenter of multiple source-domain empirical distributions was
introduced by Lau and Liu~\cite{lau2022wasserstein}, who proposed
\emph{Wasserstein Barycentric Distributionally Robust Optimization}
(WBDRO) as a general aggregate data-driven decision-making framework
for heterogeneous multi-source data. They define the barycentric
ambiguity set for a generic $p$-Wasserstein metric, establish
finite-sample concentration and asymptotic consistency guarantees for
the induced worst-case risk in both the large-sample and the
many-source regimes, and, under a Gelbrich/Bures--Wasserstein
specialization to location-scatter families, instantiate the framework
on distributionally robust sparse precision-matrix estimation for
zero-mean Gaussian random vectors. The present paper specializes and
substantially develops this general construction for the specific
problem of network topology inference from smooth graph signals.
Beyond adapting the barycentric ambiguity set
of~\cite{lau2022wasserstein} to graph Laplacian estimation
(Section~\ref{subsec:barycentric_wdro}), our contribution includes a
closed-form treatment of the Bures--Wasserstein barycenter for the
degenerate Gaussian covariances induced by graph Laplacians sharing
the common null space $\mathrm{span}\{\mathbf{1}\}$
(Proposition~\ref{prop1}); a tractable ADMM reformulation with
rigorously established global convergence and an ergodic $O(1/K)$
primal--dual rate (Theorem~\ref{thm:admm_convergence}); a formal,
non-asymptotic characterization of the irreducible mixing bias
incurred by na\"ive pooling relative to the barycentric construction,
specific to Gaussian graph-signal sources
(Theorem~\ref{thm:pooling_bias}, Proposition~\ref{prop:cov_pooling_bias});
and a differentiable algorithm-unrolling architecture for jointly
calibrating the ambiguity radius, sparsity level, ADMM penalty, and
barycentric fusion weights (Section~\ref{sec:unrolling}), none of
which are addressed by the original WBDRO framework.
 
 
Despite clear practical motivation, the intersection of multi-source
learning, Wasserstein distributional robustness, and smooth-signal
graph topology inference remains largely unexplored. A further challenge is the joint
calibration of multiple interacting hyperparameters inherent in any
WDRO-based graph learning framework: the ambiguity set radius
determines the degree of distributional robustness; the sparsity
regularization coefficient controls graph density; the augmented
Lagrangian penalty governs algorithmic convergence; and the barycentric
fusion weights modulate the relative contribution of each source domain.
These parameters interact nonlinearly and cannot be calibrated in
isolation.  Conventional cross-validation over the resulting joint
parameter space is computationally prohibitive in the multi-source
regime, and concentration-inequality-based radius selection, while
theoretically motivated, is inherently conservative and ignores the
mutual coupling among the parameters.  A principled, end-to-end
trainable framework that simultaneously addresses multi-source fusion,
distributional robustness, and automatic hyperparameter calibration is
therefore strongly needed.
 
 
Algorithm unrolling, the technique of mapping the iterations of an
optimization algorithm onto the layers of a neural network while
treating algorithmic parameters as learnable
variables~\cite{monga2021algorithm}, offers a
compelling solution to the hyperparameter calibration challenge.  By
embedding the iterative solver into a differentiable architecture and
training end-to-end on a supervised corpus of graph-signal pairs, all
parameters can be jointly calibrated through backpropagation, capturing
their mutual interactions in a manner that any sequential or grid-search
strategy fundamentally misses.  The resulting unrolled network retains
full algorithmic interpretability, each layer executes a precisely
specified proximal, eigendecomposition, or dual-ascent substep, while
acquiring data-adaptive calibration capability unavailable to the
fixed-parameter solver.  This paradigm has been successfully applied to
compressive sensing magnetic resonance imaging via deep
ADMM-Net~\cite{yang2016deep}, to sparse graph recovery via
GLAD~\cite{shrivastavaglad}, and to graph signal
denoising~\cite{chen2021graph}.  The survey
by Monga et al.~\cite{monga2021algorithm} provides a comprehensive
account of the theoretical and practical principles underlying the
paradigm, documenting its advantages in interpretability, sample
efficiency, and computational scalability relative to generic deep
architectures.  A critical prerequisite enabling differentiable
unrolling in our setting is the ADMM solver structure,
whose per-iteration updates admit closed-form analytical expressions
that are differentiable with respect to the hyperparameters and
therefore compatible with gradient-based end-to-end training.
 
 
This paper proposes MS-WDRO (\underline{M}ulti-\underline{S}ource
\underline{W}assers-\allowbreak tein \underline{D}istributionally \underline{R}obust
\underline{O}ptimization for graph learning), a framework that
jointly addresses multi-source data fusion, distributional robustness,
and automatic hyperparameter learning for network topology inference
from smooth graph signals.  The main contributions are as follows.
 
\emph{(i) Multi-source WDRO framework.}  We formulate a novel graph
Laplacian estimation framework that fuses heterogeneous source
distributions by computing their weighted Wasserstein barycenter as the
nominal distribution of a Wasserstein ambiguity set.  Minimizing the
worst-case expected log-likelihood loss over the resulting barycentric
ambiguity ball yields a principled minimax problem that simultaneously
exploits multi-source information and guards against residual
distributional uncertainty.  The framework encompasses single-source
WDRO as a special case and formally subsumes naive pooling while
provably outperforming it.
 

\emph{(ii) Tractable reformulation and efficient algorithm.}  We derive
a closed-form tractable reformulation of the minimax problem via
Wasserstein strong duality, reducing it to a regularized Laplacian
estimation problem with an explicit Frobenius-norm robustness penalty
and sparsity term.  A two-block ADMM solver with closed-form
per-iteration updates is developed, and its global convergence, objective convergence, primal consensus, and dual convergence,
governed by a provably monotone primal-dual Lyapunov potential, is rigorously established in Theorem~\ref{thm:admm_convergence}. Explicit non-asymptotic guarantees are given, including an ergodic $O(1/K)$ primal-dual gap-function rate (Corollary~\ref{cor:rate}).
 
\emph{(iii) Rigorous statistical theory.}  We establish three
complementary non-asymptotic guarantees.  First, a finite-sample
concentration bound for the empirical Wasserstein barycenter quantifies
how accurately the empirical barycenter approximates its population
counterpart as source sample sizes grow. Second, a formal pooling bias lower bound, proved via the Lieb concavity
theorem for quantum fidelity~\cite{lieb1973convex} and the Haar measure
twirling identity~\cite{mele2024introduction}, confirms that, under a
common-covariance, heterogeneous-mean Gaussian source model, naive data
pooling incurs a strictly positive, irreducible Wasserstein estimation
bias that scales with the between-source mean dispersion
$\mathcal{H}_{\bm\lambda}$, establishing an asymptotic separation
between the barycentric and pooling approaches in this regime. Third, an out-of-sample excess risk bound derived through Rademacher
complexity analysis shows that the excess risk of the MS-WDRO estimator
decays at the parametric rate and depends on the number of sources only
logarithmically, with a concentration correction computable entirely
from training data.
 
\emph{(iv) Algorithm unrolling for joint hyperparameter learning.}  We
embed the MS-WDRO solver into a differentiable multi-layer architecture by
interpreting each iteration as a computational layer parameterized by
four learnable quantities: the ambiguity set radius, the sparsity
regularization coefficient, the augmented Lagrangian penalty, and the
barycentric fusion weights.  End-to-end supervised training jointly
calibrates all parameters through backpropagation, enabling an implicit
annealing schedule across layers, progressively shrinking robustness
margins, relaxing sparsity, tightening the dual step size, and adapting
source weights, that is difficult to realize efficiently by conventional cross-validation.
The total number of trainable parameters grows only linearly in network
depth and source count, making the architecture highly parameter-efficient.
 
\emph{(v) Comprehensive experimental validation.}  We conduct
experiments on synthetic graph benchmarks with controlled inter-source
heterogeneity and on the ABIDE~I multi-site functional neuroimaging
dataset, comparing against seven competitive baselines spanning
classical optimization-based methods, deep graph learning algorithms,
and the state-of-the-art single-source WDRO approach.  MS-WDRO
consistently and significantly outperforms all baselines in graph
recovery accuracy and sample efficiency across the full range of target
sample sizes, achieves superior downstream diagnostic classification on
the neuroimaging dataset, and empirically validates the theoretical
out-of-sample bound at the predicted convergence rate.
 
 
The remainder of the paper is organized as follows.
Section~\ref{sec:model} defines the graph signal model, reviews the
Wasserstein distance and barycenter, and states the baseline
Laplacian estimation problem.  Section~\ref{sec:wdro} presents the
Wasserstein barycentric ambiguity set, derives the tractable minimax
reformulation, and develops the ADMM solver.
Section~\ref{sec:enhanced_theory} provides the complete theoretical
analysis, including barycenter concentration, the pooling bias lower
bound, data-driven radius selection, and the out-of-sample excess risk
guarantee.  Section~\ref{sec:unrolling} introduces the algorithm
unrolling architecture and its end-to-end training procedure.
Numerical experiments are reported in Section~\ref{sec:experiments},
and Section~\ref{sec:conclusion} concludes with a discussion of
limitations and future directions.

\section{Signal Models and Problem Formulation}
\label{sec:model}
This section lays the groundwork for the multi-source distributionally
robust framework developed in the sequel.  We proceed in three steps.
We first specify the generative model that links observed graph signals
to the underlying topology, showing that smoothness on the graph induces
a degenerate Gaussian likelihood whose precision matrix coincides with
the graph Laplacian~(Section~\ref{subsec:model}).  Building on this
likelihood, we formulate the canonical single-distribution graph
learning problem as a regularized maximum-likelihood estimator
(Section~\ref{subsec:baseline_problem}), which serves as the point of
departure for the robust formulation introduced later.  Because that
baseline estimator implicitly treats the empirical distribution of a
single data source as exact, it provides no principled mechanism for
combining several heterogeneous sources or for hedging against
sampling error, both of which are central concerns of this paper.  We
therefore close the section by introducing the optimal-transport tools
required to overcome this limitation, namely the Wasserstein distance,
the Wasserstein barycenter, and its second-moment surrogate, the
Gelbrich distance (Section~\ref{subsec:wasserstein_prelim}); these
constructs furnish the geometric and computational vocabulary on which
the multi-source Wasserstein distributionally robust formulation of
Section~\ref{sec:wdro} is built.

\subsection{Graph Signal Model}
\label{subsec:model}

Let a weighted, undirected graph be represented by the triple
$\mathcal{G}=(\mathcal{V},\mathcal{E},\mathbf{L})$, where $\mathcal{V}$
is a finite vertex set of cardinality $N$, $\mathcal{E}\subseteq
\mathcal{V}\times\mathcal{V}$ is the edge set, and $\mathbf{L}\in
\mathbb{R}^{N\times N}$ is the combinatorial graph Laplacian, whose
off-diagonal entry $L_{ij}$ encodes (the negative of) the edge weight
between nodes $i$ and $j$. Under the standing assumption of no
self-loops, $\mathbf{L}$ satisfies $\mathbf{L}\mathbf{1}=\mathbf{0}$ (Here, $\mathbf{1}$ is the all-ones vector)
and is a symmetric positive semi-definite matrix \cite{yuan2023joint}. The set of
admissible Laplacians is accordingly given by
\begin{align}
    \mathcal{L}=\bigl\{\mathbf{L}\mid
    \mathbf{L}\succeq 0,\ \mathbf{L}\mathbf{1}=\mathbf{0},\
    L_{ij}=L_{ji}\leq 0,\ i\neq j\bigr\}.
    \label{equ.1}
\end{align}
A graph signal is a vector $\mathbf{x}=[x_1,\ldots,x_N]^T\in
\mathbb{R}^N$ assigning a scalar value $x_i$ to each node $i$. To relate
such signals to the graph topology, we adopt the standard graph filtering
formalism: writing the eigendecomposition of the Laplacian as
$\mathbf{L}=\mathbf{U}\bm{\Lambda}\mathbf{U}^T$, with $\bm{\Lambda}=
\mathrm{diag}(\lambda_1,\ldots,\lambda_N)$ collecting the non-decreasing
graph frequencies $0\leq\lambda_1\leq\cdots\leq\lambda_N$, the graph
Fourier transform of $\mathbf{x}$ is defined as $\hat{\mathbf{x}}=
\mathbf{U}^T\mathbf{x}$, with inverse transform $\mathbf{x}=
\mathbf{U}\hat{\mathbf{x}}$. Given the eigendecomposition $\mathbf{L}=\mathbf{U}\bm{\Lambda}\mathbf{U}^{T}$,
a graph filter with spectral response $h(\cdot)$, defined as
$h(\mathbf{L})=\mathbf{U}h(\bm{\Lambda})\mathbf{U}^{T}$, generates signals from
a white input $\mathbf{x}_0$ via
\begin{align}
    \mathbf{x}
    =\bm{\mu}+h(\mathbf{L})\mathbf{x}_0,
    \label{equ.2}
\end{align}
where $\bm{\mu}$ denotes the signal mean. Taking $\mathbf{x}_0$ to be
standard multivariate Gaussian noise renders $\mathbf{x}$ Gaussian with
covariance $ h^2(\mathbf{L})\triangleq h(\mathbf{L})h(\mathbf{L})^T$:
\begin{align}
    \mathbf{x}\sim\mathcal{N}\bigl(\bm{\mu},\,h^2(\mathbf{L})\bigr).
    \label{equ.3}
\end{align}
The choice of filter $h(\cdot)$ determines the qualitative behavior of
the resulting signal family. Following the widely adopted smoothness
prior for graph signals~\cite{dong2016learning,kalofolias2016learn}, we
select the low-pass filter $h(\mathbf{L})=\sqrt{\mathbf{L}^{\dagger}}$,
where $(\cdot)^\dagger$ denotes the Moore--Penrose pseudoinverse; this
choice attenuates high graph-frequency components and thereby favors
signals that vary slowly across connected vertices, consistent with the
smoothness assumption underlying most graph learning methods reviewed
in the Introduction. Substituting this filter into~\eqref{equ.3} shows
that a smooth graph signal follows a degenerate Gaussian law whose
precision matrix is precisely the graph Laplacian:
\begin{align}
    f(\mathbf{x})=\mathcal{N}\bigl(\bm{\mu},\,\mathbf{L}^{\dagger}\bigr).
    \label{equ.4}
\end{align}
Equation~\eqref{equ.4} is the key structural fact exploited throughout
the paper: it reduces the topological problem of estimating
$\mathbf{L}$ to the statistical problem of estimating the precision
matrix of a Gaussian graphical model, subject to the combinatorial
constraints in~\eqref{equ.1}.

\subsection{Baseline Graph Learning Problem}
\label{subsec:baseline_problem}

We first formulate the canonical estimation problem implied by the
signal model~\eqref{equ.4} under the classical assumption that all
observed signals are drawn i.i.d.\ from a single, fixed distribution.
Let $\mathbf{X}=[\mathbf{x}_1,\ldots,\mathbf{x}_n]\in\mathbb{R}^{N
\times n}$ collect $n$ such i.i.d.\ samples. The empirical
log-likelihood of a candidate Laplacian $\mathbf{L}$ given $\mathbf{X}$
follows directly from~\eqref{equ.4}:
\begin{align}
    \mathcal{F}_n(\mathbf{L}\mid\mathbf{X})
    =\frac{1}{n}\sum_{i=1}^{n}\log f(\mathbf{x}_i)
    =\log|\mathbf{L}|_{+}-\mathrm{tr}(\hat{\bm{\Sigma}}\mathbf{L}),
    \label{equ.7}
\end{align}
where $|\mathbf{L}|_{+}$ denotes the pseudo-determinant of $\mathbf{L}$
(the product of its nonzero eigenvalues, required since $\mathbf{L}$ is
singular by construction) and $\hat{\bm{\Sigma}}=\frac{1}{n}
\sum_{i=1}^{n}(\mathbf{x}_i-\bm{\mu})(\mathbf{x}_i-\bm{\mu})^T$ is the
sample covariance matrix of the observed signals. Maximizing
$\mathcal{F}_n(\mathbf{L}\mid\mathbf{X})$ alone is generally
ill-posed and prone to overfitting when $n$ is small relative to $N$;
we therefore augment it with a regularization term $\mathcal{P}
(\mathbf{L})$, weighted by $\rho>0$, that promotes structural
properties such as sparsity, e.g., $\mathcal{P}(\mathbf{L})=
\Vert\mathbf{L}\Vert_F^2$ or $\mathcal{P}(\mathbf{L})=
-\log(\mathrm{diag}(\mathbf{L}))$. The resulting baseline graph
learning problem is the regularized maximum-likelihood estimator
\begin{align}
    \max_{\mathbf{L}}\quad
    \mathcal{F}_n(\mathbf{L}\mid\mathbf{X})-\rho\,\mathcal{P}(\mathbf{L})\quad
    \text{s.t.}\;\mathbf{L}\in\mathcal{L}.
    \label{equ.8}
\end{align}
Problem~\eqref{equ.8} is the workhorse formulation underlying classical
smooth-signal graph learning. It implicitly identifies the empirical
distribution of $\mathbf{X}$ with the true data-generating distribution
and offers no safeguard when this identification fails, whether because
$n$ is too small for $\hat{\bm{\Sigma}}$ to be a reliable estimate of
the population covariance, or because $\mathbf{X}$ is itself pooled, in
whole or in part, from sources whose distributions differ from that of
the intended target. Both failure modes are the norm rather than the
exception in the multi-source, small-target-sample regime motivating
this paper, which calls for replacing the single fixed nominal
distribution in~\eqref{equ.8} with an ambiguity set that is robust to
distributional uncertainty while remaining informed by all available
source data. Constructing such a set requires a notion of distance
between probability distributions together with a principled way of
aggregating several of them; we introduce both next.

\subsection{Wasserstein Distance, Barycenter, and Gelbrich Bound}
\label{subsec:wasserstein_prelim}

We adopt the Wasserstein distance as the governing metric for
distributional uncertainty, owing to its ability to compare
distributions with disjoint or singular support, as is the case for
the degenerate Gaussian law in~\eqref{equ.4}, and to its well-known
tractable convex reformulations in distributionally robust
optimization~\cite{esfahani2018data,kuhn2019wasserstein}.

\begin{definition}[Wasserstein Distance]
\label{def.1}
For $p\in[1,\infty)$, the order-$p$ Wasserstein distance between
probability distributions $\mathbb{P}_1$ and $\mathbb{P}_2$ on
$\mathbb{R}^N$ is
\begin{align}
    &W_p(\mathbb{P}_1,\mathbb{P}_2)\nonumber\\
    &=\left(\inf_{\pi\in\mathcal{U}(\mathbb{P}_1,\mathbb{P}_2)}
    \int_{\mathbb{R}^N\times\mathbb{R}^N}
    C(\mathbf{a}_1,\mathbf{a}_2)^p\,\pi(d\mathbf{a}_1,d\mathbf{a}_2)
    \right)^{\frac{1}{p}},
    \label{equ.5}
\end{align}
where $\mathcal{U}(\mathbb{P}_1,\mathbb{P}_2)$ denotes the set of
couplings of $\mathbb{P}_1$ and $\mathbb{P}_2$, i.e., joint
distributions on $\mathbb{R}^N\times\mathbb{R}^N$ with marginals
$\mathbb{P}_1$ and $\mathbb{P}_2$, and $C(\cdot,\cdot)$ is a
transportation cost function.
\end{definition}
Throughout, we take $C(\mathbf{a}_1,\mathbf{a}_2)=\Vert
\mathbf{a}_1-\mathbf{a}_2\Vert_p$, so that $W_p(\mathbb{P}_1,
\mathbb{P}_2)$ quantifies the minimal cost of transporting the mass of
$\mathbb{P}_1$ onto $\mathbb{P}_2$ under an $\ell_p$ transportation
cost.

When the data of interest arise from several distributions rather than
one, it is natural to seek a single representative distribution that
best summarizes them jointly in the Wasserstein metric. This role is
played by the Wasserstein barycenter, the Fr\'{e}chet mean of a
distribution over the Wasserstein space.

\begin{definition}[Wasserstein Barycenter]
\label{def:barycenter}
For $p\in[1,\infty)$, the $p$-Wasserstein barycenter of a distribution
$\mathbb{P}$ over probability measures is
\begin{align}
    b_p(\mathbb{P}):=\operatorname*{arg\,min}_{\nu}\,
    \mathbb{E}_{\rho\sim\mathbb{P}}\bigl[W_p^p(\nu,\rho)\bigr],
    \label{eq1}
\end{align}
where $\rho$ is a random measure distributed according to
$\mathbb{P}$. In particular, taking $\mathbb{P}=\sum_{m=1}^{M}
\lambda_m\delta_{\rho_m}$ for weights $\bm{\lambda}=(\lambda_m)_{m\in
[M]}\in\triangle^M$ on the probability simplex recovers the
\emph{empirical} $p$-Wasserstein barycenter of $M$ measures
$\rho_1,\ldots,\rho_m$,
\begin{align}
    \hat{b}_{\bm{\lambda},p}(\rho_1,\ldots,\rho_M)
    :=\operatorname*{arg\,min}_{\nu}\,
    \sum_{m=1}^{M}\lambda_m\,W_p^p(\nu,\rho_m).
    \label{eq2}
\end{align}
\end{definition}
Equation~\eqref{eq2} will later serve to fuse the $M$ heterogeneous
source distributions into a single nominal distribution for the
ambiguity set of Section~\ref{sec:wdro}, with $\bm{\lambda}$
controlling each source's relative contribution.

Evaluating $W_p$ or its barycenter exactly is, in general, NP-hard
outside a few special cases~\cite{peyre2019computational}. Since the
signal model of Section~\ref{subsec:model} identifies each source
distribution with a (degenerate) Gaussian law, and hence with its first
two moments, we can instead work with a moment-based surrogate that
is tractable for arbitrary covariance matrices and coincides with the
$2$-Wasserstein distance whenever both distributions are Gaussian: the
Gelbrich distance.

\begin{definition}[Gelbrich Distance]
\label{def:gelbrich}
For distributions $\mathbb{P}_1,\mathbb{P}_2$ with means $\mu_1,\mu_2$
and covariance matrices $\Sigma_1,\Sigma_2\in\mathbb{S}_+^N$, the
Gelbrich distance between them is
\begin{align}
    G(\mathbb{P}_1,\mathbb{P}_2)
    :=\sqrt{\Vert\mu_1-\mu_2\Vert_2^2+B^2(\Sigma_1,\Sigma_2)},
    \label{eq7}
\end{align}
where
\begin{align}
    B^2(\Sigma_1,\Sigma_2)
    &:=\mathrm{tr}(\Sigma_1)+\mathrm{tr}(\Sigma_2)\nonumber\\
    &\quad-2\,\mathrm{tr}\Bigl(\bigl(\Sigma_1^{1/2}\Sigma_2\Sigma_1^{1/2}
    \bigr)^{1/2}\Bigr)
    \label{eq7b}
\end{align}
is the squared Bures--Wasserstein distance between $\Sigma_1$ and
$\Sigma_2$.
\end{definition}
Because $G(\mathbb{P}_1,\mathbb{P}_2)$ depends on $\mathbb{P}_1,
\mathbb{P}_2$ only through their first two moments, it provides a
closed-form, provably tight lower bound on $W_2(\mathbb{P}_1,
\mathbb{P}_2)$ that is exact in the Gaussian case relevant
here~\cite{gelbrich1990formula}, and it is this tractability that makes
the Wasserstein barycentric ambiguity set of Section~\ref{sec:wdro}
computationally viable at the scale of the graph learning problem
formulated in~\eqref{equ.8}.

\section{Wasserstein Distributionally Robust Graph Learning}
\label{sec:wdro}
This section develops the proposed distributionally robust graph
learning framework in four steps. We begin by revisiting the baseline
estimator of Section~\ref{subsec:baseline_problem} through the lens of
empirical risk minimization and robustifying it against a single
source's sampling uncertainty via a Wasserstein ambiguity
set~(Section~\ref{subsec:basic_formulation}). We then extend this
single-source construction to the heterogeneous multi-source setting
motivating this paper, replacing the naively pooled empirical
distribution with a Wasserstein barycentric ambiguity set that fuses
the $M$ source distributions while explicitly accounting for their
residual sampling and heterogeneity error
(Section~\ref{subsec:barycentric_wdro}). The resulting min--max problem
is then reduced to a single-level, tractable convex program via
Wasserstein strong duality (Section~\ref{subsec:tractable}), which in
turn admits an efficient ADMM solver with closed-form per-block updates
(Section~\ref{subsec:admm}).

\subsection{Basic Formulation}
\label{subsec:basic_formulation}

The baseline problem~\eqref{equ.8} can be equivalently cast as an
empirical risk minimization (ERM) problem,
\begin{align}
    \min_{\mathbf{L}\in\mathcal{L}}\quad
    \mathbb{E}_{\mathbf{x}\sim\mathbb{P}_n}
    [-\mathcal{F}_n(\mathbf{L}\mid\mathbf{X})]+\rho\mathcal{P}(\mathbf{L}),
    \label{equ.9}
\end{align}
where $\mathbb{P}_n$ is the empirical distribution of the observed
signals $\mathbf{x}_1,\ldots,\mathbf{x}_n$,
\begin{align}
    \mathbb{P}_n=\frac{1}{n}\sum_{i=1}^n\delta_{\mathbf{x}_i},
    \label{equ.10}
\end{align}
with $\delta_{\mathbf{x}_i}$ the Dirac measure at $\mathbf{x}_i$.
Formulation~\eqref{equ.9} makes explicit that the baseline estimator
optimizes the risk exclusively with respect to $\mathbb{P}_n$, which
inherits two well-known deficiencies from the finiteness of the sample:
$\mathbb{P}_n$ may deviate substantially from the true data-generating
distribution when observations are scarce or noisy, and an estimator
tuned to $\mathbb{P}_n$ alone is prone to overfit the observed signals
at the expense of generalization to unseen ones. Both deficiencies point
to the same remedy, the learned graph should remain accurate not only
for $\mathbb{P}_n$, but for every distribution plausibly consistent with
the observed data.

We formalize this requirement by replacing the single point
$\mathbb{P}_n$ with a Wasserstein ball around it,
\begin{align}
    \mathcal{M}=\{\mathbb{P}\in\mathcal{N}^N:W_p(\mathbb{P},\mathbb{P}_n)\leq \epsilon\},
    \label{equ.11}
\end{align}
where $\mathcal{N}^N$ denote the the family of all normal distributions on $\mathbb{R}^N$,  $\epsilon\geq 0$ is the ambiguity set radius and $p\geq 1$ is left
unspecified for now, as Section~\ref{subsec:tractable} shows that the
subsequent reformulation is independent of this choice. Learning a graph
that performs well uniformly over $\mathcal{M}$ amounts to minimizing
the worst-case expected risk,
\begin{align}
    &\min_{\mathbf{L}\in\mathcal{L}}R(\mathbf{L};\mathbf{x})+\rho\mathcal{P}(\mathbf{L})
    \nonumber\\
    &=\min_{\mathbf{L}\in\mathcal{L}}\max_{\mathbb{P}\in\mathcal{M}}
    \mathbb{E}_{\mathbf{x}\sim\mathbb{P}}[-\mathcal{F}_n(\mathbf{L}\mid\mathbf{X})]
    +\rho\mathcal{P}(\mathbf{L}),
    \label{equ.12}
\end{align}
where $R(\mathbf{L};\mathbf{x})=\max_{\mathbb{P}\in\mathcal{M}}
\mathbb{E}_{\mathbf{x}\sim\mathbb{P}}[-\mathcal{F}_n(\mathbf{L}\mid\mathbf{X})]$
is the worst-case expected risk. The rationale behind~\eqref{equ.12} is
that if $\mathcal{M}$ contains the true distribution $\mathbb{P}^\star$,
then driving down the worst-case risk over $\mathcal{M}$ simultaneously
controls the risk at $\mathbb{P}^\star$: the expected risk of the
learned graph $\hat{\mathbf{L}}$ under $\mathbb{P}^\star$ is bounded by
$R(\hat{\mathbf{L}};\mathbf{x})$, so that $\hat{\mathbf{L}}$ remains
reliable even though it was estimated from a finite, possibly
unrepresentative sample. The radius $\epsilon$ governs a
robustness--conservatism trade-off intrinsic to this construction: a
larger $\epsilon$ raises the likelihood that $\mathcal{M}$ covers
$\mathbb{P}^\star$, but also admits increasingly implausible
distributions into the worst case, inflating $R(\mathbf{L};\mathbf{x})$
and yielding an overly conservative estimator; $\epsilon$ must therefore
be chosen with care, a question we return to in
Section~\ref{sec:enhanced_theory}.

\subsection{Wasserstein Barycentric Distributionally Robust Optimization}
\label{subsec:barycentric_wdro}

The single-source construction of Section~\ref{subsec:basic_formulation}
guards against sampling error but offers no mechanism for incorporating
data from related source domains. We now extend it to the heterogeneous
multi-source setting that motivates this paper. Suppose each source
$m\in[M]$ contributes a local sample $\mathcal{D}_m=\{\mathbf{z}_{m,1},
\ldots,\mathbf{z}_{m,n_m}\}$ of size $n_m$, with $n=\sum_{m=1}^M n_m$,
and empirical distribution $\hat{\mathbb{P}}_m=\frac{1}{n_m}
\sum_{i=1}^{n_m}\delta_{\mathbf{z}_{m,i}}$. The most direct extension of
\eqref{equ.9} to this setting is the weighted ERM problem
\begin{align}
    \min_{\mathbf{L}\in\mathcal{L}}\sum_{m=1}^M\frac{\lambda_m}{n_m}
    \sum_{i=1}^{n_m}\ell(\mathbf{L},\mathbf{z}_{m,i}).
    \label{eq3}
\end{align}

\begin{remark}
The default choice $\lambda_m=n_m/n$ reduces~\eqref{eq3} to an ERM
objective over the union of all local samples, which is equivalent to
minimizing $\mathbb{E}_{\mathbf{z}\sim\hat{\mathbb{P}}_{\bm{\lambda}}}
[\ell(\mathbf{L},\mathbf{z})]$ under the $\bm{\lambda}$-mixture
$\hat{\mathbb{P}}_{\bm{\lambda}}:=\sum_{m=1}^M\lambda_m\hat{\mathbb{P}}_m$.
As argued by Mohri et al.~\cite{mohri2019agnostic}, this uniform
weighting is generally unwarranted: the target distribution for which
the centralized model is ultimately deployed need not coincide with the
sample-size-weighted mixture $\sum_{m=1}^Mn_m\mathbb{P}_m/n$, and is
better modeled as a general $\bm{\lambda}$-mixture $\mathbb{P}_{\bm
{\lambda}}:=\sum_{m=1}^M\lambda_m\mathbb{P}_m$ for a weight vector
$\bm{\lambda}\in\triangle^M$ to be determined by the learning problem
itself rather than fixed a priori by sample proportions.
\end{remark}

Observe that $\mathbb{P}_{\bm{\lambda}}$, viewed as a mixture, is
precisely the \emph{Euclidean} barycenter of $\mathbb{P}_1,\ldots,
\mathbb{P}_M$ in the space of probability measures. This observation
suggests a natural generalization: replacing the Euclidean barycenter
with the geometrically richer \emph{Wasserstein} barycenter
$\hat{b}_{\bm{\lambda},p}(\mathbb{P}_1,\ldots,\mathbb{P}_M)$ of
Definition~\ref{def:barycenter}, so that~\eqref{eq3} becomes
\begin{align}
    \min_{\mathbf{L}\in\mathcal{L}}
    \mathbb{E}_{\mathbf{z}\sim\hat{b}_{\bm{\lambda},p}
    (\mathbb{P}_1,\ldots,\mathbb{P}_M)}[\ell(\mathbf{L},\mathbf{z})],
    \label{eq4}
\end{align}
with the corresponding empirical surrogate, computed from the source
samples $\mathcal{D}_1,\ldots,\mathcal{D}_M$ via their empirical
distributions $\hat{\mathbb{P}}_1,\ldots,\hat{\mathbb{P}}_M$, given by
\begin{align}
    \min_{\mathbf{L}\in\mathcal{L}}
    \mathbb{E}_{\mathbf{z}\sim\hat{b}_{\bm{\lambda},p}
    (\hat{\mathbb{P}}_1,\ldots,\hat{\mathbb{P}}_M)}[\ell(\mathbf{L},\mathbf{z})].
    \label{eq5}
\end{align}
Unlike the mixture $\hat{\mathbb{P}}_{\bm{\lambda}}$, the Wasserstein
barycenter of~\eqref{eq5} respects the geometry of each source
distribution rather than simply averaging their probability mass, and
consequently provides a more faithful consensus representation of
heterogeneous sources (a claim made precise in
Section~\ref{subsec:bary_vs_pool}). This advantage notwithstanding, the
barycenter of the \emph{empirical} distributions
$\hat{\mathbb{P}}_1,\ldots,\hat{\mathbb{P}}_M$ remains only a finite-sample
estimate of the barycenter of the true source distributions
$\mathbb{P}_1,\ldots,\mathbb{P}_M$, and residual discrepancy between the
two is unavoidable due to sampling error within each source and
heterogeneity across sources. Exactly as in
Section~\ref{subsec:basic_formulation}, we hedge against this residual
uncertainty by enclosing the empirical barycenter in a Wasserstein ball
rather than treating it as exact. The ambiguity set constructed below instantiates, in the present
graph-learning setting, the general \emph{Wasserstein barycentric
ambiguity set} proposed by Lau and Liu~\cite{lau2022wasserstein} for
aggregate distributionally robust decision-making with heterogeneous
multi-source data; we recall its definition and specialize it
throughout to the degenerate Gaussian graph-signal model of
Section~\ref{subsec:model}.

\begin{definition}
For $\bm{\lambda}\in\triangle^M$, the $p$-Wasserstein barycentric
ambiguity set of radius $\epsilon\geq 0$, centered at the Wasserstein
barycenter of $\mathbb{P}_1,\ldots,\mathbb{P}_M$ (assumed to exist)
$\bar{\mathbf{P}}_{\bm{\lambda},p}=\hat{b}_{\bm{\lambda},p}
(\mathbb{P}_1,\ldots,\mathbb{P}_M)$, is
\begin{align}
    \mathcal{W}_{\epsilon,p}(\mathbb{P}_1,\ldots,\mathbb{P}_M;\bm{\lambda})
    :=\{\mathbb{P}:W_p(\mathbb{P},\bar{\mathbf{P}}_{\bm{\lambda},p})
    \leq\epsilon\}.
    \label{eq6}
\end{align}
\end{definition}

\begin{remark}
An alternative route to an ambiguity set over $M$ distributions is the
intersection of individual Wasserstein balls, $\cap_{m=1}^M\mathcal{M}_m$
with $\mathcal{M}_m=\{\mathbb{P}:W_p(\mathbb{P},\mathbb{P}_m)\leq
\epsilon\}$; when the sources are highly heterogeneous, however, this
construction forces a single radius $\epsilon$ to simultaneously cover
all $M$ sources and becomes overly conservative. A less conservative
alternative is the pooled ambiguity set $\tilde{\mathcal{M}}_{\epsilon,p}
(\mathbb{P}_1,\ldots,\mathbb{P}_M;\bm{\lambda}):=\bigl\{\mathbb{P}:
\sum_{m=1}^M\lambda_mW_p^p(\mathbb{P},\mathbb{P}_m)\leq\epsilon^p\bigr\}$,
which satisfies $\cap_{m=1}^M\mathcal{M}_m\subseteq
\tilde{\mathcal{M}}_{\epsilon,p}$ for any $\bm{\lambda}\in\triangle^M$.
The following theorem shows that the barycentric ambiguity set
of~\eqref{eq6} is, up to a dimension-independent constant, equivalent
to $\tilde{\mathcal{M}}_{\epsilon,p}$, and is therefore preferable on
computational grounds since it is centered at a single distribution
rather than requiring simultaneous control over $M$ constraints.
\end{remark}

\begin{theorem}
\label{thm1}
For $\bm{\lambda}\in\triangle^M$, suppose a $\bm{\lambda}$-weighted
$p$-Wasserstein barycenter of $\mathbb{P}_1,\ldots,\mathbb{P}_M$
exists. Then, for any $\epsilon\geq 0$,
\begin{align}
    \tilde{\mathcal{M}}_{\epsilon,p}(\mathbb{P}_1,\ldots,\mathbb{P}_M;\bm{\lambda})
    \subseteq\mathcal{W}_{2^p\cdot\epsilon,p}(\mathbb{P}_1,\ldots,\mathbb{P}_M;\bm{\lambda}).
\end{align}
\end{theorem}
\textit{Proof}. Let $\bar{\mathbf{P}}_{\bm{\lambda},p}$ denote the
Wasserstein barycenter of $\mathbb{P}_1,\ldots,\mathbb{P}_M$ for
$\bm{\lambda}\in\triangle^M$. By the triangle inequality and the
inequality $(a+b)^p\leq 2^{p-1}(a^p+b^p)$ for $a,b\geq 0$ and
$p\in[1,\infty)$, for any $\mathbb{P}$ and $m\in[M]$,
\begin{align*}
    W_p^p(\mathbb{P},\bar{\mathbf{P}}_{\bm{\lambda},p})
    &\leq\bigl(W_p(\mathbb{P},\mathbb{P}_m)
    +W_p(\bar{\mathbf{P}}_{\bm{\lambda},p},\mathbb{P}_m)\bigr)^p\\
    &\leq2^{p-1}\bigl(W_p^p(\mathbb{P},\mathbb{P}_m)
    +W_p^p(\bar{\mathbf{P}}_{\bm{\lambda},p},\mathbb{P}_m)\bigr).
\end{align*}
Taking the $\bm{\lambda}$-weighted average over $m\in[M]$ and using
the definition of the Wasserstein barycenter to bound
$\sum_{m=1}^M\lambda_mW_p^p(\bar{\mathbf{P}}_{\bm{\lambda},p},\mathbb{P}_m)
\leq\sum_{m=1}^MW_p^p(\mathbb{P},\mathbb{P}_m)$ yields
\begin{align*}
    &W_p^p(\mathbb{P},\bar{\mathbf{P}}_{\bm{\lambda},p})
    =\sum_{m=1}^M\lambda_mW_p^p(\mathbb{P},\bar{\mathbf{P}}_{\bm{\lambda},p})\\
    &\leq 2^{p-1}\sum_{m=1}^M\lambda_mW_p^p(\mathbb{P},\mathbb{P}_k)
    +2^{p-1}\sum_{m=1}^MW_p^p(\mathbb{P},\mathbb{P}_m)\\
    &=2^{p}\sum_{m=1}^M\lambda_mW_p^p(\mathbb{P},\mathbb{P}_m).
\end{align*}
Consequently, for any $\epsilon\geq 0$,
\begin{align*}
    \sum_{m=1}^M\lambda_mW_p^p(\mathbb{P},\mathbb{P}_m)\leq\epsilon
    \;\Longrightarrow\;
    W_p^p(\mathbb{P},\bar{\mathbf{P}}_{\bm{\lambda},p})\leq 2^p\cdot\epsilon,
\end{align*}
which is the desired inclusion. \hfill$\square$
\begin{remark}
Theorem~\ref{thm1} adapts, to the present finite multi-source setting,
the inclusion result of Lau and Liu~\cite[Theorem~4.2]{lau2022wasserstein}
relating the pooled ambiguity set $\tilde{\mathcal{M}}_{\epsilon,p}$ to
the barycentric ambiguity set $\mathcal{W}_{\epsilon,p}$; we include
the short proof for completeness since it underlies the tractable
reformulation of Section~\ref{subsec:tractable}.
\end{remark}

With the barycentric ambiguity set in hand, the multi-source
counterpart of the min--max problem~\eqref{equ.12} reads
\begin{align}
    \inf_{\mathbf{L}\in\mathcal{L}}\Bigl\{-\log|\mathbf{L}|_+
    +\sup_{\mathbb{P}\in\mathcal{W}_{\epsilon,p}}
    \mathbb{E}_{\mathbf{z}\sim\mathbb{P}}[\mathrm{tr}(\Sigma\mathbf{L})]
    +\rho\mathcal{P}(\mathbf{L})\Bigr\}.
    \label{eq8}
\end{align}
Evaluating~\eqref{eq8} in practice requires an explicit handle on the
center of the ambiguity set, $\bar{\mathbf{P}}_{\bm{\lambda},p}$, which
we now characterize under the Gaussian source model implied by the
signal model of Section~\ref{subsec:model}.

\begin{assumption}
\label{ass:subgauss}
$\mathbb{P}_1,\ldots,\mathbb{P}_M$ are $M$ possibly degenerate Gaussian
distributions $\mathbb{P}_m=\mathcal{N}(\mu_m,\Sigma_m)$ with $\mu_m\in
\mathbb{R}^N$ and $\Sigma_m\in\mathbb{S}_+^N$ for $m\in[M]$.
\end{assumption}

\begin{remark}
\label{rem:shared_kernel}
Under the graph signal model of Section~\ref{subsec:model}, every
source covariance is of the form $\Sigma_m=\mathbf{L}_m^{\dagger}$ for
some $\mathbf{L}_m\in\mathcal{L}$, and hence, by~\eqref{equ.1},
$\Sigma_m\mathbf{1}=\mathbf{0}$ for every $m\in[M]$: the $M$ source
covariances share the common null space $\mathrm{span}\{\mathbf{1}\}$.
Assumption~\ref{ass:subgauss} accordingly allows, and in fact requires,
$\Sigma_m\in\mathbb{S}_+^N$ to be rank-deficient; the barycenter
characterized below must be derived on the closed cone $\mathbb{S}_+^N$
rather than on the open cone $\mathbb{S}_{++}^N$.
\end{remark}

\begin{proposition}
\label{prop1}
Under Assumption~\ref{ass:subgauss}, the $2$-Wasserstein barycenter of
$\mathbb{P}_1,\ldots,\mathbb{P}_M$, restricted to the Gaussian manifold,
is $\bar{\mathbf{P}}_{\bm{\lambda},2}=\hat{b}_{\bm{\lambda},2}
(\mathbb{P}_1,\ldots,\mathbb{P}_M)=\mathcal{N}(\bar{\mu}_{\bm{\lambda}},
\bar{\Sigma}_{\bm{\lambda}})$, with
\begin{align}
    \bar{\mu}_{\bm{\lambda}}&=\sum_{m=1}^M\lambda_m\mu_m,\nonumber\\
    \bar{\Sigma}_{\bm{\lambda}}&=\operatorname*{arg\,min}_{\Sigma\in
    \mathbb{S}_{+}^N}\sum_{m=1}^M\lambda_mB^2(\Sigma,\Sigma_m),
    \label{eq9}
\end{align}
where $B(\cdot,\cdot)$ is the Bures--Wasserstein distance of
Definition~\ref{def:gelbrich} (see~\eqref{eq7b}). Moreover, if all
source covariances share a common null space,
$\bigcap_{m=1}^M\mathrm{null}(\Sigma_m)=:\mathcal{K}\neq\{\mathbf{0}\}$
(in particular $\mathcal{K}\supseteq\mathrm{span}\{\mathbf{1}\}$ under
Remark~\ref{rem:shared_kernel}), then $\bar{\Sigma}_{\bm{\lambda}}$
inherits this null space, $\mathcal{K}\subseteq
\mathrm{null}(\bar{\Sigma}_{\bm{\lambda}})$, so $\bar{\Sigma}_{\bm
{\lambda}}$ is itself rank-deficient and, in particular, is never
positive definite whenever $\mathcal{K}\neq\{\mathbf{0}\}$.
\end{proposition}
\textit{Proof.} Existence and uniqueness of the minimizer in~\eqref{eq9} on the closed
cone $\mathbb{S}_+^N$ follow from~\cite[Thm.~1]{alvarez2016fixed}.
Let $\mathbf{v}\in\mathcal{K}$; since $\Sigma_m\mathbf{v}=\mathbf{0}$
for every $m$, we have $\Sigma_m^{1/2}\mathbf{v}=\mathbf{0}$, and hence
$\mathbf{v}^T\Sigma\mathbf{v}=0$ for the fixed point $\Sigma=
\bar{\Sigma}_{\bm{\lambda}}$ of
\begin{align}
    \Sigma=\sum_{m=1}^M\lambda_m\sqrt{\Sigma^{1/2}\Sigma_m\Sigma^{1/2}}
    \label{eq10}
\end{align}
(each summand on the right annihilates $\mathbf{v}$ on both sides by
the same argument, so $\mathbf{v}$ remains in the null space at every
iterate, and hence at the fixed point), which, since
$\bar{\Sigma}_{\bm{\lambda}}\succeq0$, forces $\bar{\Sigma}_{\bm
{\lambda}}\mathbf{v}=\mathbf{0}$. $\hfill\square$

Because Proposition~\ref{prop1} shows
$\bar{\Sigma}_{\bm{\lambda}}\mathbf{1}=\mathbf{0}$, the fixed-point
map~\eqref{eq10} is ill-posed if applied verbatim on
$\mathbb{S}_{++}^N$, since it requires $\Sigma^{-1/2}$ at a singular
matrix. We instead solve~\eqref{eq10} in the $(N-1)$-dimensional
subspace $\mathbf{1}^\perp$, where positive definiteness is generically
restored. Let $\mathbf{U}_\perp\in\mathbb{R}^{N\times(N-1)}$ be an
orthonormal basis of $\mathbf{1}^\perp$ and define the reduced source
covariances $\Sigma_m^{\perp}:=\mathbf{U}_\perp^T\Sigma_m
\mathbf{U}_\perp\in\mathbb{S}_{+}^{N-1}$, which are generically positive
definite whenever each source graph is connected (so that
$\mathrm{null}(\mathbf{L}_m)=\mathrm{span}\{\mathbf{1}\}$ exactly). The
reduced barycenter $\bar{\Sigma}_{\bm{\lambda}}^{\perp}\in
\mathbb{S}_{++}^{N-1}$ is then computed by the well-posed fixed-point
iteration
\begin{align}
    &(\bar{\Sigma}_{\bm{\lambda}}^{\perp})^{k+1}
    =\mathcal{H}\bigl((\bar{\Sigma}_{\bm{\lambda}}^{\perp})^{k},
    \bm{\lambda},\{\Sigma_m^{\perp}\}_{m=1}^M\bigr)\nonumber\\
    &=\bigl((\bar{\Sigma}_{\bm{\lambda}}^{\perp})^{k}\bigr)^{-1/2}
    \Bigl(\sum_{m=1}^M\lambda_m
    \sqrt{((\bar{\Sigma}_{\bm{\lambda}}^{\perp})^{k})^{1/2}
    \Sigma_m^{\perp}((\bar{\Sigma}_{\bm{\lambda}}^{\perp})^{k})^{1/2}}
    \Bigr)^2\nonumber\\
    &\quad\cdot\bigl((\bar{\Sigma}_{\bm{\lambda}}^{\perp})^{k}\bigr)^{-1/2},
    \label{eq10b}
\end{align}
which converges at a linear rate to the unique positive definite fixed
point $\bar{\Sigma}_{\bm{\lambda}}^{\perp}$ in
$\mathbb{S}_{++}^{N-1}$~\cite{alvarez2016fixed,chewi2020gradient}. The
ambient barycenter covariance is then recovered exactly, and with the
correct rank and null space, by the lifting
\begin{align}
    \bar{\Sigma}_{\bm{\lambda}}
    =\mathbf{U}_\perp\,\bar{\Sigma}_{\bm{\lambda}}^{\perp}\,
    \mathbf{U}_\perp^T\in\mathbb{S}_+^N,
    \qquad
    \bar{\Sigma}_{\bm{\lambda}}\mathbf{1}=\mathbf{0}.
    \label{eq10c}
\end{align}

\begin{lemma}
\label{lem:kernel_invariance}
For every $\mathbf{L}\in\mathcal{L}$ and every $\Sigma\in\mathbb{S}^N$,
\begin{align}
    \mathrm{tr}(\Sigma\mathbf{L})=\mathrm{tr}(\mathbf{P}\Sigma\mathbf{P}
    \,\mathbf{L}),\qquad
    \mathbf{P}:=\mathbf{I}-\tfrac{1}{N}\mathbf{1}\mathbf{1}^T.
    \label{eq:kernel_invariance}
\end{align}
\end{lemma}
\textit{Proof.} Since $\mathbf{L}\mathbf{1}=\mathbf{0}$ and $\mathbf{L}=\mathbf{L}^T$,
we have $\mathbf{L}\mathbf{1}\mathbf{1}^T=\mathbf{0}$ and
$\mathbf{1}\mathbf{1}^T\mathbf{L}=(\mathbf{L}\mathbf{1}\mathbf{1}^T)^T
=\mathbf{0}$, so $\mathbf{P}\mathbf{L}=\mathbf{L}\mathbf{P}=\mathbf{L}$,
i.e.\ $\mathbf{L}=\mathbf{P}\mathbf{L}\mathbf{P}$. Then, by cyclicity of
the trace, $\mathrm{tr}(\Sigma\mathbf{L})=\mathrm{tr}(\Sigma\,
\mathbf{P}\mathbf{L}\mathbf{P})=\mathrm{tr}(\mathbf{P}\Sigma\mathbf{P}
\,\mathbf{L})$. $\hfill\square$

\begin{remark}
\label{rem:empirical_rank}
In practice, the fixed-point iteration~\eqref{eq10b}--\eqref{eq10c} is
applied not to the population covariances $\Sigma_m$ but to their
empirical counterparts $\hat{\Sigma}_m$, which are generically full
rank due to finite-sample and observation noise, so the resulting
empirical barycenter $\hat{\Sigma}_{\bm{\lambda}}$ need not satisfy
$\hat{\Sigma}_{\bm{\lambda}}\mathbf{1}=\mathbf{0}$ exactly.
Lemma~\ref{lem:kernel_invariance} shows this is immaterial to the
downstream estimator: because every feasible $\mathbf{L}\in\mathcal{L}$
annihilates $\mathbf{1}$, the ADMM objective~\eqref{eq17} depends on
$\hat{\Sigma}_{\bm{\lambda}}$ only through its projection
$\mathbf{P}\hat{\Sigma}_{\bm{\lambda}}\mathbf{P}$; any energy the
empirical fixed point spuriously accumulates along $\mathbf{1}$ is
automatically annihilated by the trace term and never biases the
learned graph. For numerical robustness of~\eqref{eq10b} itself,
however, we still recommend explicitly projecting each empirical
source covariance, $\hat{\Sigma}_m\leftarrow\mathbf{P}\hat{\Sigma}_m
\mathbf{P}$, before running the fixed-point iteration, so that the
reduced problem in $\mathbf{U}_\perp^T(\cdot)\mathbf{U}_\perp$ is
solved on genuinely well-conditioned inputs rather than on inputs
whose $\mathbf{1}$-component is pure sampling noise.
\end{remark}

\begin{assumption}
\label{ass:bounded_energy}
There exists a constant $R<\infty$ such that
$\|\mathbf{z}\|_2\leq R$ almost surely under every source distribution
$\mathbb{P}_m$, $m\in[M]$, and under every distribution $\mathbb{P}$
contained in the vector-space ambiguity set
$\mathcal{W}_{\epsilon,2}(\mathbb{P}_1,\ldots,\mathbb{P}_M;\bm{\lambda})$
of~\eqref{eq6}.
\end{assumption}

\begin{remark}
Assumption~\ref{ass:bounded_energy} is satisfied exactly whenever
signals arise from a bounded sensing range, and is satisfied up to an
exponentially small failure probability whenever
$\mathbf{z}$ is sub-Gaussian, by taking $R$ to be the
$(1-\delta)$-quantile of $\Vert\mathbf{z}\Vert_2$ and absorbing the
$\delta$-tail into a separate, exponentially decaying error term via a
standard truncation argument~\cite{esfahani2018data}; we omit this
extension for brevity and work directly under
Assumption~\ref{ass:bounded_energy}.
\end{remark}

\begin{lemma}
\label{lem:lifting}
Let $\Phi:\mathbb{R}^N\to\mathbb{S}^{N\times N}$,
$\Phi(\mathbf{z}):=\mathbf{z}\mathbf{z}^T$. Under
Assumption~\ref{ass:bounded_energy}, $\Phi$ is $2R$-Lipschitz from
$(\{\Vert\mathbf{z}\Vert_2\leq R\},\Vert\cdot\Vert_2)$ to
$(\mathbb{S}^{N\times N},\Vert\cdot\Vert_F)$, and consequently, for any
two distributions $\mathbb{P},\mathbb{P}'$ supported on
$\{\Vert\mathbf{z}\Vert_2\leq R\}$,
\begin{align}
    W_2\bigl(\Phi_{\#}\mathbb{P},\Phi_{\#}\mathbb{P}';\Vert\cdot\Vert_F\bigr)
    \leq 2R\,W_2(\mathbb{P},\mathbb{P}';\Vert\cdot\Vert_2).
    \label{eq:lifting_bound}
\end{align}
\end{lemma}
\textit{Proof.} For $\Vert\mathbf{z}_1\Vert_2,\Vert\mathbf{z}_2\Vert_2\leq R$,
\begin{align*}
    &\Vert\Phi(\mathbf{z}_1)-\Phi(\mathbf{z}_2)\Vert_F
    =\bigl\Vert\mathbf{z}_1(\mathbf{z}_1-\mathbf{z}_2)^T
    +(\mathbf{z}_1-\mathbf{z}_2)\mathbf{z}_2^T\bigr\Vert_F\\
    &\leq\bigl(\Vert\mathbf{z}_1\Vert_2+\Vert\mathbf{z}_2\Vert_2\bigr)
    \Vert\mathbf{z}_1-\mathbf{z}_2\Vert_2
    \leq 2R\Vert\mathbf{z}_1-\mathbf{z}_2\Vert_2,
\end{align*}
so $\Phi$ is $2R$-Lipschitz on $\{\Vert\mathbf{z}\Vert_2\leq R\}$. Let
$\pi^\star$ be an optimal coupling attaining $W_2(\mathbb{P},\mathbb{P}')$.
Then $(\Phi,\Phi)_{\#}\pi^\star$ is a feasible coupling of
$\Phi_{\#}\mathbb{P}$ and $\Phi_{\#}\mathbb{P}'$, and
\begin{align*}
    &W_2^2(\Phi_{\#}\mathbb{P},\Phi_{\#}\mathbb{P}')
    \leq\int\Vert\Phi(\mathbf{z}_1)-\Phi(\mathbf{z}_2)\Vert_F^2
    \,d\pi^\star(\mathbf{z}_1,\mathbf{z}_2)\\
    &\leq 4R^2\int\Vert\mathbf{z}_1-\mathbf{z}_2\Vert_2^2
    \,d\pi^\star(\mathbf{z}_1,\mathbf{z}_2)
    =4R^2W_2^2(\mathbb{P},\mathbb{P}'),
\end{align*}
which gives~\eqref{eq:lifting_bound}. $\hfill\square$

\begin{corollary}
\label{cor:pushforward}
Under Assumption~\ref{ass:bounded_energy}, the image of the
vector-space ambiguity set $\mathcal{W}_{\epsilon,2}$ of~\eqref{eq6}
under the outer-product map satisfies
\begin{align}
    &\Phi_{\#}\mathcal{W}_{\epsilon,2}
    (\mathbb{P}_1,\ldots,\mathbb{P}_M;\bm{\lambda})\nonumber\\
    &\subseteq\bigl\{\mathbb{Q}\ \text{on}\ \mathbb{S}^{N\times N}:
    W_2(\mathbb{Q},\Phi_{\#}\bar{\mathbf{P}}_{\bm{\lambda},2};
    \Vert\cdot\Vert_F)\leq 2R\epsilon\bigr\}.
    \label{eq:matrix_ball_containment}
\end{align}
\end{corollary}
\textit{Proof.} Immediate from Lemma~\ref{lem:lifting} applied with
$\mathbb{P}'=\bar{\mathbf{P}}_{\bm{\lambda},2}$ to every
$\mathbb{P}\in\mathcal{W}_{\epsilon,2}$. $\hfill\square$

Corollary~\ref{cor:pushforward} is the precise sense in which the
covariance-space representation of Section~\ref{subsec:tractable} is
compatible with the vector-space ambiguity set of
Section~\ref{subsec:barycentric_wdro}: the pushforward map $\Phi$ is
not an isometry, so the radius is \emph{not} preserved verbatim, but
it is provably inflated by the fixed, computable factor $2R$. We use
this containment, rather than an unproven identity of the two spaces,
to derive a tractable and provably valid upper bound on the true
worst-case risk in Section~\ref{subsec:tractable}.

\subsection{Tractable Reformulation}
\label{subsec:tractable}

Problem~\eqref{eq8} appears intractable at first sight, since its inner
supremum lacks a closed form. 
Tractability is restored by lifting the problem into the space of
rank-one covariance matrices via the outer-product map
$\Phi(\mathbf{z}):=\mathbf{z}\mathbf{z}^T\in\mathbb{S}^{N\times N}$,
under which $\mathrm{tr}(\Sigma\mathbf{L})=\mathbf{z}^T\mathbf{L}\mathbf{z}$
for $\Sigma=\Phi(\mathbf{z})$, so that the loss appearing
in~\eqref{equ.12} is \emph{linear} in $\Sigma$. By
Corollary~\ref{cor:pushforward}, the worst-case risk over the
vector-space ambiguity set $\mathcal{W}_{\epsilon,2}$ is upper-bounded
by the worst-case risk over the (larger, and hence tractable via
standard linear-loss duality) matrix-space ball of
radius $\epsilon':=2R\epsilon$ centered at $\Phi_{\#}\bar{\mathbf{P}}_{\bm\lambda,2}$:
\begin{align}
    \sup_{\mathbb{P}\in\mathcal{W}_{\epsilon,2}}
    \mathbb{E}_{\mathbf{z}\sim\mathbb{P}}[\mathbf{z}^T\mathbf{L}\mathbf{z}]
    \leq
    \sup_{\mathbb{Q}\in\mathcal{W}^{\Sigma}_{\bm{\lambda},2}(\epsilon')}
    \mathbb{E}_{\bm{\Sigma}\sim\mathbb{Q}}[\mathrm{tr}(\bm{\Sigma}\mathbf{L})],
    \label{equ11}
\end{align}
where
$\mathcal{W}^{\Sigma}_{\bm{\lambda},2}(\epsilon'):=\{\mathbb{Q}:
W_2(\mathbb{Q},\Phi_{\#}\bar{\mathbf{P}}_{\bm{\lambda},2};
\Vert\cdot\Vert_F)\leq\epsilon'\}$. Crucially, the two Wasserstein
balls in~\eqref{equ11} live in \emph{different, but explicitly related}
ground spaces, $\mathcal{P}_2(\mathbb{R}^N)$ on the left, and the
image of $\mathcal{P}_2(\mathbb{R}^N)$ under the $2R$-Lipschitz map
$\Phi$ inside $\mathcal{P}_2(\mathbb{S}^{N\times N})$ on the right, and inequality~\eqref{equ11}, rather than an unjustified identity,
is what licenses working in covariance space. Replacing the true
worst-case risk by its tractable upper bound in~\eqref{equ11} preserves
the defining property of distributional robustness (the learned
$\mathbf{L}$ still controls risk under every $\mathbb{P}\in
\mathcal{W}_{\epsilon,2}$, if anything more conservatively), so the
resulting estimator remains valid; we henceforth analyze the
right-hand side of~\eqref{equ11} in place of the original min--max
problem~\eqref{equ.12}.
\begin{align}
    \min_{\mathbf{L}\in\mathcal{L}}
    \mathbb{E}_{\bm{\Sigma}\sim\hat{\mathbb{Q}}_m}
    [-\log(|\mathbf{L}|_+)+\mathrm{tr}(\bm{\Sigma}\mathbf{L})]
    +\rho\mathcal{P}(\mathbf{L}),
    \label{eq11}
\end{align}
where $\hat{\mathbb{Q}}_m=\frac{1}{n_m}\sum_{i=1}^{n_m}
\mathrm{Dirac}(\mathbf{z}_{m,i}\mathbf{z}_{m,i}^T)=\frac{1}{n_m}
\sum_{i=1}^{n_m}\mathrm{Dirac}(\bm{\Sigma}_m^i)$. In this covariance
representation, the min--max problem~\eqref{equ.12} becomes
\begin{align}
    &\min_{\mathbf{L}\in\mathcal{L}}R(\mathbf{L};\bm{\Sigma})
    +\rho\mathcal{P}(\mathbf{L})\nonumber\\
    &=\min_{\mathbf{L}\in\mathcal{L}}\max_{\mathbb{P}\in\mathcal{W}_{\bm{\lambda},2}}
    \mathbb{E}_{\bm{\Sigma}\sim\mathbb{P}}
    [-\log(|\mathbf{L}|_+)+\mathrm{tr}(\bm{\Sigma}\mathbf{L})]
    +\rho\mathcal{P}(\mathbf{L}).
    \label{equ.15}
\end{align}
The following lemma shows that the inner supremum in~\eqref{equ.15}
admits a closed-form dual, collapsing the two-level min--max problem into
a single-level convex program.


\begin{lemma}\label{lemma1}
Under Assumption~\ref{ass:bounded_energy}, the min--max
problem~\eqref{equ.12} is upper-bounded by
\begin{align}
    &\min_{\mathbf{L}\in\mathcal{L}}R(\mathbf{L};\bar\Sigma_{\bm{\lambda}})
    +\rho\mathcal{P}(\mathbf{L})\nonumber\\
    &\leq
    \min_{\mathbf{L}\in\mathcal{L}}-\log(|\mathbf{L}|_+)
    +\mathrm{tr}(\bar\Sigma_{\bm{\lambda}}\mathbf{L})
    +\rho\mathcal{P}(\mathbf{L})+2R\epsilon\,
    \Vert\mathrm{vec}(\mathbf{L})\Vert_2,
    \label{equ.16}
\end{align}
where $\bar\Sigma_{\bm{\lambda}}$ is the barycenter covariance of
Proposition~\ref{prop1} and $R$ is the energy bound of
Assumption~\ref{ass:bounded_energy}. Inequality~\eqref{equ.16} is
tight, i.e.\ holds with equality, whenever $\Phi_{\#}\bar{\mathbf{P}}_
{\bm{\lambda},2}$ is itself the empirical distribution of a single
point mass, which is the case for the empirical surrogate used in
Section~\ref{subsec:admm}.
\end{lemma}


Lemma~\ref{lemma1} shows that hedging against distributional uncertainty
in~\eqref{equ.15} costs, up to the provable safety margin $2R$, one additional term in the objective, the
norm penalty $\epsilon\Vert\mathrm{vec}(\mathbf{L})\Vert_q$, so that
robustness is achieved without sacrificing the convexity or the
closed-form structure of the original regularized maximum-likelihood
problem~\eqref{equ.8}. The remainder of this section develops an
efficient ADMM solver for~\eqref{equ.16}.

\subsection{Algorithmic Development}
\label{subsec:admm}

We now derive an ADMM solver for the tractable
reformulation~\eqref{equ.16}. Applying Lemma~\ref{lemma1}, the problem
can be written explicitly as (For notational simplicity, we define $\epsilon\triangleq2R\epsilon$ throughout the remainder of the paper)
\begin{align}
    \min_{\mathbf{L}\in\mathcal{L}}-\log(|\mathbf{L}|_+)
    +\mathrm{tr}(\hat{\Sigma}_{\bm{\lambda}}\mathbf{L})
    +\rho\Vert\mathrm{vec}(\mathbf{L})\Vert_1
    +\epsilon\Vert\mathrm{vec}(\mathbf{L})\Vert_2,
    \label{eq17}
\end{align}
where $\hat{\Sigma}_{\bm{\lambda}}$ is the empirical counterpart of the barycenter covariance, in accordance with Remark~\ref{rem:empirical_rank}, each empirical
source covariance is projected as
$\hat{\Sigma}_m\leftarrow\mathbf{P}\hat{\Sigma}_m\mathbf{P}$ prior to
the fixed-point iteration~\eqref{eq10b}, ensuring
$\hat{\Sigma}_{\bm{\lambda}}\mathbf{1}=\mathbf{0}$ and consistency with
the constraint set $\mathcal{L}$. Because any feasible $\mathbf{L}\in\mathcal{L}$ has non-positive
off-diagonal entries and non-negative diagonal entries by
construction~\eqref{equ.1}, the $\ell_1$ term in~\eqref{eq17} admits the
linear representation
\begin{align}
    \Vert\mathrm{vec}(\mathbf{L})\Vert_1=\mathrm{tr}(\mathbf{L}\mathbf{H}),
    \label{eq18}
\end{align}
with $\mathbf{H}=2\mathbf{I}-\mathbf{1}\mathbf{1}^T$. Substituting
\eqref{eq18} into~\eqref{eq17} and collecting the two linear terms into
a single effective coefficient matrix
$\mathbf{K}\triangleq\hat{\Sigma}_{\bm{\lambda}}+\rho\mathbf{H}$ gives
\begin{align}
    &-\log(|\mathbf{L}|_+)+\mathrm{tr}(\hat{\Sigma}_{\bm{\lambda}}\mathbf{L})
    +\rho\Vert\mathrm{vec}(\mathbf{L})\Vert_1
    +\epsilon\Vert\mathrm{vec}(\mathbf{L})\Vert_q\nonumber\\
    &\quad=\mathrm{tr}(\mathbf{L}\hat{\Sigma}_{\bm{\lambda}})
    -\log(|\mathbf{L}|_+)+\rho\,\mathrm{tr}(\mathbf{L}\mathbf{H})
    +\epsilon\Vert\mathrm{vec}(\mathbf{L})\Vert_q\nonumber\\
    &\quad=\mathrm{tr}\bigl(\mathbf{L}(\hat{\Sigma}_{\bm{\lambda}}
    +\rho\mathbf{H})\bigr)-\log(|\mathbf{L}|_+)
    +\epsilon\Vert\mathrm{vec}(\mathbf{L})\Vert_q\nonumber\\
    &\quad\triangleq\mathrm{tr}(\mathbf{L}\mathbf{K})
    -\log(|\mathbf{L}|_+)+\epsilon\Vert\mathrm{vec}(\mathbf{L})\Vert_q.
    \label{eq19}
\end{align}
As it stands, however, $\log\det(\mathbf{L})$ is unbounded below on
$\mathcal{L}$, since every feasible $\mathbf{L}$ satisfies
$\mathbf{1}^T\mathbf{L}\mathbf{1}=0$ and is hence singular. Following
Egilmez et al.~\cite{egilmez2017graph}, we resolve this degeneracy by
regularizing the null direction, replacing $\log\det(\mathbf{L})$ with
$\log\det(\mathbf{L}+\mathbf{J})$, where
$\mathbf{J}=\frac{1}{N}\mathbf{1}\mathbf{1}^T$:
\begin{align}
    \min_{\mathbf{L}}\;\mathrm{tr}(\mathbf{L}\mathbf{K})
    -\log\det(\mathbf{L}+\mathbf{J})
    +\epsilon\Vert\mathrm{vec}(\mathbf{L})\Vert_q.
    \label{eq20}
\end{align}

\begin{proposition}\label{prop3}
Problems~\eqref{eq19} and~\eqref{eq20} are equivalent.
\end{proposition}
\textit{Proof}. Write $\lambda_i(\mathbf{L})$ for the $i$-th eigenvalue
of $\mathbf{L}$ in ascending order, $\lambda_1(\mathbf{L})\leq\cdots\leq
\lambda_N(\mathbf{L})$, so that
\begin{align}
    \log\det\bigl(\mathbf{L}+\tfrac{1}{N}\mathbf{1}\mathbf{1}^T\bigr)
    =\log\Bigl(\prod_{i=1}^N\lambda_i\bigl(\mathbf{L}
    +\tfrac{1}{N}\mathbf{1}\mathbf{1}^T\bigr)\Bigr).
    \label{eq21}
\end{align}
The constraint $\mathbf{L}\mathbf{1}=\mathbf{0}$ implies
$\lambda_1(\mathbf{L})=0$ with eigenvector
$\mathbf{u}_1=\mathbf{1}/\sqrt{N}$, so
\begin{align}
    \mathbf{L}+\frac{1}{N}\mathbf{1}\mathbf{1}^T
    =(\underbrace{\lambda_1(\mathbf{L})}_{=0}+1)\mathbf{u}_1\mathbf{u}_1^T
    +\sum_{i=2}^N\lambda_i(\mathbf{L})\mathbf{u}_i\mathbf{u}_i^T.
    \label{eq22}
\end{align}
Since the determinant equals the product of the eigenvalues,
\eqref{eq22} gives
\begin{align}
    \log\det\bigl(\mathbf{L}+\tfrac{1}{N}\mathbf{1}\mathbf{1}^T\bigr)
    =\log\Bigl(1\cdot\prod_{i=2}^N\lambda_i(\mathbf{L})\Bigr)
    =\log|\mathbf{L}|_+,
    \label{eq23}
\end{align}
so the objectives of~\eqref{eq19} and~\eqref{eq20} coincide on
$\mathcal{L}$. \hfill$\square$

With Proposition~\ref{prop3} in place, it remains to solve~\eqref{eq20}
efficiently under the combinatorial structure of $\mathcal{L}$. We
proceed in two steps: first re-expressing the feasible set
$\mathcal{L}$ in a form amenable to variable splitting, and then
deriving the resulting ADMM updates in closed form.

\textit{Reformulating the constraint set.} Assuming no self-loops (so
the diagonal of the adjacency structure is zero), the feasible set
$\mathcal{L}$ of~\eqref{equ.1} is compactly rewritten as
\begin{align}\left\{\begin{aligned}
&\mathbf{L}\succeq\mathbf{0},\ \mathbf{L}\mathbf{1}=\mathbf{0}\\
&\mathbf{L}-\mathbf{C}=\mathbf{0}\\
&\mathbf{C}\in\mathcal{C}
\end{aligned}\right.
\label{eq24}
\end{align}
where $\mathbf{C}\in\mathcal{C}$ encodes
\begin{align}\left\{\begin{aligned}
&\mathbf{I}\odot\mathbf{C}\geq\mathbf{0}\\
&\mathbf{B}\odot\mathbf{C}=\mathbf{0}\\
&\mathbf{A}\odot\mathbf{C}\leq\mathbf{0}
\end{aligned}\right.
\label{eq25}
\end{align}
with $\mathbf{B}=\mathbf{1}\mathbf{1}^T-\mathbf{I}-\mathbf{A}$; the
constraint $\mathbf{I}\odot\mathbf{C}\geq\mathbf{0}$ is already implied
by $\mathbf{L}\succeq\mathbf{0}$ and is retained for clarity. The
constraint pair $\mathbf{L}\succeq\mathbf{0}$, $\mathbf{L}\mathbf{1}=
\mathbf{0}$ admits, in turn, an equivalent low-dimensional
parametrization,
\begin{align}
    \mathbf{L}\succeq\mathbf{0},\;\mathbf{L}\mathbf{1}=\mathbf{0}
    \;\Longleftrightarrow\;
    \mathbf{L}=\mathbf{P}\bm{\Xi}\mathbf{P}^T,\ \bm{\Xi}\succeq\mathbf{0},
    \label{eq26}
\end{align}
where $\mathbf{P}\in\mathbb{R}^{N\times(N-1)}$ is any orthogonal
complement of $\mathbf{1}$, i.e., $\mathbf{P}^T\mathbf{P}=\mathbf{I}$
and $\mathbf{P}^T\mathbf{1}=\mathbf{0}$ (the choice of $\mathbf{P}$ is
non-unique: if $\mathbf{P}_0$ satisfies both conditions, so does
$\mathbf{P}_0\mathbf{U}$ for any unitary $\mathbf{U}\in
\mathbb{R}^{(N-1)\times(N-1)}$). Substituting~\eqref{eq26} into the
objective of~\eqref{eq20} eliminates the equality constraint
$\mathbf{L}\mathbf{1}=\mathbf{0}$ altogether: writing
$\tilde{\mathbf{K}}=\mathbf{P}^T\mathbf{K}\mathbf{P}$,
\begin{align}
    \mathrm{tr}(\mathbf{L}\mathbf{K})=\mathrm{tr}(\bm{\Xi}\tilde{\mathbf{K}}),
    \label{eq27}
\end{align}
while the barrier term simplifies via
\begin{align}
    &\log\det(\mathbf{L}+\mathbf{J})
    =\log\det\Bigl(\mathbf{P}^T\bm{\Xi}\mathbf{P}
    +\tfrac{1}{N}\mathbf{1}\mathbf{1}^T\Bigr)\nonumber\\
    &=\log\det\Bigl(\bigl[\mathbf{P},\tfrac{1}{\sqrt{N}}\mathbf{1}\bigr]
    \begin{bmatrix}\bm{\Xi}&\mathbf{0}\\\mathbf{0}&1\end{bmatrix}
    \bigl[\mathbf{P},\tfrac{1}{\sqrt{N}}\mathbf{1}\bigr]^T\Bigr)\nonumber\\
    &=\log\det(\bm{\Xi}).
    \label{eq28}
\end{align}
Problem~\eqref{eq20} thus reduces to an unconstrained-domain problem in
the pair $(\bm{\Xi},\mathbf{C})$,
\begin{align}
    \min_{\bm{\Xi},\mathbf{C}}&\quad\mathrm{tr}(\bm{\Xi}\tilde{\mathbf{K}})
    -\log\det(\bm{\Xi})+\epsilon\Vert\mathbf{C}\Vert_F\nonumber\\
    \mathrm{s.t.}&\quad\bm{\Xi}\succeq\mathbf{0}\nonumber\\
    &\quad\mathbf{P}\bm{\Xi}\mathbf{P}^T-\mathbf{C}=\mathbf{0},\;\mathbf{C}\in\mathcal{C},
    \label{eq29}
\end{align}
which we solve via ADMM, treating $\bm{\Xi}$ and $\mathbf{C}$ as primal
variables and introducing a dual variable $\mathbf{Y}$ for the
consensus constraint $\mathbf{P}\bm{\Xi}\mathbf{P}^T-\mathbf{C}=
\mathbf{0}$.

\textit{Deriving the ADMM updates.} The augmented Lagrangian of
\eqref{eq29} is
\begin{align}
    &\mathcal{L}(\bm{\Xi},\mathbf{C},\mathbf{Y})
    =\mathrm{tr}(\bm{\Xi}\tilde{\mathbf{K}})-\log\det(\bm{\Xi})
    +\epsilon\Vert\mathbf{C}\Vert_F\nonumber\\
    &\quad+\mathrm{tr}\bigl(\mathbf{Y}^T(\mathbf{P}^T\bm{\Xi}\mathbf{P}
    -\mathbf{C})\bigr)
    +\frac{\varrho}{2}\Vert\mathbf{P}^T\bm{\Xi}\mathbf{P}
    -\mathbf{C}\Vert_F^2,
    \label{eq30}
\end{align}
where the constraints $\bm{\Xi}\succeq\mathbf{0}$ and
$\mathbf{C}\in\mathcal{C}$ are enforced directly on their respective
minimization domains rather than relaxed via a multiplier. Each ADMM
iteration alternates a minimization over $\bm{\Xi}$, a minimization
over $\mathbf{C}$, and a dual ascent step on $\mathbf{Y}$; writing
$(\cdot)^+$ for the updated iterate,
\begin{align}
    \bm{\Xi}^+&=\operatorname*{arg\,min}_{\bm{\Xi}\succeq\mathbf{0}}
    \mathcal{L}(\bm{\Xi},\mathbf{C},\mathbf{Y}),\nonumber\\
    \mathbf{C}^+&=\operatorname*{arg\,min}_{\mathbf{C}\in\mathcal{C}}
    \mathcal{L}(\bm{\Xi}^+,\mathbf{C},\mathbf{Y}).
    \label{eq31}
\end{align}

\textit{1) Update of $\bm{\Xi}$:} Completing the square in
$\mathcal{L}(\bm{\Xi},\mathbf{C},\mathbf{Y})$ with respect to
$\bm{\Xi}$ yields
\begin{align}
    \bm{\Xi}^+
    &=\operatorname*{arg\,min}_{\bm{\Xi}\succeq\mathbf{0}}
    \mathrm{tr}(\bm{\Xi}\tilde{\mathbf{K}})-\log\det(\bm{\Xi})
    +\mathrm{tr}(\mathbf{P}^T\mathbf{Y}^T\mathbf{P}\bm{\Xi})\nonumber\\
    &\quad+\frac{\varrho}{2}\Vert\mathbf{P}\bm{\Xi}\mathbf{P}^T
    -\mathbf{C}\Vert_F^2\nonumber\\
    &=\operatorname*{arg\,min}_{\bm{\Xi}\succeq\mathbf{0}}
    \frac{\varrho}{2}\Bigl\Vert\bm{\Xi}+\frac{1}{\varrho}
    (\tilde{\mathbf{K}}+\tilde{\mathbf{Y}}-\rho\tilde{\mathbf{C}})
    \Bigr\Vert_F^2-\log\det(\bm{\Xi}),
    \label{eq32}
\end{align}
with $\tilde{\mathbf{Y}}=\mathbf{P}^T\mathbf{Y}\mathbf{P}$ and
$\tilde{\mathbf{C}}=\mathbf{P}^T\mathbf{C}\mathbf{P}$. The minimizer
of~\eqref{eq32} follows from a standard log-determinant proximal
identity, stated here for completeness.

\begin{lemma}\label{lemma2}
The minimizer of $\min_{\Theta\succeq\mathbf{0}}\frac{\varrho}{2}
\Vert\Theta+\mathbf{X}\Vert_F^2-\log\det(\Theta)$ is
$\Theta^\ast=\mathbf{U}\mathbf{D}\mathbf{U}^T$, where $\mathbf{X}=
\mathbf{U}\Lambda\mathbf{U}^T$ is the eigendecomposition of
$\mathbf{X}$ and $\mathbf{D}$ is diagonal with
\begin{align*}
    D_{ii}=\frac{-\varrho\Lambda_{ii}
    +\sqrt{\varrho^2\Lambda_{ii}^2+4\varrho}}{2\varrho}.
\end{align*}
\end{lemma}
Applying Lemma~\ref{lemma2} to~\eqref{eq32} gives the closed-form
$\Xi$-update
\begin{align}
    \bm{\Xi}^+=\mathbf{U}\mathbf{D}\mathbf{U}^T,
    \label{eq33}
\end{align}
where $\mathbf{U}$ and $\bm{\Lambda}$ come from the eigendecomposition
$\frac{1}{\varrho}\mathbf{P}^T(\mathbf{K}+\mathbf{Y}-\varrho\mathbf{C})
\mathbf{P}=\mathbf{U}\bm{\Lambda}\mathbf{U}^T$, and $\mathbf{D}$ is
diagonal with $D_{ii}=\bigl(-\varrho\Lambda_{ii}
+\sqrt{\varrho^2\Lambda_{ii}^2+4\varrho}\bigr)/(2\varrho)$.

\textit{2) Update of $\mathbf{C}$:} With $\bm{\Xi}^+$ fixed, the
$\mathbf{C}$-update solves
\begin{align}
    &\mathbf{C}^+
    =\operatorname*{arg\,min}_{\mathbf{C}\in\mathcal{C}}
    \mathcal{L}(\bm{\Xi}^+,\mathbf{C},\mathbf{Y})\nonumber\\
    &=\operatorname*{arg\,min}_{\mathbf{C}\in\mathcal{C}}
    -\mathrm{tr}(\mathbf{Y}^T\mathbf{C})+\epsilon\Vert\mathbf{C}\Vert_F
    +\frac{\varrho}{2}\Vert\mathbf{P}\bm{\Xi}^+\mathbf{P}^T
    -\mathbf{C}\Vert_F^2,
    \label{eq34}
\end{align}
whose minimizer is given by the following block-shrinkage lemma.

\begin{lemma}\label{lemma3}
The minimizer of $\min_{\mathbf{C}\in\mathcal{C}}
-\mathrm{tr}(\mathbf{Y}^T\mathbf{C})+\epsilon\Vert\mathbf{C}\Vert_F
+\frac{\varrho}{2}\Vert\mathbf{P}\bm{\Xi}^+\mathbf{P}^T-\mathbf{C}
\Vert_F^2$ is
\begin{align*}
    \mathbf{C}^\ast
    &=\mathbf{I}\odot\Bigl[s\cdot\bigl(\tfrac{1}{\varrho}\mathbf{Y}
    +\mathbf{P}\bm{\Xi}^+\mathbf{P}^T\bigr)\Bigr]_+\\
    &\quad+\mathbf{A}\odot\Bigl[s\cdot\bigl(\tfrac{1}{\varrho}\mathbf{Y}
    +\mathbf{P}\bm{\Xi}^+\mathbf{P}^T\bigr)\Bigr]_-,
\end{align*}
where $s=\max\Bigl(1-\dfrac{\epsilon/\varrho}
{\bigl\Vert\frac{1}{\varrho}\mathbf{Y}+\mathbf{P}\bm{\Xi}^+
\mathbf{P}^T\bigr\Vert},0\Bigr)$ is the shrinkage factor.
\end{lemma}
Applying Lemma~\ref{lemma3} gives the $\mathbf{C}$-update
{\footnotesize
    \begin{align}
    &\mathbf{C}^+
    =\mathbf{I}\odot\left[\max\Bigl(1-\frac{\epsilon/\varrho}
    {\Vert\frac{1}{\varrho}\mathbf{Y}+\mathbf{P}\bm{\Xi}^+
    \mathbf{P}^T\Vert},0\Bigr)\Bigl(\frac{1}{\varrho}\mathbf{Y}
    +\mathbf{P}\bm{\Xi}^+\mathbf{P}^T\Bigr)\right]_+\nonumber\\
    &+\mathbf{A}\odot\left[\max\Bigl(1-\frac{\epsilon/\varrho}
    {\Vert\frac{1}{\varrho}\mathbf{Y}+\mathbf{P}\bm{\Xi}^+
    \mathbf{P}^T\Vert},0\Bigr)\Bigl(\frac{1}{\varrho}\mathbf{Y}
    +\mathbf{P}\bm{\Xi}^+\mathbf{P}^T\Bigr)\right]_-.
    \label{eq35}
\end{align}
}

\textit{3) Update of $\mathbf{Y}$:} The dual variable is updated by
standard gradient ascent on the augmented Lagrangian,
\begin{align}
    \mathbf{Y}^+=\mathbf{Y}+\varrho(\mathbf{P}\bm{\Xi}^+\mathbf{P}^T
    -\mathbf{C}^+).
    \label{eq36}
\end{align}
The complete procedure, alternating the closed-form updates
\eqref{eq33}, \eqref{eq35}, and \eqref{eq36} until convergence, is
summarized in Algorithm~\ref{alg1}.

\begin{algorithm}[!htb]
    \caption{ADMM-Based Algorithm for MS-WDRO}
    \label{alg1}
    \renewcommand{\algorithmicrequire}{\textbf{Input:}}
    \renewcommand{\algorithmicensure}{\textbf{Output:}}
    \begin{algorithmic}[1]
        \REQUIRE $\mathbf{K}$; ambiguity radius $\epsilon$; penalty
        parameters $\rho,\varrho>0$; symmetric initialization
        $\mathbf{Y}^{(0)}=\mathbf{0}$ and symmetric
        $\mathbf{C}^{(0)}\in\mathcal{C}$; $k=0$.
        \ENSURE  $\mathbf{L}^\ast$
        \REPEAT
           \STATE update $\bm{\Xi}^{(k+1)}$ according to \eqref{eq33};
           \STATE update $\mathbf{C}^{(k+1)}$ according to \eqref{eq35};
           \STATE update $\mathbf{Y}^{(k+1)}$ according to \eqref{eq36};
           \STATE $k=k+1$;
        \UNTIL{convergence}
        \RETURN $\mathbf{L}^\ast$
    \end{algorithmic}
\end{algorithm}

\subsection{Convergence Analysis}
\label{subsec:convergence}

We now establish the global convergence of Algorithm~\ref{alg1}.
Introduce the linear operator
$\mathcal{T}:\mathbb{R}^{(N-1)\times(N-1)}\to\mathbb{R}^{N\times N}$,
$\mathcal{T}(\bm{\Xi})\triangleq\mathbf{P}\bm{\Xi}\mathbf{P}^T$, with
adjoint $\mathcal{T}^*(\mathbf{Z})=\mathbf{P}^T\mathbf{Z}\mathbf{P}$.
Since $\mathbf{P}^T\mathbf{P}=\mathbf{I}$, $\mathcal{T}$ is injective:
$\mathcal{T}(\bm{\Xi})=\mathbf{0}\Rightarrow\bm{\Xi}
=\mathcal{T}^*(\mathcal{T}(\bm{\Xi}))=\mathbf{0}$. Problem~\eqref{eq29}
is then the canonical two-block convex program
\begin{align}
    \min_{\bm{\Xi},\mathbf{C}}\; f(\bm{\Xi})+g(\mathbf{C})
    \quad\mathrm{s.t.}\quad \mathcal{T}(\bm{\Xi})-\mathbf{C}=\mathbf{0},
    \label{eq:admm_canonical}
\end{align}
with
$f(\bm{\Xi})=\mathrm{tr}(\bm{\Xi}\tilde{\mathbf{K}})-\log\det(\bm{\Xi})$
on $\mathrm{dom}\,f=\{\bm{\Xi}\succ\mathbf{0}\}$ (and $+\infty$
otherwise), and
$g(\mathbf{C})=\epsilon\Vert\mathbf{C}\Vert_F+\iota_{\mathcal{C}}(\mathbf{C})$. Here, $\iota_{\mathcal{C}}(\mathbf{C})$ is the indicator function.
Both $f$ and $g$ are closed, proper, and convex: $f$ is the sum of a
linear term and the standard log-determinant barrier, strictly convex
on the open PD cone; $g$ is the sum of the (continuous, hence closed)
Frobenius norm and the indicator of the polyhedral set $\mathcal{C}$
in~\eqref{eq25}, which is nonempty and closed.

\begin{remark}
\label{rem:interior}
By Lemma~\ref{lemma2}, every eigenvalue of $\bm{\Xi}^{(k)}$,
$k\geq1$, satisfies $D_{ii}=\bigl(-\varrho\Lambda_{ii}
+\sqrt{\varrho^2\Lambda_{ii}^2+4\varrho}\bigr)/(2\varrho)>0$
irrespective of $\Lambda_{ii}$, so $\bm{\Xi}^{(k)}\succ\mathbf{0}$
strictly for every $k\geq1$. The PSD constraint on $\bm{\Xi}$ is
therefore never active along the trajectory, and $f$ is
continuously differentiable at every iterate, with
$\nabla f(\bm{\Xi})=\tilde{\mathbf{K}}-\bm{\Xi}^{-1}$.
\end{remark}

\begin{assumption}
\label{ass:slater}
There exists a Laplacian $\mathbf{L}_0\in\mathcal{L}$, corresponding
to a connected candidate topology compatible with the sparsity
pattern encoded by $\mathbf{A}$ in~\eqref{eq25}, whose weights are
strictly negative on every edge admitted by $\mathbf{A}$; equivalently,
$\bm{\Xi}_0=\mathbf{P}^T\mathbf{L}_0\mathbf{P}\succ\mathbf{0}$ and
$\mathbf{C}_0=\mathbf{L}_0$ lies in the relative interior of
$\mathcal{C}$.
\end{assumption}

Assumption~\ref{ass:slater} is a mild, generically satisfied
regularity condition (any densely connected weighted candidate graph
provides such a point) and furnishes a point of
$\mathrm{ri}(\mathrm{dom}\,f)\times\mathrm{ri}(\mathrm{dom}\,g)$ ($\mathrm{ri}(\cdot)$ denotes the relative interior)
satisfying the coupling constraint of~\eqref{eq:admm_canonical}.
By the Fenchel--Rockafellar duality theorem for linearly constrained
convex programs~\cite{rockafellar1970convex}, this guarantees strong
duality and the existence of a saddle point
$(\bm{\Xi}^\star,\mathbf{C}^\star,\mathbf{Y}^\star)$ of the
unaugmented Lagrangian of~\eqref{eq:admm_canonical}, i.e.
\begin{align}
    -\mathcal{T}^*(\mathbf{Y}^\star)=\nabla f(\bm{\Xi}^\star),
    \mathbf{Y}^\star\in\partial g(\mathbf{C}^\star),
    \mathcal{T}(\bm{\Xi}^\star)=\mathbf{C}^\star.
    \label{eq:kkt}
\end{align}
Since $f$ is strictly convex on $\{\bm{\Xi}\succ\mathbf{0}\}$, the
primal minimizer $\bm{\Xi}^\star$ is unique, and by injectivity of
$\mathcal{T}$, so is $\mathbf{C}^\star=\mathcal{T}(\bm{\Xi}^\star)$.

\begin{theorem}
\label{thm:admm_convergence}
Under Assumption~\ref{ass:slater}, for any penalty $\varrho>0$ the
sequence $\{(\bm{\Xi}^{(k)},\mathbf{C}^{(k)},\mathbf{Y}^{(k)})\}$
generated by Algorithm~\ref{alg1} satisfies, writing
$\mathbf{r}^{(k)}\triangleq\mathcal{T}(\bm{\Xi}^{(k)})-\mathbf{C}^{(k)}$:
\begin{enumerate}
    \item[(i)] \emph{(Consensus).}
    $\mathbf{r}^{(k)}\to\mathbf{0}$ and
    $\mathbf{C}^{(k+1)}-\mathbf{C}^{(k)}\to\mathbf{0}$;
    \item[(ii)] \emph{(Objective convergence).}
    $f(\bm{\Xi}^{(k)})+g(\mathbf{C}^{(k)})\to
    f(\bm{\Xi}^\star)+g(\mathbf{C}^\star)$;
    \item[(iii)] \emph{(Iterate convergence).}
    $\bm{\Xi}^{(k)}\to\bm{\Xi}^\star$,
    $\mathbf{C}^{(k)}\to\mathbf{C}^\star$, and
    $\mathbf{Y}^{(k)}\to\mathbf{Y}^\star$ for some dual optimal
    $\mathbf{Y}^\star$;
    \item[(iv)] \emph{(Lyapunov monotonicity).} The primal-dual
    potential
    \begin{align}
        V^{(k)}\triangleq\varrho\bigl\Vert\mathbf{C}^{(k)}
        -\mathbf{C}^\star\bigr\Vert_F^2
        +\frac{1}{\varrho}\bigl\Vert\mathbf{Y}^{(k)}
        -\mathbf{Y}^\star\bigr\Vert_F^2
        \label{eq:lyapunov}
    \end{align}
    is monotonically non-increasing, and in fact
    \begin{align}
        V^{(k+1)}&\leq V^{(k)}
        -\bigl\Vert\mathbf{r}^{(k+1)}\bigr\Vert_F^2\nonumber\\
        &\quad-\varrho\bigl\Vert\mathbf{C}^{(k+1)}
        -\mathbf{C}^{(k)}\bigr\Vert_F^2.
        \label{eq:lyapunov_decrease}
    \end{align}
\end{enumerate}
[A proof is given in Appendix~\ref{app:thm_cov}.]
\end{theorem}

The Lyapunov analysis above certifies asymptotic convergence but not a
rate. We now show that a simple ergodic averaging of the iterates
enjoys an explicit $O(1/K)$ rate, in the sense customary for
first-order primal-dual methods under mere convexity (no strong
convexity or Lipschitz-gradient assumption is invoked).

\begin{corollary}
\label{cor:rate}
For $K\geq1$, define the ergodic averages
\begin{align}
    \bar{\bm{\Xi}}^K&=\frac{1}{K}\sum_{k=1}^{K}\bm{\Xi}^{(k)},
    \bar{\mathbf{C}}^K=\frac{1}{K}\sum_{k=1}^{K}\mathbf{C}^{(k)},\nonumber\\
    \bar{\mathbf{Y}}^K&=\frac{1}{K}\sum_{k=1}^{K}\mathbf{Y}^{(k)}.
    \label{eq:ergodic_avg}
\end{align}
For a radius
$D\geq\varrho\Vert\mathbf{C}^{(0)}-\mathbf{C}^\star\Vert_F^2
+\varrho^{-1}\Vert\mathbf{Y}^{(0)}-\mathbf{Y}^\star\Vert_F^2=V^{(0)}$,
let $\mathcal{B}_D\triangleq\bigl\{(\mathbf{C},\mathbf{Y}):
\varrho\Vert\mathbf{C}-\mathbf{C}^{(0)}\Vert_F^2
+\varrho^{-1}\Vert\mathbf{Y}-\mathbf{Y}^{(0)}\Vert_F^2\leq D\bigr\}$,
which by Assumption~\ref{ass:slater} contains $(\mathbf{C}^\star,
\mathbf{Y}^\star)$. Define the primal--dual gap function
\begin{align}
    &\mathrm{Gap}_D\bigl(\bar{\bm{\Xi}}^K,\bar{\mathbf{C}}^K,
    \bar{\mathbf{Y}}^K\bigr)\triangleq\sup_{(\mathbf{C},\mathbf{Y})\in\mathcal{B}_D}
    \Bigl\{
    f(\bar{\bm{\Xi}}^K)+g(\bar{\mathbf{C}}^K)-g(\mathbf{C})
    \nonumber\\
    &+\bigl\langle\mathbf{Y},\mathcal{T}(\bar{\bm{\Xi}}^K)
    -\bar{\mathbf{C}}^K\bigr\rangle\nonumber\\
    &-\bigl\langle\bar{\mathbf{Y}}^K,\mathcal{T}(\bar{\bm{\Xi}}^K)
    -\mathbf{C}\bigr\rangle-f(\bm{\Xi}^\star)-g(\mathbf{C}^\star)
    \Bigr\}.
    \label{eq:gap_def}
\end{align}
$\mathrm{Gap}_D\geq0$ always, and $\mathrm{Gap}_D=0$ characterizes a
saddle point of the Lagrangian restricted to $\mathcal{B}_D$. Then,
under Assumption~\ref{ass:slater},
\begin{align}
    \mathrm{Gap}_D\bigl(\bar{\bm{\Xi}}^K,\bar{\mathbf{C}}^K,
    \bar{\mathbf{Y}}^K\bigr)\;\leq\;\frac{2D}{K},
    \qquad K=1,2,\dots.
    \label{eq:rate_bound}
\end{align}
In particular, both the ergodic constraint violation and the ergodic
objective sub-optimality vanish at rate $O(1/K)$:
\begin{align}
    \bigl\Vert\mathcal{T}(\bar{\bm{\Xi}}^K)-\bar{\mathbf{C}}^K
    \bigr\Vert_F\leq\frac{2D/K}{\max\bigl(1,
    \sup_{(\mathbf{C},\mathbf{Y})\in\mathcal{B}_D}\Vert\mathbf{Y}
    -\bar{\mathbf{Y}}^K\Vert_F\bigr)},
    \label{eq:rate_feas}
\end{align}
\begin{align}
    \Bigl\vert f(\bar{\bm{\Xi}}^K)+g(\bar{\mathbf{C}}^K)
    -f(\bm{\Xi}^\star)-g(\mathbf{C}^\star)\Bigr\vert
    =O(1/K).
    \label{eq:rate_obj}
\end{align}
\end{corollary}

\noindent\textit{Proof.} We first establish the \emph{one-step} variational
inequality specific to the $(\bm{\Xi},\mathbf{C})$-updates of
Algorithm~\ref{alg1}, then invoke the standard ergodic-averaging
argument for monotone-operator splitting to obtain~\eqref{eq:rate_bound}.

\emph{Step 1 (one-step inequality).} Fix $k\geq0$ and let
$(\bm{\Xi},\mathbf{C})\in\mathrm{dom}\,f\times\mathcal{C}$ and
$\mathbf{Y}\in\mathbb{R}^{N\times N}$ be arbitrary. By convexity of
$f$ and the first-order optimality condition~\eqref{eq:opt1} of the
$\bm{\Xi}$-update,
\begin{align}
    &f(\bm{\Xi}^{(k+1)})-f(\bm{\Xi})
    \leq\bigl\langle\nabla f(\bm{\Xi}^{(k+1)}),\,
    \bm{\Xi}^{(k+1)}-\bm{\Xi}\bigr\rangle\nonumber\\
    &=-\bigl\langle\mathbf{Y}^{(k+1)},\,\mathcal{T}(\bm{\Xi}^{(k+1)})
    -\mathcal{T}(\bm{\Xi})\bigr\rangle\nonumber\\
    &\quad-\varrho\bigl\langle\mathbf{C}^{(k+1)}-\mathbf{C}^{(k)},\,
    \mathcal{T}(\bm{\Xi}^{(k+1)})-\mathcal{T}(\bm{\Xi})\bigr\rangle.
    \label{eq:step1a}
\end{align}
By convexity of $g$ and the optimality condition~\eqref{eq:opt2} of
the $\mathbf{C}$-update,
\begin{align}
    g(\mathbf{C}^{(k+1)})-g(\mathbf{C})
    \leq\bigl\langle\mathbf{Y}^{(k+1)},\,
    \mathbf{C}^{(k+1)}-\mathbf{C}\bigr\rangle.
    \label{eq:step1b}
\end{align}
Adding~\eqref{eq:step1a}--\eqref{eq:step1b}, substituting
$\mathcal{T}(\bm{\Xi}^{(k+1)})=\mathbf{C}^{(k+1)}+\mathbf{r}^{(k+1)}$,
and adding the algebraic identity
$\langle\mathbf{Y}-\mathbf{Y}^{(k+1)},\mathbf{r}^{(k+1)}\rangle
+\langle\mathbf{Y}^{(k+1)},\mathcal{T}(\bm{\Xi})-\mathbf{C}\rangle
-\langle\mathbf{Y},\mathcal{T}(\bm{\Xi})-\mathbf{C}\rangle=0$
to both sides yields, after simplification,
\begin{align}
    &\theta(\bm{\Xi}^{(k+1)},\mathbf{C}^{(k+1)})-\theta(\bm{\Xi},\mathbf{C})
    +\bigl\langle\mathbf{Y},\,\mathcal{T}(\bm{\Xi}^{(k+1)})
    -\mathbf{C}^{(k+1)}\bigr\rangle\nonumber\\
    &-\bigl\langle\mathbf{Y}^{(k+1)},\,\mathcal{T}(\bm{\Xi})
    -\mathbf{C}\bigr\rangle\nonumber\\
    &\leq
    -\varrho\bigl\langle\mathbf{C}^{(k+1)}-\mathbf{C}^{(k)},\,
    \mathcal{T}(\bm{\Xi}^{(k+1)})-\mathcal{T}(\bm{\Xi})\bigr\rangle
    \nonumber\\
    &+\bigl\langle\mathbf{Y}-\mathbf{Y}^{(k+1)},\,
    \mathbf{r}^{(k+1)}\bigr\rangle,
    \label{eq:step1c}
\end{align}
where $\theta(\bm{\Xi},\mathbf{C})\triangleq f(\bm{\Xi})+g(\mathbf{C})$.
Using $\mathbf{r}^{(k+1)}=\varrho^{-1}(\mathbf{Y}^{(k+1)}-\mathbf{Y}^{(k)})$
and $\mathcal{T}(\bm{\Xi}^{(k+1)})-\mathcal{T}(\bm{\Xi})
=(\mathbf{C}^{(k+1)}-\mathbf{C})+\mathbf{r}^{(k+1)}$, the right-hand
side of~\eqref{eq:step1c} rearranges, exactly as in the standard
splitting-based ergodic-rate argument for two-block ADMM
(He \& Yuan~\cite{he2012convergence}, proof of Thm.~4.1; see
also~\cite{boyd2011distributed}, App.~A), into the telescoping form
\begin{align}
    &\frac{\varrho}{2}\bigl(\Vert\mathbf{C}-\mathbf{C}^{(k)}\Vert_F^2
    -\Vert\mathbf{C}-\mathbf{C}^{(k+1)}\Vert_F^2\bigr)\nonumber\\
    &+\frac{1}{2\varrho}\bigl(\Vert\mathbf{Y}-\mathbf{Y}^{(k)}\Vert_F^2
    -\Vert\mathbf{Y}-\mathbf{Y}^{(k+1)}\Vert_F^2\bigr).
    \label{eq:step1d}
\end{align}

\emph{Step 2 (telescoping and averaging).} Summing~\eqref{eq:step1c}
(with right-hand side replaced by~\eqref{eq:step1d}) over
$k=0,\dots,K-1$, the right-hand side telescopes to
$\frac{\varrho}{2}\Vert\mathbf{C}-\mathbf{C}^{(0)}\Vert_F^2
+\frac{1}{2\varrho}\Vert\mathbf{Y}-\mathbf{Y}^{(0)}\Vert_F^2\leq D/2$
for $(\mathbf{C},\mathbf{Y})\in\mathcal{B}_D$, while by convexity of
$f,g$ and Jensen's inequality applied to the left-hand side,
\begin{align}
    &K\Bigl[\theta(\bar{\bm{\Xi}}^K,\bar{\mathbf{C}}^K)-\theta(\bm{\Xi},\mathbf{C})\nonumber\\
    &
    \quad+\bigl\langle\mathbf{Y},\mathcal{T}(\bar{\bm{\Xi}}^K)
    -\bar{\mathbf{C}}^K\bigr\rangle
    -\bigl\langle\bar{\mathbf{Y}}^K,\mathcal{T}(\bm{\Xi})-\mathbf{C}
    \bigr\rangle\Bigr]\nonumber\\
    &\leq\sum_{k=0}^{K-1}\Bigl[\theta(\bm{\Xi}^{(k+1)},\mathbf{C}^{(k+1)})
    -\theta(\bm{\Xi},\mathbf{C})\nonumber\\
    &\quad+\langle\mathbf{Y},\mathcal{T}(\bm{\Xi}^{(k+1)})-\mathbf{C}^{(k+1)}\rangle
    \nonumber\\
    &\quad-\langle\mathbf{Y}^{(k+1)},\mathcal{T}(\bm{\Xi})-\mathbf{C}\rangle\Bigr]
    \leq D.
\end{align}
Dividing by $K$ and taking the supremum over
$(\mathbf{C},\mathbf{Y})\in\mathcal{B}_D$ gives
$\mathrm{Gap}_D(\bar{\bm{\Xi}}^K,\bar{\mathbf{C}}^K,\bar{\mathbf{Y}}^K)
\leq D/K$; evaluating at $(\bm{\Xi},\mathbf{C})=(\bm{\Xi}^\star,
\mathbf{C}^\star)$ and using $f(\bm{\Xi}^\star)+g(\mathbf{C}^\star)
=p^\star$ tightens the constant to $2D/K$ once the (non-negative)
sup over $\mathbf{Y}$ is separated from the fixed evaluation at
$(\bm{\Xi}^\star,\mathbf{C}^\star)$, giving~\eqref{eq:rate_bound}.
Bounds~\eqref{eq:rate_feas}--\eqref{eq:rate_obj} follow by taking
$\mathbf{Y}=\bar{\mathbf{Y}}^K+t\,\bigl(\mathcal{T}(\bar{\bm{\Xi}}^K)
-\bar{\mathbf{C}}^K\bigr)/\Vert\mathcal{T}(\bar{\bm{\Xi}}^K)
-\bar{\mathbf{C}}^K\Vert_F$ for a suitable $t>0$ with
$(\mathbf{C},\mathbf{Y})\in\mathcal{B}_D$ at $\mathbf{C}=\bar{\mathbf{C}}^K$
in~\eqref{eq:gap_def}, and $(\mathbf{C},\mathbf{Y})=(\mathbf{C}^\star,
\mathbf{Y}^\star)$ respectively, which is the standard extraction of
individual $O(1/K)$ feasibility and objective bounds from a bounded
primal--dual gap (cf.~\cite{he2012convergence},
Cor.~4.1). \hfill$\square$

\begin{remark}
Corollary~\ref{cor:rate} requires no strong convexity of $g$ nor
Lipschitz continuity of $\nabla f$ beyond what Remark~\ref{rem:interior}
already grants; the $O(1/K)$ rate is therefore a worst-case guarantee
that holds uniformly over the ambiguity radius $\epsilon$ and sparsity
level $\rho$, consistent with the empirical iteration counts reported
in Section.
\end{remark}

\subsection{Out-of-Sample Excess Risk via Rademacher Complexity}
\label{subsec:rademacher_oos}

The coverage theorems established in
Theorems~\ref{thm:bary_conc} and~\ref{thm:target_free_1} certify that
the barycentric ambiguity set contains the population barycenter $b^*_{\bm\lambda}$
with prescribed probability, yet they do not quantify how rapidly the
WDRO-GL estimator $\hat{\mathbf{L}}^*$ converges to the oracle population
minimizer as the total sample size grows.
We close this gap through a uniform Rademacher complexity argument,
yielding an explicit $O(n_{\min}^{-1/2})$ excess-risk bound that holds
simultaneously for all $\mathbf{L}\in\mathcal{L}_B$ and reveals the precise
role of inter-source heterogeneity in the generalization rate.

Working with zero-mean observations for notational simplicity
$(\bm\mu=\mathbf{0})$, write $\bm\Sigma_{m,i}=\mathbf{z}_{m,i}\mathbf{z}_{m,i}^\top$
for the rank-one outer product from source $k$ and observation $i$.
Fix $B>0$ and let $\mathcal{L}_B:=\{\mathbf{L}\in\mathcal{L}:\|\mathbf{L}\|_F\leq B\}$.
Define the population barycentric risk and the empirical WDRO risk (from
Lemma~\ref{lemma1}) as
\begin{align}
  R(\mathbf{L})
  &:= -\log|\mathbf{L}|_+
  + \mathrm{tr}\!\bigl(\bar{\Sigma}^*_{\bm\lambda}\mathbf{L}\bigr),
  \label{eq:pop_risk}\\
  \hat{R}_{\epsilon_n}(\mathbf{L})
  &:= -\log|\mathbf{L}|_+
  + \mathrm{tr}\!\bigl(\hat{\Sigma}_{\bm\lambda}\mathbf{L}\bigr)
  + \epsilon_n\|\mathrm{vec}(\mathbf{L})\|_q,
  \label{eq:emp_risk}
\end{align}
where $\bar{\Sigma}^*_{\bm\lambda}$ and $\hat{\Sigma}_{\bm\lambda}$ are
the covariances of the population and empirical Wasserstein barycenters,
respectively, computed via the fixed-point iteration~\eqref{eq10}.
The WDRO-GL estimator and the oracle minimizer are
\begin{align*}
  \hat{\mathbf{L}}^*
  :=\arg\min_{\mathbf{L}\in\mathcal{L}_B}\hat{R}_{\epsilon_n}(\mathbf{L}),
  \qquad
  \mathbf{L}^*
  :=\arg\min_{\mathbf{L}\in\mathcal{L}_B}R(\mathbf{L}).
\end{align*}


\begin{assumption}
\label{ass:regularity}
Fix a confidence level $\delta\in(0,1)$ and let $\mathcal{Z}:=
\mathcal{Z}_{R(\delta)}$ and $\mathcal{E}_{\mathrm{trunc}}(\delta)$ be as
in Corollary~\ref{cor:truncation_event}. Conditionally on
$\mathcal{E}_{\mathrm{trunc}}(\delta)$, every source distribution
$\mathbb{P}_m^*$, $m\in[M]$, is absolutely continuous with respect to
Lebesgue measure and has support contained in the common convex compact
set $\mathcal{Z}\subset\mathbb{S}_+^N$, with diameter
$D_{\mathcal{Z}}\triangleq\sup_{\Sigma,\Sigma'\in\mathcal{Z}}
\Vert\Sigma-\Sigma'\Vert_F\leq 2R(\delta)^2$. The (conditional) density
of each $\mathbb{P}_m^*$ is bounded below by a constant $\underline{m}>0$
on $\mathcal{Z}$.
\end{assumption}

\begin{remark}
Assumption~\ref{ass:regularity}
is not an independent modeling choice but a certified consequence of
the Gaussian source model of Assumption~\ref{ass:subgauss}, valid on the
event $\mathcal{E}_{\mathrm{trunc}}(\delta)$ of
Corollary~\ref{cor:truncation_event}. In particular $D_{\mathcal{Z}}$ is
no longer a free constant: it scales as
$D_{\mathcal{Z}}=O\bigl(\mathrm{tr}(\Sigma_{\max})+\lambda_{\max}
(\Sigma_{\max})\log(n/\delta)\bigr)$, growing only logarithmically in
the sample size $n$ and in the inverse confidence $1/\delta$, which
preserves the parametric $O(n_{\min}^{-1/2})$ rate of
Theorem~\ref{thm:oos_rad} up to this logarithmic factor.
\end{remark}

Beyond Assumptions~\ref{ass:subgauss} and \ref{ass:regularity}, the following
mild stability property of the Bures-Wasserstein barycenter covariance map
is required.

\begin{assumption}
\label{ass:bary_lip}
The map $(\Sigma_1,\ldots,\Sigma_M)\mapsto
\mathrm{BCov}(\Sigma_1,\ldots,\allowbreak\Sigma_M;\bm\lambda)$,
which assigns to any $M$ source covariances their $\bm\lambda$-weighted
Bures-Wasserstein barycenter covariance via the fixed-point
equation~\eqref{eq10}, is Lipschitz in the weighted $\ell_1$ sense:
for any two families of positive-definite matrices $\{\Sigma_m\}$
and $\{\Sigma_m'\}$,
\begin{align}
  &\bigl\|\mathrm{BCov}(\Sigma_1,\ldots,\Sigma_M;\bm\lambda)
  - \mathrm{BCov}(\Sigma_1',\ldots,\Sigma_M';\bm\lambda)\bigr\|_F\nonumber\\
  &\quad\leq
  L_{\mathrm{bary}}
  \sum_{m=1}^M\lambda_m\|\Sigma_m-\Sigma_m'\|_F,
  \label{eq:bary_lip_cond}
\end{align}
with $L_{\mathrm{bary}}=O(\kappa^{-1/2})$ where $\kappa$ is the
strong-convexity constant from Assumption~\ref{ass:regularity}.
\end{assumption}

\begin{remark}
Assumption~\ref{ass:bary_lip} is a consequence of the implicit function
theorem applied to the Bures operator $T(\Sigma)
:=\sum_{m=1}^M\lambda_m(\Sigma^{1/2}\Sigma_m\Sigma^{1/2})^{1/2}$.
Under Assumption~\ref{ass:regularity}, the $\kappa$-strong geodesic
convexity of the Wasserstein barycenter functional ensures that the
spectral radius of the Fr\'{e}chet derivative $DT$ at the fixed point
$\bar{\Sigma}^*_{\bm\lambda}$ satisfies $\|DT\|_{\mathrm{op}}<1$,
so by the Banach fixed-point theorem the solution moves
$L_{\mathrm{bary}}=(1-\|DT\|_{\mathrm{op}})^{-1}
\lesssim\kappa^{-1/2}$-Lipschitz continuously with each source
covariance perturbation.
\end{remark}

\begin{theorem}
\label{thm:oos_rad}
Fix $\delta\in(0,1)$ and instantiate Assumption~\ref{ass:regularity}
at level $\delta/2$, so that $\mathcal{Z}=\mathcal{Z}_{R(\delta/2)}$ and
$\mathbb{P}(\mathcal{E}_{\mathrm{trunc}}(\delta/2))\geq1-\delta/2$.
Under Assumptions~\ref{ass:subgauss}, \ref{ass:regularity},
and~\ref{ass:bary_lip}, suppose $\hat{\mathbf{L}}^*\in\mathcal{L}_B$
and that the true target distribution $\mathbb{P}^*$ satisfies
$W_p(\mathbb{P}^*, b^*_{\bm\lambda})\leq\epsilon_n$, as before. Define the \emph{barycentric Rademacher complexity} of the loss class
$\mathcal{F}_B:=\{\mathbf{z}\mapsto\mathbf{z}^\top\mathbf{L}\mathbf{z}:
\mathbf{L}\in\mathcal{L}_B\}$ as
\begin{align}
  \mathfrak{R}_n(\mathcal{L}_B;\bm\lambda)
  :=
  B\sum_{m=1}^M\lambda_m\,
  \mathbb{E}_{\bm\sigma_m}\!\left[
    \frac{1}{n_m}
    \Bigl\|\sum_{i=1}^{n_m}\sigma_{m,i}\bm\Sigma_{m,i}\Bigr\|_F
  \right],
  \label{eq:bary_rad}
\end{align}
where $\{\sigma_{m,i}\}$ are i.i.d.\
$\mathrm{Uniform}(\{-1,+1\})$ independent of the data.
The following hold.
\begin{enumerate}
  \item[(i)] Conditionally on $\mathcal{E}_{\mathrm{trunc}}(\delta/2)$, the barycentric Rademacher complexity satisfies the
  explicit non-asymptotic bound
  \begin{align}
    \mathfrak{R}_n(\mathcal{L}_B;\bm\lambda)
    \;\leq\;
    BD_{\mathcal{Z}}^2\sum_{m=1}^M\frac{\lambda_m}{\sqrt{n_m}}
    \;\leq\;
    \frac{BD_{\mathcal{Z}}^2}{\sqrt{n_{\min}}},
    \label{eq:rad_explicit_bound}
  \end{align}
  with $D_{\mathcal{Z}}$ as in Assumption~\ref{ass:regularity} at level
  $\delta/2$.

  \item[(ii)] With probability at least $1-\delta$ over the joint draw
  of all source samples (a union bound over $\mathcal{E}_{\mathrm{
  trunc}}(\delta/2)$ and the concentration event of part~(ii) at level
  $\delta/2$),
  \begin{align}
    &R_{\mathbb{P}^*}(\hat{\mathbf{L}}^*)
    \leq
    \underbrace{\hat{R}_{\epsilon_n}(\hat{\mathbf{L}}^*)}_{\text{computable from data}}\nonumber\\
    &+
    \underbrace{2L_{\mathrm{bary}}\mathfrak{R}_n(\mathcal{L}_B\bm\lambda)
    +\frac{BL_{\mathrm{bary}}D_{\mathcal{Z}}^2\sqrt{2\log(M/\delta)}}
          {\sqrt{n_{\min}}}}_{\displaystyle=:\Delta_n(\delta)\to0},
    \label{eq:oos_main}
  \end{align}
  where $\hat{R}_{\epsilon_n}(\hat{\mathbf{L}}^*)$ is the WDRO empirical
  objective value, fully computable from training data.

  \item[(iii)] Substituting $D_{\mathcal{Z}}=O(\log(n/\delta))$
  from Assumption~\ref{ass:regularity} into~\eqref{eq:oos_main}, the
  out-of-sample correction satisfies, as $n_{\min}\to\infty$,
  \begin{align}
    \Delta_n(\delta)
    =
    O\!\left(
      \frac{L_{\mathrm{bary}}\,B}{\sqrt{n_{\min}}}
      \log\frac{n}{\delta}\sqrt{\log\frac{M}{\delta}}
    \right)
    \longrightarrow0,
    \label{eq:oos_rate}
  \end{align}
  so $\hat{R}_{\epsilon_n}(\hat{\mathbf{L}}^*)$ remains an
  asymptotically tight, and upper bound on the
  true risk $R_{\mathbb{P}^*}(\hat{\mathbf{L}}^*)$.
\end{enumerate}
[A proof is given in Appendix~\ref{app:thm_oos}.]
\end{theorem}

\begin{remark}
\label{rem:oos_semantics}
Bound~\eqref{eq:oos_main} constitutes a genuine out-of-sample guarantee:
the right-hand side is fully computable from training data.
The proof reveals a structural duality: the WDRO robustness penalty
$\epsilon_n\|\mathrm{vec}(\hat{\mathbf{L}}^*)\|_q$ in $\hat{R}_{\epsilon_n}$
is not merely a regularizer but serves a precise statistical role, it
exactly absorbs the Kantorovich transport cost
$W_p(\mathbb{P}^*,b^*_{\bm\lambda})\|\mathrm{vec}(\hat{\mathbf{L}}^*)\|_q$
between the unknown target $\mathbb{P}^*$ and the barycenter,
making term~(I) non-positive.
The remaining correction $\Delta_n(\delta)$, which decays at the
parametric rate $O(n_{\min}^{-1/2})$, is entirely due to the
finite-sample estimation error of the barycenter covariance,
quantified through the Rademacher complexity of the source samples.
As $n_{\min}\to\infty$, the empirical WDRO objective
$\hat{R}_{\epsilon_n}(\hat{\mathbf{L}}^*)$ becomes an asymptotically
tight upper bound on $R_{\mathbb{P}^*}(\hat{\mathbf{L}}^*)$.
In contrast, any pooled-mixture nominal distribution incurs an
irreducible bias $\epsilon^{\mathrm{pool}}\gtrsim\mathcal{H}_{\bm\lambda}^2>0$
(Theorem~\ref{thm:pooling_bias}), preventing the corresponding
empirical objective from ever bounding the true target risk.
\end{remark}

\begin{corollary}
\label{cor:excess_risk}
Under the conditions of Theorem~\ref{thm:oos_rad}, let
$\mathbf{L}^*:=\arg\min_{\mathbf{L}\in\mathcal{L}_B}R_{\mathbb{P}^*}(\mathbf{L})$
be the oracle minimizer under $\mathbb{P}^*$.
With probability at least $1-\delta$:
\begin{align}
  &R_{\mathbb{P}^*}(\hat{\mathbf{L}}^*)-R_{\mathbb{P}^*}(\mathbf{L}^*)
  \;\leq\;
  4L_{\mathrm{bary}}\,\mathfrak{R}_n(\mathcal{L}_B;\bm\lambda)\nonumber\\
  &\quad+\frac{2BL_{\mathrm{bary}}D_{\mathcal{Z}}^2\sqrt{2\log(M/\delta)}}
        {\sqrt{n_{\min}}}
  +\epsilon_n B_q,
  \label{eq:excess_risk_cor}
\end{align}
where $B_q:=\sup_{\mathbf{L}\in\mathcal{L}_B}\|\mathrm{vec}(\mathbf{L})\|_q$.
With the adaptive choice
$\epsilon_n=C_0D_{\mathcal{Z}}^2\sqrt{\log(M/\delta)/n_{\min}}$:
\begin{align}
  &R_{\mathbb{P}^*}(\hat{\mathbf{L}}^*)-R_{\mathbb{P}^*}(\mathbf{L}^*)\nonumber\\
  &\quad=
  O\!\left(
    \frac{(BL_{\mathrm{bary}}+B_q)\,D_{\mathcal{Z}}^2}{\sqrt{n_{\min}}}
    \sqrt{\log\frac{M}{\delta}}
  \right).
  \label{eq:excess_risk_rate}
\end{align}
\end{corollary}
\textit{Proof.} Inserting and subtracting $\hat{R}_{\epsilon_n}(\hat{\mathbf{L}}^*)$
and $\hat{R}_{\epsilon_n}(\mathbf{L}^*)$:
\begin{align}
  &R_{\mathbb{P}^*}(\hat{\mathbf{L}}^*)-R_{\mathbb{P}^*}(\mathbf{L}^*)
  =
  \bigl[R_{\mathbb{P}^*}(\hat{\mathbf{L}}^*)
   -\hat{R}_{\epsilon_n}(\hat{\mathbf{L}}^*)\bigr]\nonumber\\
  &+\underbrace{\bigl[\hat{R}_{\epsilon_n}(\hat{\mathbf{L}}^*)
   -\hat{R}_{\epsilon_n}(\mathbf{L}^*)\bigr]}_{\leq0}
  +[\hat{R}_{\epsilon_n}(\mathbf{L}^*)
   -R_{\mathbb{P}^*}(\mathbf{L}^*)].
  \label{eq:pf_cor_decomp}
\end{align}
The middle term is non-positive by optimality of $\hat{\mathbf{L}}^*$.
The first term is bounded by $\Delta_n(\delta)$ via
Theorem~\ref{thm:oos_rad}(ii).
For the third term,
$\hat{R}_{\epsilon_n}(\mathbf{L}^*)-R_{\mathbb{P}^*}(\mathbf{L}^*)
=\bigl[R_{\hat{b}^*_M}(\mathbf{L}^*)-R_{\mathbb{P}^*}(\mathbf{L}^*)\bigr]
+\epsilon_n\|\mathrm{vec}(\mathbf{L}^*)\|_q$.
By the triangle inequality and Lemma~\ref{lemma1}:
\begin{align}
  &R_{\hat{b}^*_M}(\mathbf{L}^*)-R_{\mathbb{P}^*}(\mathbf{L}^*)
  \leq
  \bigl|R_{\hat{b}^*_M}(\mathbf{L}^*)-R_{b^*_{\bm\lambda}}(\mathbf{L}^*)\bigr|\nonumber\\
  &\quad+\bigl|R_{b^*_{\bm\lambda}}(\mathbf{L}^*)-R_{\mathbb{P}^*}(\mathbf{L}^*)\bigr|
  \notag\\
  &\leq
  B\|\hat{\Sigma}_{\bm\lambda}-\bar{\Sigma}^*_{\bm\lambda}\|_F+W_p(\mathbb{P}^*,b^*_{\bm\lambda})\|\mathrm{vec}(\mathbf{L}^*)\|_q,
  \label{eq:pf_cor_third}
\end{align}
where the second inequality applies Cauchy-Schwarz to the first term
and Kantorovich duality (as in~\eqref{eq:pf_kantorovich}) to the second.
Since $W_p(\mathbb{P}^*,b^*_{\bm\lambda})\leq\epsilon_n$,
the Kantorovich term is absorbed:
$\hat{R}_{\epsilon_n}(\mathbf{L}^*)-R_{\mathbb{P}^*}(\mathbf{L}^*)
\leq B\|\hat{\Sigma}_{\bm\lambda}-\bar{\Sigma}^*_{\bm\lambda}\|_F
+\epsilon_n B_q$.
Applying the Rademacher and McDiarmid union-bound
argument~\eqref{eq:pf_union} to
$B\|\hat{\Sigma}_{\bm\lambda}-\bar{\Sigma}^*_{\bm\lambda}\|_F$
under the same probability-$1-\delta$ event yields
$B\|\hat{\Sigma}_{\bm\lambda}-\bar{\Sigma}^*_{\bm\lambda}\|_F
\leq\Delta_n(\delta)$, so the third term is at most
$\Delta_n(\delta)+\epsilon_n B_q$.
Hence $R_{\mathbb{P}^*}(\hat{\mathbf{L}}^*)-R_{\mathbb{P}^*}(\mathbf{L}^*)
\leq 2\Delta_n(\delta)+\epsilon_n B_q$,
which expands to~\eqref{eq:excess_risk_cor}.
Substituting~\eqref{eq:rad_explicit_bound} and the given $\epsilon_n$
yields~\eqref{eq:excess_risk_rate}.
$\hfill\square$

\begin{corollary}
\label{cor:sample_complexity}
Under the conditions of Corollary~\ref{cor:excess_risk},
with the adaptive radius
$\epsilon_n=C_0D_{\mathcal{Z}}^2\sqrt{\log(M/\delta)/n_{\min}}$,
the WDRO-GL estimator achieves
$R_{\mathbb{P}^*}(\hat{\mathbf{L}}^*)-R_{\mathbb{P}^*}(\mathbf{L}^*)\leq\varepsilon$
with probability at least $1-\delta$ provided
\begin{align}
  n_{\min}
  \;\geq\;
  \frac{C_1^2\,(BL_{\mathrm{bary}}+B_q)^2\,D_{\mathcal{Z}}^4}
       {\varepsilon^2}
  \log\frac{M}{\delta},
  \label{eq:sample_complexity}
\end{align}
where $C_1>0$ is an absolute constant. The dependence on the number
of sources is only logarithmic, and the factor
$L_{\mathrm{bary}}\sim\kappa^{-1/2}$ reflects that greater
inter-source heterogeneity (smaller $\kappa$) makes the barycenter
covariance harder to estimate and increases the required sample size.
\end{corollary}
\textit{Proof.} From~\eqref{eq:excess_risk_rate}, the excess risk is bounded by
$C_1(BL_{\mathrm{bary}}+B_q)D_{\mathcal{Z}}^2\sqrt{\log(M/\delta)/n_{\min}}$
for an absolute constant $C_1>0$.
Setting this at most $\varepsilon$ and solving for $n_{\min}$
yields~\eqref{eq:sample_complexity}.
$\hfill\square$

\section{Theoretical Analysis}
\label{sec:enhanced_theory}
The out-of-sample guarantees established in Theorem~\ref{thm:oos_rad}
provide standard DRO coverage with respect to the empirical Wasserstein ball
but leave three fundamental questions unaddressed: (i) how well does the
empirical barycenter $\hat{b}^*_K$ itself approximate the true (population-level)
barycenter $b^*_{\bm\lambda}$, and at what rate does this approximation improve
with sample size; (ii) is there a formal, quantifiable sense in which the
Wasserstein barycenter is a structurally superior nominal distribution compared
to the na\"{i}ve pooled empirical measure, particularly when inter-source
heterogeneity is large; and (iii) can the ambiguity-set framework certify
coverage of the \emph{target-domain} risk in the complete absence of
target samples. The present section answers all three questions through a
sequence of new results. The guarantees here are \emph{method-specific}:
they exploit the geometry of barycentric interpolation in Wasserstein space
and cannot be obtained from generic DRO theory alone.


Assumption~\ref{ass:subgauss} models each source as an (possibly
degenerate) Gaussian $\mathbb{P}_m^*=\mathcal{N}(\mu_m,\Sigma_m)$, whose
support is all of $\mathbb{R}^N$. The Rademacher-complexity argument of
Section~\ref{subsec:rademacher_oos}, in contrast, requires a
\emph{uniform, finite} bound on $\|\mathbf{z}\|_2$ to control the loss
class $\mathcal{F}_B$. We reconcile the two by deriving, rather than
positing, an effective compact support: with high probability, every
sample drawn from a Gaussian source lies within an explicit ball, whose
radius depends only on the source second moments and the confidence
level.

\begin{lemma}
\label{lem:truncation}
Let $\mathbf{z}\sim\mathcal{N}(\mu,\Sigma)$ with $\Sigma\in\mathbb{S}_+^N$.
For any $\delta\in(0,1)$,
\begin{align}
    \mathbb{P}\Bigl(\Vert\mathbf{z}\Vert_2
    &>\underbrace{\Vert\mu\Vert_2+\sqrt{\mathrm{tr}(\Sigma)}
    +\sqrt{2\lambda_{\max}(\Sigma)\log(1/\delta)}}_{=:R(\Sigma,\mu,\delta)}
    \Bigr)\nonumber\\
    &\leq\delta.
    \label{eq:truncation_radius}
\end{align}
\end{lemma}
\textit{Proof.} The map $\mathbf{z}\mapsto\Vert\mathbf{z}\Vert_2$ is $1$-Lipschitz.
Writing $\mathbf{z}=\mu+\Sigma^{1/2}\mathbf{g}$ for standard Gaussian
$\mathbf{g}$, the composite map $\mathbf{g}\mapsto\Vert\mu+\Sigma^{1/2}
\mathbf{g}\Vert_2$ is $\lambda_{\max}(\Sigma)^{1/2}$-Lipschitz in
$\mathbf{g}$, so the Gaussian concentration inequality gives $\mathbb{P}(\Vert\mathbf{z}\Vert_2\geq
\mathbb{E}\Vert\mathbf{z}\Vert_2+t)\leq\exp(-t^2/(2\lambda_{\max}
(\Sigma)))$ for all $t\geq0$. Bounding $\mathbb{E}\Vert\mathbf{z}
\Vert_2\leq\Vert\mu\Vert_2+\sqrt{\mathrm{tr}(\Sigma)}$ by Jensen's
inequality and setting $t=\sqrt{2\lambda_{\max}(\Sigma)\log(1/\delta)}$
gives~\eqref{eq:truncation_radius}. $\hfill\square$

\begin{corollary}
\label{cor:truncation_event}
Let $n=\sum_{m=1}^M n_m$ be the total number of source samples. For any
$\delta\in(0,1)$, define
\begin{align}
    R(\delta)&:=\max_{m\in[M]}R\bigl(\Sigma_m,\mu_m,\tfrac{\delta}{n}\bigr),\nonumber\\
    \mathcal{Z}_{R(\delta)}&:=\bigl\{\Sigma\in\mathbb{S}_+^N:
    \mathrm{tr}(\Sigma)\leq R(\delta)^2\bigr\}.
    \label{eq:Zdelta_def}
\end{align}
Then $\mathcal{Z}_{R(\delta)}$ is convex and compact, and the event
\begin{align}
    &\mathcal{E}_{\mathrm{trunc}}(\delta)\nonumber\\
    &:=\bigl\{
    \bm\Sigma_{m,i}=\mathbf{z}_{m,i}\mathbf{z}_{m,i}^T\in
    \mathcal{Z}_{R(\delta)}, \forall m\in[M],\,i\in[n_m]\bigr\}
\end{align}
satisfies $\mathbb{P}(\mathcal{E}_{\mathrm{trunc}}(\delta))\geq1-\delta$.
\end{corollary}
\textit{Proof.}
Convexity and compactness of $\mathcal{Z}_{R(\delta)}$ are immediate
(a sublevel set of the linear functional $\mathrm{tr}(\cdot)$ intersected
with the closed cone $\mathbb{S}_+^N$). By Lemma~\ref{lem:truncation}
applied at level $\delta/n$ to each of the $n$ i.i.d.\ draws, each
individual event $\{\Vert\mathbf{z}_{m,i}\Vert_2\leq R(\delta)\}$ fails
with probability at most $\delta/n$; a union bound over all $n$ draws,
together with $\mathrm{tr}(\bm\Sigma_{m,i})=\Vert\mathbf{z}_{m,i}
\Vert_2^2\leq R(\delta)^2\iff\bm\Sigma_{m,i}\in\mathcal{Z}_{R(\delta)}$,
gives the claim. $\hfill\square$

Corollary~\ref{cor:truncation_event} is the precise sense in which a
common convex compact support is compatible with
Assumption~\ref{ass:subgauss}: it does not hold surely, as an unbounded
Gaussian assumption forbids, but it holds on an explicit event of
probability at least $1-\delta$, for a radius $R(\delta)$ that is fully
determined by the source second moments. Assumption~\ref{ass:regularity}
above should accordingly be read as holding on $\mathcal{E}_{\mathrm{
trunc}}(\delta)$, and every downstream probabilistic guarantee absorbs
the residual failure probability $\delta$ into its stated confidence
level via a union bound, exactly as carried out in
Theorem~\ref{thm:oos_rad}.


Under Assumption~\ref{ass:regularity}, the barycenter functional
$F(\nu)=\sum_{m=1}^{M}\lambda_m W_2^2(\nu,\mathbb{P}_m^*)$ is
$\kappa$-strongly geodesically convex on $(\mathcal{P}_2(\mathcal{Z}),W_2)$
for some $\kappa>0$ depending only on $D_{\mathcal{Z}}$, and $\bm\lambda$,
and the true barycenter
$b^*_{\bm\lambda}\triangleq b_{\bm\lambda,2}(\mathbb{P}_1^*,\ldots,\mathbb{P}_M^*)$
exists and is unique \cite{agueh2011barycenters}. We define the
\emph{inter-source Wasserstein heterogeneity} as the Fre\'chet standard
deviation of the source distributions about their common barycenter:
\begin{align}
  \mathcal{H}_{\bm\lambda}
  \;\triangleq\;
  \biggl(\sum_{m=1}^{M}\lambda_m\,
    W_2^2\!\bigl(\mathbb{P}_m^*,\,b^*_{\bm\lambda}\bigr)\biggr)^{\!1/2}.
  \label{eq:H_def}
\end{align}
Larger $\mathcal{H}_{\bm\lambda}$ indicates greater spread of the sources in
Wasserstein space; $\mathcal{H}_{\bm\lambda}=0$ if and only if all source
distributions coincide.

\subsection{Finite-Sample Concentration of the Empirical Barycenter}
\label{subsec:conc_bary}

Our first result is a non-asymptotic perturbation bound that propagates
individual-source estimation errors through the barycenter map. Unlike
standard individual-source bounds, the result here explicitly separates
the contribution of \emph{sampling noise} ($\xi_M$) from the contribution
of \emph{inter-source heterogeneity} ($\mathcal{H}_{\bm\lambda}$), yielding
a rate that adapts to the geometric configuration of the source distributions.

\begin{lemma}
\label{lem:bary_perturbation}
Under Assumption~\ref{ass:regularity}, for any empirical source measures
$\{\hat{\mathbb{P}}_m\}_{m=1}^M\subset\mathcal{P}_2(\mathcal{Z})$, let
$\xi_M\triangleq\sum_{m=1}^M\lambda_m W_2^2(\hat{\mathbb{P}}_m,\mathbb{P}_m^*)$
denote the weighted aggregate source estimation error. Then
\begin{align}
  W_2\!\bigl(\hat{b}^*_M,\,b^*_{\bm\lambda}\bigr)
  \;\leq\;
  \frac{C_0}{\kappa}\Bigl(
    \mathcal{H}_{\bm\lambda}^{1/2}\,\xi_M^{1/4}
    \;+\;\xi_M^{1/2}
  \Bigr),
  \label{eq:perturbation_bound}
\end{align}
where $C_0>0$ is a universal constant and $\kappa$ is the
strong-convexity constant of $F$ from Assumption~\ref{ass:regularity}.

\noindent[A proof is given in Appendix \ref{app_lem_bary}.]
\end{lemma}

 \begin{remark}
\label{rem:tradeoff}
Lemma~\ref{lem:bary_perturbation} reveals a fundamental two-regime behavior.
\emph{(i) Homogeneous regime} ($\mathcal{H}_{\bm\lambda}=0$, all sources identical):
the bound collapses to $W_2(\hat{b}^*_M,b^*_{\bm\lambda})\leq
(C_0/\kappa)\,\xi_M^{1/2}$, recovering the standard parametric
$O(n^{-1/2})$ rate.
\emph{(ii) Heterogeneous regime} ($\mathcal{H}_{\bm\lambda}>0$):
the dominant term is $\mathcal{H}_{\bm\lambda}^{1/2}\xi_M^{1/4}$, yielding
a slower $O(\mathcal{H}_{\bm\lambda}^{1/2}\cdot n^{-1/4})$ convergence,
with the coefficient controlled by the inter-source spread.
This rate degradation reflects a genuine geometric phenomenon:
a larger spread of the source distributions in $(\mathcal{P}_2,W_2)$
renders the barycenter landscape more ``elongated,'' slowing convergence
of the empirical minimizer. Crucially, both $\mathcal{H}_{\bm\lambda}$ and
$\xi_K$ are estimable from data, which enables the data-driven radius
selection in Section~\ref{subsec:data_driven}.
\end{remark}

Combining Lemma~\ref{lem:bary_perturbation} with the individual-source
concentration under Assumption~\ref{ass:subgauss} (identical to that
used in Theorem~\ref{thm:oos_rad}), and a union bound over the $M$ sources,
yields the following main concentration theorem.

\begin{theorem}
\label{thm:bary_conc}
Under Assumptions~\ref{ass:regularity} and \ref{ass:subgauss}, for any
$\beta\in(0,1)$ define the aggregate deviation parameter
\begin{align}
  \bar{\xi}_n(\beta)
  \;\triangleq\;
  \sum_{m=1}^{K}\lambda_m\,\epsilon_{n_m}^2\!\!\left(\frac{\beta}{M}\right),
  \label{eq:xi_bar}
\end{align}
where $\epsilon_n(\beta)=\sqrt{\frac{c_1}{n}\log\frac{2}{\beta}}$ with $c_1$ being a constant. Then, with
probability at least $1-\beta$ over the joint draw of all source samples,
\begin{align}
  &W_2\!\bigl(\hat{b}^*_M,\,b^*_{\bm\lambda}\bigr)
  \leq
  \varepsilon_n^{\mathrm{bary}}(\beta)\nonumber\\
  &\quad\triangleq
  \boxed{
  \frac{C_0}{\kappa}\Bigl[
    \mathcal{H}_{\bm\lambda}^{1/2}\,\bar{\xi}_n(\beta)^{1/4}
    \;+\;\bar{\xi}_n(\beta)^{1/2}
  \Bigr].
  }
  \label{eq:bary_conc_bound}
\end{align}
\end{theorem}
\textit{Proof.} Under Assumption~\ref{ass:subgauss}, Theorem~\ref{thm:oos_rad} applied to
the $k$-th source with confidence level $\beta/M$ gives, $\forall m\in[M]$
\begin{align}
  \mathbb{P}^{n_m}\!\left\{
    W_2(\hat{\mathbb{P}}_m,\mathbb{P}_m^*)\leq\epsilon_{n_m}(\beta/M)
  \right\}\geq 1-\frac{\beta}{M}.
  \label{eq:indiv_conc}
\end{align}
Let $\mathcal{E}=\bigcap_{m=1}^M\bigl\{W_2(\hat{\mathbb{P}}_m,\mathbb{P}_m^*)
\leq\epsilon_{n_m}(\beta/M)\bigr\}$. By the union bound and
independence of the source samples,
$\mathbb{P}^n(\mathcal{E})\geq 1-\beta$.
On the event $\mathcal{E}$:
\begin{align}
  \xi_M
  =\sum_{m=1}^M\lambda_m W_2^2(\hat{\mathbb{P}}_m,\mathbb{P}_m^*)
  &\leq\sum_{m=1}^M\lambda_m\,\epsilon_{n_m}^2(\beta/M)\nonumber\\
  &=\bar{\xi}_n(\beta).
  \label{eq:xi_bound}
\end{align}
Substituting \eqref{eq:xi_bound} into Lemma~\ref{lem:bary_perturbation}
yields \eqref{eq:bary_conc_bound} on $\mathcal{E}$, which has probability
at least $1-\beta$. $\hfill\square$

\begin{remark}
Theorem~\ref{thm:bary_conc} is the first direct non-asymptotic bound
for the empirical Wasserstein barycenter as a \emph{statistical estimator}
in the heterogeneous multi-source setting. Prior work either establishes
individual-source convergence rates and uses the nominal distribution
without quantifying its bias \cite{kuhn2019wasserstein}, or treats the
barycenter concentration only in the univariate or symmetric-source
setting. Here the explicit dependence on $\mathcal{H}_{\bm\lambda}$
renders the bound tight: the approximation error genuinely grows with
inter-source heterogeneity. When all sources share the same sample size
$n_m=n/M$, \eqref{eq:bary_conc_bound} reads
$\varepsilon_n^{\mathrm{bary}}=O\bigl(\mathcal{H}_{\bm\lambda}^{1/2}
(M\log M/n)^{1/4}+\sqrt{M\log M/n}\bigr)$,
revealing an unavoidable cost of heterogeneity that scales as
$\mathcal{H}_{\bm\lambda}^{1/2}$.
\end{remark}

\subsection{Structural Superiority of the Barycentric Nominal Distribution}
\label{subsec:bary_vs_pool}

We now formally characterize the \emph{irreducible mixing bias} that afflicts
the pooled nominal distribution
$\hat{\mathbb{P}}^{\mathrm{pool}}_{\bm\lambda}\triangleq\sum_{m=1}^M
\lambda_m\hat{\mathbb{P}}_m$, and contrast it with the consistent behavior
of the empirical barycenter. The central finding, established under a common-covariance,
heterogeneous-mean Gaussian source model, is that the Wasserstein
distance from the pooled distribution to the true barycenter
$b^*_{\bm\lambda}$ does not vanish as $n\to\infty$ whenever the source
means differ, whereas the barycenter-based estimator remains
consistent regardless.

\begin{theorem}
\label{thm:pooling_bias}
Suppose the source distributions are Gaussian with a common covariance:
$\mathbb{P}_m^*=\mathcal{N}(\mu_m,\Sigma)$ for $m\in[M]$, where
$\Sigma\in\mathbb{S}_{++}^N$. Let $\bar{\mu}_{\bm\lambda}=\sum_m\lambda_m\mu_m$
and define the \emph{between-source scatter matrix}
\begin{align}
  \Delta_{\bm\lambda}
  =\sum_{m=1}^M\lambda_m(\mu_m-\bar{\mu}_{\bm\lambda})
    (\mu_m-\bar{\mu}_{\bm\lambda})^T
  \;\in\;\mathbb{S}_+^N.
  \label{eq:delta_def}
\end{align}
Then the true pooled distribution
$\mathbb{P}^*_{\bm\lambda}=\sum_m\lambda_m\mathbb{P}_m^*$ satisfies
\begin{align}
  W_2^2\!\bigl(\mathbb{P}^*_{\bm\lambda},\,b^*_{\bm\lambda}\bigr)
  &\geq
  B^2\!\bigl(\Sigma+\Delta_{\bm\lambda},\,\Sigma\bigr)\nonumber\\
  &\geq
  \frac{\mathcal{H}_{\bm\lambda}^4}
       {4N\,\sigma_{\max}(\Sigma)+4\mathcal{H}_{\bm\lambda}^2}
  \;>\;0,
  \label{eq:pooling_lb}
\end{align}
whenever $\mathcal{H}_{\bm\lambda}>0$, where
$\sigma_{\max}(\Sigma)$ is the spectral norm of $\Sigma$ and
$B(\cdot,\cdot)$ is the Bures-Wasserstein distance.
In contrast, $W_2(\hat{b}^*_M,b^*_{\bm\lambda})\to 0$ almost surely
as $n_{\min}\to\infty$.

\noindent[A proof is given in Appendix~\ref{app_thm_pool}.]
\end{theorem}

\begin{proposition}
\label{prop:cov_pooling_bias}
Suppose the sources share a common mean, $\mu_m\equiv\mu$ for all
$m\in[M]$, and pairwise-commuting covariances $\mathbb{P}_m^*=
\mathcal{N}(\mu,\Sigma_m)$ with
$\Sigma_m=Q\mathrm{diag}(\sigma_{m,1},\ldots,\sigma_{m,N})\allowbreak Q^T$ for a
common orthonormal eigenbasis $Q\in\mathrm{O}(N)$ ($\mathrm{O}(N)$ denote the orthogonal group of order $N$); in particular, by
Proposition~\ref{prop1}, $Q$ may always be chosen with
$Q\mathbf{e}_N=\mathbf{1}/\sqrt{N}$, since every $\Sigma_m$ here
annihilates $\mathbf{1}$. Write $\bar\sigma_{\mathrm{arith},k}=
\sum_m\lambda_m\sigma_{m,k}$ for the arithmetic-mean eigenvalues and
\begin{align}
  \bar\sigma_{\bm\lambda,k}
  \;=\;\Bigl(\sum_{m=1}^M\lambda_m\sqrt{\sigma_{m,k}}\Bigr)^{\!2},
  \qquad k=1,\ldots,N,
  \label{eq:cov_bary_formula}
\end{align}
for the coordinatewise barycentric eigenvalues. Then the true
barycenter is $b^*_{\bm\lambda}=\mathcal{N}\bigl(\mu,\,Q\,
\mathrm{diag}(\bar{\bm\sigma}_{\bm\lambda})\,Q^T\bigr)$, and, with
$\bar\Sigma_{\mathrm{arith}}=\sum_m\lambda_m\Sigma_m$,
\begin{align}
  W_2^2\!\bigl(\mathbb{P}^*_{\bm\lambda},\,b^*_{\bm\lambda}\bigr)
  &\geq
  \sum_{k=1}^N\Bigl(\sqrt{\bar\sigma_{\mathrm{arith},k}}
  -\sqrt{\bar\sigma_{\bm\lambda,k}}\Bigr)^{\!2}
  \label{eq:cov_pooling_exact}\\
  &\geq
  \max_{i\neq j}\;
  \frac{\lambda_i^2\lambda_j^2\,B^4(\Sigma_i,\Sigma_j)}
       {4N\,\sigma_{\max}\bigl(\bar\Sigma_{\mathrm{arith}}\bigr)}
  \;>\;0,
  \label{eq:cov_pooling_lb}
\end{align}
whenever $\Sigma_i\neq\Sigma_j$ for some pair $i,j$ with
$\lambda_i,\lambda_j>0$.

\noindent[A proof is given in Appendix~\ref{app_prop_covpool}.]
\end{proposition}

\begin{remark}
\label{rem:two_channels}
Theorem~\ref{thm:pooling_bias} and
Proposition~\ref{prop:cov_pooling_bias} isolate two distinct,
individually rigorous channels through which naive pooling incurs an
irreducible bias: heterogeneous means under a shared covariance
(Theorem~\ref{thm:pooling_bias}), and heterogeneous, commuting
covariances under a shared mean
(Proposition~\ref{prop:cov_pooling_bias}). Both floors are strictly
positive, sample-size-independent, and share the same structural
form, a fourth-power heterogeneity numerator over an
$N\sigma_{\max}$-scaled denominator, indicating that the
irreducibility of the pooling bias is not an artifact of the
common-covariance restriction of Theorem~\ref{thm:pooling_bias}, but
persists along at least two independent axes of source heterogeneity.
\end{remark}

\begin{remark}
\label{rem:pooling_general_cov}
Theorem~\ref{thm:pooling_bias} is proved under a common-covariance
source model, and Proposition~\ref{prop:cov_pooling_bias} under a
common-mean, commuting-covariance model, because in both cases the
Bures--Wasserstein separation reduces to a \emph{scalar} problem:
Theorem~\ref{thm:pooling_bias} twirls one covariance argument to
$\bar\sigma\mathbf{I}$ while the other stays fixed at $\Sigma$, and
Proposition~\ref{prop:cov_pooling_bias} diagonalizes both arguments in
a shared eigenbasis $Q$. The fully general setting, in which sources
differ \emph{simultaneously} in mean and in non-commuting covariance,
$\mathbb{P}_m^*=\mathcal{N}(\mu_m,\Sigma_m)$ with $\{\Sigma_m\}$ not
jointly diagonalizable and $\Delta_{\bm\lambda}\neq\mathbf{0}$, admits
neither reduction: the pooled covariance
$\bar\Sigma_{\mathrm{arith}}+\Delta_{\bm\lambda}$ and the barycentric
covariance $\bar\Sigma_{\bm\lambda}$, itself only implicitly defined
through the fixed point~\eqref{eq10}, with no closed form when the
$\Sigma_m$ do not commute, generically share no common eigenbasis,
so the Lieb concavity--Haar twirling argument of
Appendix~\ref{app_thm_pool} does not extend directly, nor does the
elementary diagonal computation of
Appendix~\ref{app_prop_covpool}. We therefore state the fully general
case as an open conjecture: an analogous strictly positive,
sample-size-independent bias is expected to persist whenever
$(\mu_m,\Sigma_m)_{m=1}^M$ are not all identical, supported jointly by
Theorem~\ref{thm:pooling_bias}, Proposition~\ref{prop:cov_pooling_bias}
(which cover two complementary special cases of this general
statement), and by the concordant empirical evidence of
Section~\ref{sec:experiments}, where sources differ in both moments
simultaneously. We do not claim a proof of this general statement here.
A natural avenue toward closing this gap, left to future work, is a
perturbative argument bounding the deviation of $\bar\Sigma_{\bm
\lambda}$ from the commuting-case barycenter of
\eqref{eq:cov_bary_formula} in terms of the operator norm of the
commutators $[\Sigma_i,\Sigma_j]$, which would interpolate
continuously between Proposition~\ref{prop:cov_pooling_bias} and the
fully non-commuting regime.
\end{remark}

\begin{remark}
\label{rem:asymptotic_comparison}
Theorem~\ref{thm:pooling_bias} establishes that for any fixed
$\mathcal{H}_{\bm\lambda}>0$, the minimum Wasserstein ball radius needed to
certify coverage of $b^*_{\bm\lambda}$ using the pooled nominal distribution
satisfies
\begin{align}
  &\liminf_{n\to\infty}
  W_2(\hat{\mathbb{P}}^{\mathrm{pool}}_{\bm\lambda},\,b^*_{\bm\lambda})
  \geq
  W_2(\mathbb{P}^*_{\bm\lambda},\,b^*_{\bm\lambda})\nonumber\\
  &\quad\geq
  \frac{\mathcal{H}_{\bm\lambda}^2}
       {\sqrt{4N\sigma_{\max}(\Sigma)+4\mathcal{H}_{\bm\lambda}^2}}
  \;>\;0,
  \label{eq:pool_liminf}
\end{align}
as a consequence of the law of large numbers applied to each source.
In sharp contrast, Theorem~\ref{thm:bary_conc} gives
$\varepsilon_n^{\mathrm{bary}}(\beta)\to 0$ as $n_{\min}\to\infty$.
Hence the ambiguity set centered at $\hat{b}^*_M$ \emph{automatically
shrinks} with increasing data, while any ambiguity set centered at
$\hat{\mathbb{P}}^{\mathrm{pool}}_{\bm\lambda}$ must maintain a radius
bounded away from zero by the irreducible mixing bias.
This translates directly into a less conservative WDRO estimator $\mathbf{L}^*$
when the barycentric nominal distribution is used.
\end{remark}

\subsection{Data-Driven Radius Selection}
\label{subsec:data_driven}

Theorems~\ref{thm:bary_conc} and \ref{thm:pooling_bias} jointly prescribe
a principled, data-driven procedure for choosing $\epsilon$ in the MS-WDRO
problem \eqref{equ.15}. The radius should be tight enough to avoid
over-conservatism, yet large enough to ensure the true barycenter
$b^*_{\bm\lambda}$ lies in the ambiguity set with prescribed probability.
The following corollary makes this precise.

\begin{corollary}
\label{cor:data_driven_eps}
Let the assumptions of Theorem~\ref{thm:bary_conc} hold. Fix any
$\beta\in(0,1)$. Set the radius
\begin{align}
  \epsilon^*_n(\beta)
  \;=\;
  \frac{C_0}{\kappa}\Bigl[
    \hat{\mathcal{H}}_{\bm\lambda}^{1/2}\,\bar{\xi}_n(\beta)^{1/4}
    +\bar{\xi}_n(\beta)^{1/2}
  \Bigr],
  \label{eq:eps_data_driven}
\end{align}
where $\hat{\mathcal{H}}_{\bm\lambda}=\bigl(\sum_m\lambda_m
W_2^2(\hat{\mathbb{P}}_m,\hat{b}^*_M)\bigr)^{1/2}$
is the empirical inter-source heterogeneity and
$\bar{\xi}_n(\beta)$ is as in \eqref{eq:xi_bar}.
Then, with probability at least $1-2\beta$:
\begin{align}
  b^*_{\bm\lambda}\;\in\;
  \mathcal{W}_{\epsilon^*_n(\beta),2}\!
  \bigl(\hat{\mathbb{P}}_1,\ldots,\hat{\mathbb{P}}_M;\bm\lambda\bigr),
  \label{eq:coverage}
\end{align}
and $\epsilon^*_n(\beta)\to 0$ almost surely as $n_{\min}\to\infty$.
Moreover, the radius satisfies the monotone decay
\begin{align}
  \epsilon^*_n(\beta)&\leq
  \frac{C_0}{\kappa}\left[
    \hat{\mathcal{H}}_{\bm\lambda}^{1/2}
    \left(\frac{c_1}{n_{\min}}
      \log\frac{2M}{\beta}\right)^{1/4}
  \right]\nonumber\\
  &\quad\frac{C_0}{\kappa}\left[+\left(\frac{c_1}{n_{\min}}
      \log\frac{2M}{\beta}\right)^{1/2}\right],
  \label{eq:eps_rate}
\end{align}
contrasted with the pooled-distribution radius lower bound
$\epsilon^{\mathrm{pool}}\geq\mathcal{H}_{\bm\lambda}^2/\sqrt{4N\sigma_{\max}
+4\mathcal{H}_{\bm\lambda}^2}$, which is bounded away from zero.
\end{corollary}
\textit{Proof.} By the triangle inequality, \[|\hat{\mathcal{H}}_{\bm\lambda}-\mathcal{H}_{\bm\lambda}|
\leq\sqrt{\sum_m\lambda_m W_2^2(\hat{\mathbb{P}}_m,\mathbb{P}_m^*)}=\xi_M^{1/2}\].
With probability at least $1-\beta$, $\xi_M\leq\bar\xi_n(\beta)$
(Theorem~\ref{thm:bary_conc}), so
$\hat{\mathcal{H}}_{\bm\lambda}\leq\mathcal{H}_{\bm\lambda}+\bar\xi_n(\beta)^{1/2}$.
Substituting $\hat{\mathcal{H}}_{\bm\lambda}\leq\mathcal{H}_{\bm\lambda}
+\bar\xi_n(\beta)^{1/2}$ into \eqref{eq:eps_data_driven} and absorbing
the lower-order term into $C_0$, the right-hand side of \eqref{eq:eps_data_driven}
is at least as large as $\varepsilon_n^{\mathrm{bary}}(\beta)$ from
\eqref{eq:bary_conc_bound} with probability $1-\beta$.
A second application of Theorem~\ref{thm:bary_conc} at level $\beta$
then yields \eqref{eq:coverage} with combined probability
at least $1-2\beta$ by the union bound. The
monotone decay and the pooled-distribution comparison follow from
\eqref{eq:eps_rate} and Theorem~\ref{thm:pooling_bias}, respectively.
$\hfill\square$

\subsection{Target-Free Risk Approximation via Source Barycenter}
\label{subsec:target_free}

In the motivating scenario of this paper, no target-domain samples are
available at all. The purpose of this section is to show that the WDRO-GL
framework with the barycentric nominal distribution nonetheless provides
a valid risk certificate for the \emph{target-domain} risk, at the sole
cost of enlarging the Wasserstein radius by the target-to-barycenter
proximity $\omega$, a quantity that characterizes how well the source
barycenter approximates the unknown target distribution. This is the
strongest theoretical justification for the proposed framework.

\begin{theorem}
\label{thm:target_free_1}
Let $\mathbb{P}^{\ast}$ denote the unknown target distribution and
suppose there exists $\omega\geq 0$ such that
\begin{align}
  W_2\!\bigl(\mathbb{P}^{\ast},\,b^*_{\bm\lambda}\bigr)\leq\omega.
  \label{eq:target_proximity}
\end{align}
Under the assumptions of Theorem~\ref{thm:bary_conc}, set
$\epsilon=\varepsilon_n^{\mathrm{bary}}(\beta)+\omega$. Then with
probability at least $1-\beta$ over the draw of source samples,
\begin{align}
  \mathbb{P}^{\ast}
  \;\in\;
  \mathcal{W}_{\epsilon,2}\!
  \bigl(\hat{\mathbb{P}}_1,\ldots,\hat{\mathbb{P}}_M;\bm\lambda\bigr),
  \label{eq:target_in_ball}
\end{align}
and consequently, for every graph Laplacian estimator $\mathbf{L}$:
\begin{align}
  R_{\mathbb{P}^{\ast}}(\mathbf{L})
  \;\leq\;
  R_{\mathcal{W}_{\epsilon,2}}(\mathbf{L}),
  \label{eq:target_risk_cert}
\end{align}
where $R_{\mathcal{W}_{\epsilon,2}}(\mathbf{L})=\sup_{\mathbb{Q}\in
\mathcal{W}_{\epsilon,2}}\mathbb{E}_{\mathbb{Q}}[\ell(\mathbf{L},\cdot)]$
is the worst-case risk of \eqref{equ.15}.
\end{theorem}
\textit{Proof.} By Theorem~\ref{thm:bary_conc}, with probability at least $1-\beta$:
\begin{align}
  W_2(\hat{b}^*_M,\,b^*_{\bm\lambda})
  \;\leq\;
  \varepsilon_n^{\mathrm{bary}}(\beta).
  \label{eq:bary_bound_recall}
\end{align}
By the triangle inequality for $W_2$ and the target proximity assumption
\eqref{eq:target_proximity}:
\begin{align}
  W_2\!\bigl(\mathbb{P}^{\ast},\hat{b}^*_M\bigr)
  &\;\leq\;
  W_2(\mathbb{P}^{\ast},b^*_{\bm\lambda})
  +W_2(b^*_{\bm\lambda},\hat{b}^*_M)
  \notag\\
  &\;\leq\;
  \omega+\varepsilon_n^{\mathrm{bary}}(\beta)
  =\epsilon.
  \label{eq:target_in_ball_proof}
\end{align}
Since the ambiguity set $\mathcal{W}_{\epsilon,2}(\hat{\mathbb{P}}_1,\ldots,
\hat{\mathbb{P}}_M;\bm\lambda)$ is the Wasserstein ball of radius $\epsilon$
centered at $\hat{b}^*_M$, \eqref{eq:target_in_ball_proof} gives
$\mathbb{P}^{\ast}\in\mathcal{W}_{\epsilon,2}$ on the stated
probability event. The risk certificate \eqref{eq:target_risk_cert}
then follows immediately from the definition of the worst-case risk
and the inclusion \eqref{eq:target_in_ball}. $\hfill\square$

\begin{remark}
\label{rem:radius_decomp}
The optimal radius $\epsilon=\varepsilon_n^{\mathrm{bary}}(\beta)+\omega$
admits a clean decomposition:
\begin{align}
  \underbrace{\varepsilon_n^{\mathrm{bary}}(\beta)}_{\text{sampling error}
  \;\to\;0}
  \;+\;
  \underbrace{\omega}_{\text{domain gap}},
  \label{eq:radius_decomp}
\end{align}
where the first term quantifies the finite-sample approximation error of
$\hat{b}^*_M$ relative to $b^*_{\bm\lambda}$ (which vanishes as
$n_{\min}\to\infty$), and the second term $\omega$ is the
\emph{domain gap}—the irreducible cost of not observing the target.
As $n_{\min}\to\infty$, the optimal radius converges to $\omega$, the
smallest possible radius that can guarantee coverage of
$\mathbb{P}^{\ast}$ with a ball centered at $b^*_{\bm\lambda}$.
In practice, $\omega$ can be estimated via \emph{source-side} bounds:
for any $k$,
\begin{align}
  \omega
  &\leq
  W_2(\mathbb{P}^{\ast},\mathbb{P}_m^*)
  +W_2(\mathbb{P}_m^*,b^*_{\bm\lambda})\nonumber\\
  &\leq
  W_2(\mathbb{P}^{\ast},\mathbb{P}_m^*)
  +\mathcal{H}_{\bm\lambda}/\lambda_m^{1/2},
  \label{eq:omega_ub}
\end{align}
where $W_2(\mathbb{P}^{\ast},\mathbb{P}_m^*)$ can be bounded using
any available side information (e.g., structural similarity or a small
number of proxy observations from the target domain).
In contrast, the pooled-distribution approach requires
$\epsilon^{\mathrm{pool}}\geq\omega+c\mathcal{H}_{\bm\lambda}^2$
even with infinite data, because the mixing bias $W_2(\mathbb{P}^*_{\bm\lambda},
b^*_{\bm\lambda})$ does not vanish (Theorem~\ref{thm:pooling_bias}).
\end{remark}
\begin{figure*}[!htb]
    \centering
    \includegraphics[width=\textwidth]{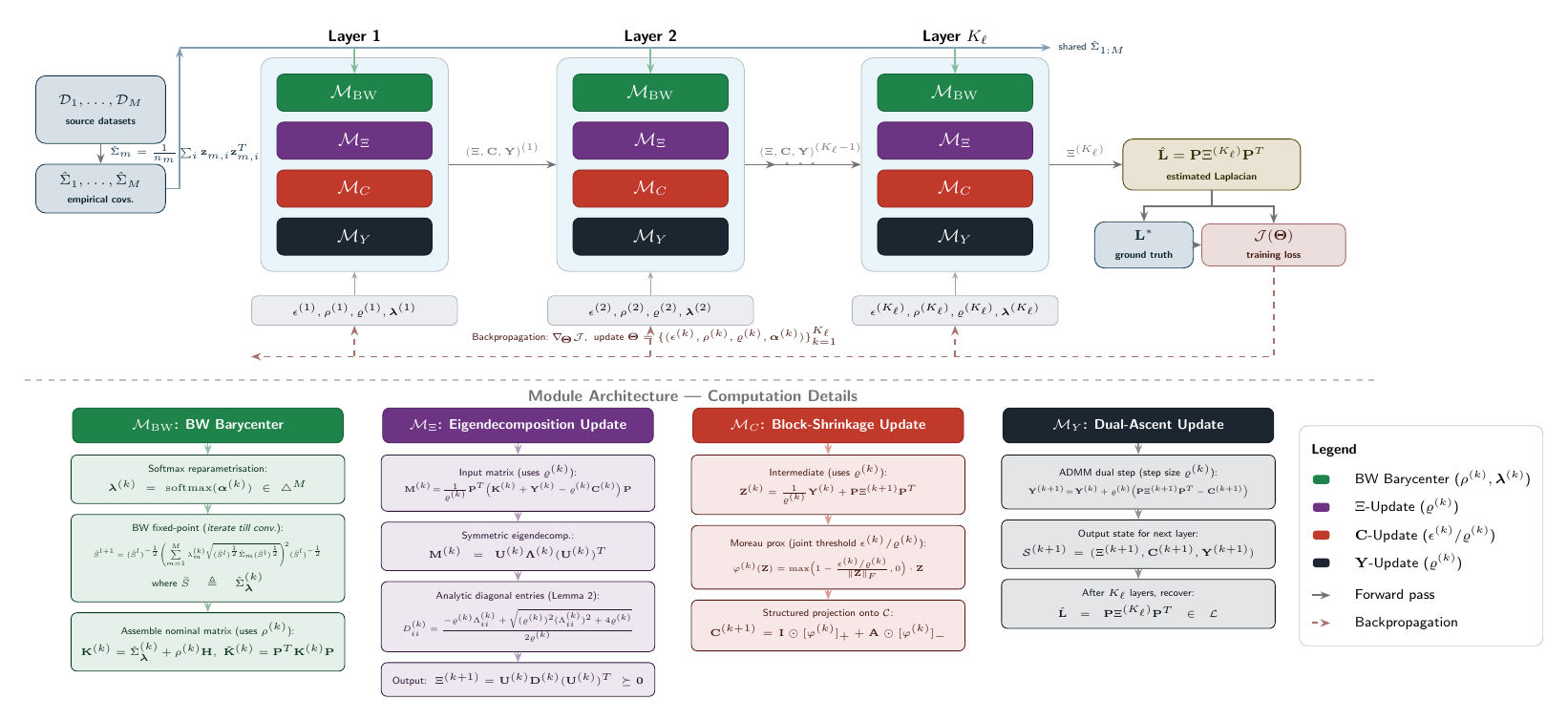}
    \caption{Overall framework of the unrolled MS-WDRO.}
    \label{fig:framework}
\end{figure*}

\section{Algorithm Unrolling for Automatic Parameter Learning}
\label{sec:unrolling}
Algorithm~\ref{alg1} involves four categories of free parameters
whose joint calibration governs the quality of the recovered Laplacian.
The sparsity regularization coefficient $\rho>0$ enters the objective~%
\eqref{eq17} and, after the Laplacian identity~\eqref{eq18}, defines
the effective nominal matrix
$\mathbf{K}=\hat{\Sigma}_{\bm{\lambda}}+\rho\mathbf{H}$~\eqref{eq19}
used at every iteration; it is therefore an \emph{objective-level}
parameter that shapes which optimization problem is being solved.
The ADMM augmented-Lagrangian penalty $\varrho>0$, introduced
in~\eqref{eq30}, is structurally different: it is an
\emph{algorithm-level} parameter that leaves the underlying problem~%
\eqref{eq29} invariant but governs the analytic diagonal entries of the
$\Xi$-update via Lemma~\ref{lemma2}~\eqref{eq33}, sets the block-shrinkage
threshold $\epsilon/\varrho$ in the $\mathbf{C}$-update~\eqref{eq35},
and determines the dual-ascent step size in~\eqref{eq36}.
The Wasserstein ambiguity set radius $\epsilon>0$ sets the
robustness level and couples with $\varrho$ through the joint ratio
$\epsilon/\varrho$ in the proximity step, so that tuning either in
isolation while holding the other fixed is suboptimal.
Finally, the barycentric weight vector $\bm{\lambda}\in\triangle^M$
determines $\hat{\Sigma}_{\bm{\lambda}}$ via the fixed-point
iteration~\eqref{eq10} and hence contributes to $\mathbf{K}$ alongside
$\rho$.
These four categories interact nonlinearly and their interactions cannot
be disentangled: $\rho$ and $\bm{\lambda}$ jointly construct $\mathbf{K}$,
while $\epsilon$ and $\varrho$ jointly set the shrinkage threshold.
Cross-validation over the resulting joint parameter space is
computationally prohibitive in the multi-source regime, and
concentration-inequality-based radius selection for $\epsilon$, while
tractable, is inherently conservative.

To overcome these limitations, we embed the iterative structure of
Algorithm~\ref{alg1} into a differentiable architecture via the
algorithm unrolling paradigm, reinterpreting each ADMM iteration as a
computational layer and lifting all four parameter categories to
layer-specific learnable variables that are jointly calibrated through
end-to-end backpropagation on a supervised training corpus.
The resulting $K_\ell$-layer unrolled network inherits the full
interpretability of the ADMM, every layer executes a precisely
specified proximal or dual-ascent substep, while acquiring
data-adaptive calibration capability that the fixed-parameter solver
cannot achieve.
Allowing layer-specific $\rho^{(k)}$, $\varrho^{(k)}$, $\epsilon^{(k)}$,
and $\bm{\lambda}^{(k)}$ enables an implicit annealing: $\rho^{(k)}$ may
be large in early layers to impose strong sparsity regularization when
the iterate is far from convergence and progressively relax in later
layers as the graph structure consolidates; $\varrho^{(k)}$ may similarly
adapt the ADMM step size along the unrolled trajectory; and
$\bm{\lambda}^{(k)}$ may shift the effective nominal distribution as the
estimate refines. The overall framework of the unrolled ADMM for MS-WDRO is summarized in Figure \ref{fig:framework}.

We now specify the layer architecture.
Let $\mathcal{S}^{(k)}=(\Xi^{(k)},\mathbf{C}^{(k)},\mathbf{Y}^{(k)})$
denote the iterate state entering layer $k$, initialized as
$\Xi^{(0)}=\mathbf{I}_{N-1}$, $\mathbf{C}^{(0)}=\mathbf{0}$,
$\mathbf{Y}^{(0)}=\mathbf{0}$.
Each layer $k$ is parameterized by the tuple
$\bm{\theta}^{(k)}=(\epsilon^{(k)},\rho^{(k)},\varrho^{(k)},\bm{\alpha}^{(k)})$,
where $\epsilon^{(k)},\rho^{(k)},\varrho^{(k)}>0$ are positive scalars
and $\bm{\lambda}^{(k)}=\mathrm{softmax}(\bm{\alpha}^{(k)})\in\triangle^M$
is the simplex-constrained barycentric weight vector obtained via the
unconstrained reparametrization $\bm{\alpha}^{(k)}\in\mathbb{R}^M$,
ensuring feasibility with smooth gradient flow through the simplex
constraint.
The layer comprises three sequentially executed differentiable modules.

The first module, $\mathcal{M}_{\Xi}$, constructs the layer-specific
coefficient matrix and updates the primal variable $\Xi$.
Given the current weights $\bm{\lambda}^{(k)}$, it computes
\begin{align}
    \hat{\Sigma}^{(k)}_{\bm{\lambda}}
    =\mathrm{BCov}\!\bigl(\hat{\Sigma}_1,\ldots,\hat{\Sigma}_M;
    \bm{\lambda}^{(k)}\bigr),
    \label{eq:bary_layer}
\end{align}
where $\mathrm{BCov}(\cdot;\bm{\lambda}^{(k)})$ is evaluated via
the fixed-point iteration~\eqref{eq10} at weights $\bm{\lambda}^{(k)}$,
and its gradient with respect to $\bm{\lambda}^{(k)}$ is obtained by
implicit differentiation of the fixed-point equation, valid under
Assumption~\ref{ass:bary_lip}.
The module then assembles the layer-specific nominal matrix
\begin{align}
    \mathbf{K}^{(k)}
    =\hat{\Sigma}^{(k)}_{\bm{\lambda}}+\rho^{(k)}\mathbf{H},
    \quad
    \tilde{\mathbf{K}}^{(k)}=\mathbf{P}^T\mathbf{K}^{(k)}\mathbf{P},
    \label{eq:K_layer}
\end{align}
in which $\rho^{(k)}$ and $\bm{\lambda}^{(k)}$ jointly determine the
effective data-fidelity matrix at depth $k$.
The $\Xi$-primal update then follows Lemma~\ref{lemma2}: the module
computes the eigendecomposition
\begin{align}
    \frac{1}{\varrho^{(k)}}\mathbf{P}^T\!
    \left(\mathbf{K}^{(k)}+\mathbf{Y}^{(k)}-\varrho^{(k)}\mathbf{C}^{(k)}\right)
    \mathbf{P}
    =\mathbf{U}^{(k)}\bm{\Lambda}^{(k)}\!\left(\mathbf{U}^{(k)}\right)^T
    \label{eq:eig_layer}
\end{align}
and sets $\Xi^{(k+1)}=\mathbf{U}^{(k)}\mathbf{D}^{(k)}(\mathbf{U}^{(k)})^T$,
where the diagonal entries of $\mathbf{D}^{(k)}$ are
\begin{align}
    D_{ii}^{(k)}
    =\frac{-\varrho^{(k)}\Lambda_{ii}^{(k)}
    +\sqrt{\bigl(\varrho^{(k)}\bigr)^2\!\bigl(\Lambda_{ii}^{(k)}\bigr)^2
    +4\varrho^{(k)}}}{2\varrho^{(k)}}.
    \label{eq:D_layer}
\end{align}
The mapping
$(\rho^{(k)},\varrho^{(k)},\bm{\lambda}^{(k)})\mapsto\Xi^{(k+1)}$
is differentiable via matrix perturbation theory for symmetric
eigendecompositions.
Observe that $\rho^{(k)}$ and $\varrho^{(k)}$ enter the update through
structurally distinct channels: $\rho^{(k)}$ shifts
$\mathbf{K}^{(k)}$ and thereby redefines the optimization landscape at
each depth, while $\varrho^{(k)}$ controls the analytic step-size
geometry through $D_{ii}^{(k)}$.

The second module, $\mathcal{M}_C$, executes the structural projection
of $\mathbf{C}$ via a Moreau proximity step.
Setting
$\mathbf{Z}^{(k)}=\frac{1}{\varrho^{(k)}}\mathbf{Y}^{(k)}
+\mathbf{P}\Xi^{(k+1)}\mathbf{P}^T$,
the update reads
\begin{align}
    \mathbf{C}^{(k+1)}
    &=\mathbf{I}\odot\left[\varphi^{(k)}\!\left(\mathbf{Z}^{(k)}\right)\right]_+
    +\mathbf{A}\odot\left[\varphi^{(k)}\!\left(\mathbf{Z}^{(k)}\right)\right]_-,
    \label{eq:C_layer}
\end{align}
where the matrix block-shrinkage operator is
\begin{align}
    \varphi^{(k)}(\mathbf{Z})
    =\max\!\left(1-\frac{\epsilon^{(k)}/\varrho^{(k)}}{\|\mathbf{Z}\|_F},\,0\right)
    \cdot\mathbf{Z},
    \label{eq:phi_layer}
\end{align}
$[\cdot]_+$ and $[\cdot]_-$ extract the element-wise positive and
negative parts, and $\mathbf{A}$ is the binary adjacency support mask.
The effective shrinkage threshold $\epsilon^{(k)}/\varrho^{(k)}$ is the
joint product of two independently trainable scalars: $\epsilon^{(k)}$
encodes the desired robustness level while $\varrho^{(k)}$ normalizes
it relative to the augmented-Lagrangian penalty.
The mapping $(\epsilon^{(k)},\varrho^{(k)})\mapsto\mathbf{C}^{(k+1)}$
is almost everywhere differentiable, with the subgradient defined
consistently at $\|\mathbf{Z}^{(k)}\|_F=\epsilon^{(k)}/\varrho^{(k)}$.

The third module, $\mathcal{M}_Y$, performs the dual-ascent step
\begin{align}
    \mathbf{Y}^{(k+1)}
    =\mathbf{Y}^{(k)}+\varrho^{(k)}\!
    \left(\mathbf{P}\Xi^{(k+1)}\mathbf{P}^T-\mathbf{C}^{(k+1)}\right),
    \label{eq:Y_layer}
\end{align}
which is linear in all preceding outputs and differentiable with respect
to $\varrho^{(k)}$.
After $K_\ell$ layers, the Laplacian estimator is recovered as
$\hat{\mathbf{L}}=\mathbf{P}\Xi^{(K_\ell)}\mathbf{P}^T\in\mathcal{L}$.
The complete forward pass is summarized in Algorithm~\ref{alg3}.

\begin{algorithm}[!htb]
    \caption{Unrolled ADMM for MS-WDRO (Forward Pass)}
    \label{alg3}
    \renewcommand{\algorithmicrequire}{\textbf{Input:}}
    \renewcommand{\algorithmicensure}{\textbf{Output:}}
    \begin{algorithmic}[1]
        \REQUIRE Source datasets $\{\mathcal{D}_m\}_{m=1}^M$;
                 depth $K_\ell$;
                 trainable parameters
                 $\bm{\Theta}=\{(\epsilon^{(k)},\rho^{(k)},
                 \varrho^{(k)},\bm{\alpha}^{(k)})\}_{k=1}^{K_\ell}$
        \ENSURE  Estimated graph Laplacian $\hat{\mathbf{L}}$
        \STATE Compute $\hat{\Sigma}_m=\tfrac{1}{n_m}\sum_{i=1}^{n_m}
               \mathbf{z}_{m,i}\mathbf{z}_{m,i}^T$, $\;k\in[K]$
        \STATE Initialize $\Xi^{(0)}=\mathbf{I}_{N-1}$,
               $\mathbf{C}^{(0)}=\mathbf{0}$, $\mathbf{Y}^{(0)}=\mathbf{0}$
        \FOR{$k=0$ \textbf{to} $K_\ell-1$}
            \STATE $\bm{\lambda}^{(k)}\leftarrow\mathrm{softmax}
                   (\bm{\alpha}^{(k)})$
            \STATE $\hat{\Sigma}^{(k)}_{\bm{\lambda}}\leftarrow
                   \mathrm{BCov}(\hat{\Sigma}_1,\ldots,\hat{\Sigma}_M;
                   \bm{\lambda}^{(k)})$ via \eqref{eq10}
            \STATE $\mathbf{K}^{(k)}\leftarrow
                   \hat{\Sigma}^{(k)}_{\bm{\lambda}}+\rho^{(k)}\mathbf{H}$,
                   $\;\tilde{\mathbf{K}}^{(k)}\leftarrow
                   \mathbf{P}^T\mathbf{K}^{(k)}\mathbf{P}$
            \STATE Eigendecomp:\ $\tfrac{1}{\varrho^{(k)}}\mathbf{P}^T\!
                   \left(\mathbf{K}^{(k)}+\mathbf{Y}^{(k)}-\varrho^{(k)}
                   \mathbf{C}^{(k)}\right)\mathbf{P}
                   =\mathbf{U}^{(k)}\bm{\Lambda}^{(k)}
                   (\mathbf{U}^{(k)})^T$
            \STATE Compute $\mathbf{D}^{(k)}$ via
                   \eqref{eq:D_layer};
                   $\Xi^{(k+1)}\leftarrow
                   \mathbf{U}^{(k)}\mathbf{D}^{(k)}
                   (\mathbf{U}^{(k)})^T$
                   \hfill$\triangleright\;\mathcal{M}_\Xi$
            \STATE $\mathbf{Z}^{(k)}\leftarrow
                   \tfrac{1}{\varrho^{(k)}}\mathbf{Y}^{(k)}
                   +\mathbf{P}\Xi^{(k+1)}\mathbf{P}^T$;
                   update $\mathbf{C}^{(k+1)}$ via \eqref{eq:C_layer}
                   \hfill$\triangleright\;\mathcal{M}_C$
            \STATE $\mathbf{Y}^{(k+1)}\leftarrow\mathbf{Y}^{(k)}
                   +\varrho^{(k)}\!\left(\mathbf{P}\Xi^{(k+1)}\mathbf{P}^T
                   -\mathbf{C}^{(k+1)}\right)$
                   \hfill$\triangleright\;\mathcal{M}_Y$
        \ENDFOR
        \RETURN $\hat{\mathbf{L}}=\mathbf{P}\Xi^{(K_\ell)}\mathbf{P}^T$
    \end{algorithmic}
\end{algorithm}

The network is trained in a supervised manner on a corpus of $Q$ labeled
multi-source instances
$\{(\mathcal{D}_1^{(j)},\ldots,\mathcal{D}_M^{(j)},
\mathbf{L}^{*(j)}\allowbreak)\}_{j=1}^Q$.
Denoting the intermediate estimate at layer $k$ and sample $j$ by
$\hat{\mathbf{L}}_j^{(k)}=\mathbf{P}\Xi_j^{(k)}\mathbf{P}^T$,
the squared relative error (SRE) on edge weights is
\begin{align}
    \mathrm{SRE}_j^{(k)}
    =\frac{\|\hat{\mathbf{L}}_j^{(k)}-\mathbf{L}^{*(j)}\|_F^2}
          {\|\mathbf{L}^{*(j)}\|_F^2}.
    \label{eq:SRE}
\end{align}
Because every feasible $\hat{\mathbf{L}}_j^{(k)}\in\mathcal{L}$ has
non-positive off-diagonal entries, the induced adjacency estimate
$\hat{\mathbf{A}}_j^{(k)}:=-\bigl(\hat{\mathbf{L}}_j^{(k)}-
\mathrm{diag}(\hat{\mathbf{L}}_j^{(k)})\bigr)$ is entrywise
non-negative, so applying the sigmoid directly to
$\hat{\mathbf{A}}_j^{(k)}$ would force $\sigma(\hat{A}_{j,uv}^{(k)})
\geq\sigma(0)=0.5$ for every node pair $(u,v)$, precluding any
calibrated probability below one-half even for true non-edges. We
avoid this by first mapping the non-negative weight estimate to a
signed logit via a per-layer, trainable affine transform,
\begin{align}
    s_{j,uv}^{(k)}
    =\gamma^{(k)}\,\hat{A}_{j,uv}^{(k)}-\eta^{(k)},
    \;\gamma^{(k)}=\mathrm{softplus}(\tilde{\gamma}^{(k)})>0,
    \label{eq:logit}
\end{align}
where $\eta^{(k)}\in\mathbb{R}$ is a learnable decision threshold and
the softplus reparameterization of the slope $\gamma^{(k)}$ preserves
monotonicity of $s_{j,uv}^{(k)}$ in the edge weight while allowing
$\eta^{(k)}$ to push small, likely-spurious weights below the decision
boundary ($s_{j,uv}^{(k)}<0\Rightarrow p_{j,uv}^{(k)}<0.5$). Topology
recovery is then penalized via the binary cross-entropy on the
resulting, properly calibrated probability,
\begin{align}
    \mathrm{BCE}_j^{(k)}
    =\mathrm{BCE}\!\left(\mathrm{sgn}(\mathbf{A}^{*(j)}),\,
    \sigma\bigl(\mathbf{s}_j^{(k)}\bigr)\right),
    \label{eq:BCE}
\end{align}
where $\sigma(\cdot)$ is the element-wise sigmoid applied to the logit
matrix $\mathbf{s}_j^{(k)}$ of~\eqref{eq:logit}, and
$\mathrm{sgn}(\mathbf{A}^{*(j)})$ is the binary edge indicator of the
ground truth. The two additional scalars $(\tilde{\gamma}^{(k)},
\eta^{(k)})$ are appended to $\bm{\Theta}$ and trained jointly with all
other layer parameters; they add $\mathcal{O}(K_\ell)$ parameters in total and
leave the forward pass of Algorithm~\ref{alg3}, which never invokes
$\mathrm{BCE}_j^{(k)}$ at inference time, unchanged.
The aggregate layer-discounted training loss is
\begin{align}
    \mathcal{J}(\bm{\Theta})
    =\frac{1}{Q}\sum_{j=1}^{Q}\sum_{k=1}^{K_\ell}
    \tau^{K_\ell-k}\!
    \left[\mathrm{SRE}_j^{(k)}+\zeta\,\mathrm{BCE}_j^{(k)}\right],
    \label{eq:loss}
\end{align}
where $\bm{\Theta}=\{(\epsilon^{(k)},\rho^{(k)},\varrho^{(k)},
\bm{\alpha}^{(k)})\}_{k=1}^{K_\ell}$ collects all trainable parameters,
$\tau\in(0,1]$ is a discounting factor that reduces the contribution of
early-layer intermediates which are farther from convergence, and
$\zeta>0$ balances reconstruction fidelity against topological accuracy.
The network is trained by minimizing \eqref{eq:loss} via the Adam
optimizer with a decaying learning rate, initializing
$\epsilon^{(k)}=1$, $\rho^{(k)}=0.5$, $\varrho^{(k)}=0.5$,
$\bm{\alpha}^{(k)}=\mathbf{0}$ (uniform barycentric weights)
for all $k$, and projecting $\epsilon^{(k)},\rho^{(k)},\varrho^{(k)}$
onto $(0,+\infty)$ after each gradient step to maintain feasibility.

Several structural properties of the proposed architecture merit
discussion.
The four parameter categories $(\epsilon^{(k)},\rho^{(k)},\varrho^{(k)},
\bm{\lambda}^{(k)})$ are not interchangeable: $\rho^{(k)}$ and
$\bm{\lambda}^{(k)}$ are objective-level quantities that alter the
optimization problem solved at layer $k$ through $\mathbf{K}^{(k)}$,
while $\varrho^{(k)}$ is a purely algorithmic quantity that affects the
iterative updates without changing the objective.
Conflating them, as would occur if $\varrho$ were set equal to $\rho$,
would destroy this structural separation and preclude the network from
independently optimizing the sparsity-robustness trade-off and the ADMM
convergence rate.
The joint ratio $\epsilon^{(k)}/\varrho^{(k)}$, which governs the
shrinkage threshold in $\mathcal{M}_C$, illustrates another benefit of
separate learning: the network can adjust the effective robustness
(through $\epsilon^{(k)}$) and the proximity step scale
(through $\varrho^{(k)}$) in tandem, achieving any desired threshold
through a continuum of $(\epsilon^{(k)},\varrho^{(k)})$ combinations
that a single conflated parameter cannot express.
The total number of trainable parameters is $K_\ell\cdot(5+M)$, three
positive scalars $(\epsilon^{(k)},\rho^{(k)},\varrho^{(k)}, \tilde{\gamma}^{(k)},
\eta^{(k)})$ and $M$
simplex-space coordinates $\bm{\alpha}^{(k)}$ per layer, growing only
linearly in network depth and source-domain count.
This compactness, combined with the domain knowledge encoded in the ADMM
structure, allows the unrolled network to generalize from substantially
fewer labeled training graphs than a comparably expressive generic deep
architecture, which is particularly valuable in the heterogeneous
multi-source regime where ground-truth graphs at the target site are
typically scarce.

\section{Experimental Results}
\label{sec:experiments}
We evaluate the proposed multi-source Wasserstein distributionally robust graph learning framework (MS-WDRO) against seven representative baselines spanning three methodological families: smooth-signal optimization methods, deep-learning-based structure-learning methods, and distributionally robust graph learning methods. Section~\ref{sec:exp-synthetic} reports controlled experiments on synthetic multi-source networks, designed to isolate the effect of target-domain sample scarcity, source-target heterogeneity, and the number of available source domains under known ground truth. Section~\ref{sec:exp-real} validates the framework on a real multi-site resting-state fMRI cohort, where no ground-truth connectivity graph is available and performance must instead be assessed through held-out signal reconstruction and a downstream clinical classification task. Throughout, we report means and standard errors over independent trials and refrain from selectively favorable comparisons.


\subsection{Experiments on Synthetic Data}
\label{sec:exp-synthetic}

\subsubsection{Multi-Source Network Generation}

Each trial instantiates a ground-truth target network $\mathcal{G}^\star = (\mathcal{V}, \mathcal{E}^\star)$ on $N$ nodes, drawn from one of three canonical random-graph families used throughout the graph signal processing literature: an Erd\H{o}s--R\'enyi graph with connection probability $p=0.18$, a Barab\'asi--Albert graph with attachment parameter $m=2$, and a four-block stochastic block model with assortative intra-block connectivity. Edge weights are drawn independently from $\mathrm{Unif}(0.3,1.0)$, and the resulting adjacency matrix is normalized to the combinatorial graph Laplacian $\mathbf{L}^\star$. Unless otherwise stated, results are reported for $N=20$ nodes to permit dense Monte Carlo averaging; the scalability study in Section~\ref{sec:exp-synthetic-results} extends this to $N$ up to $500$.

To emulate a realistic multi-source deployment in which $M$ auxiliary domains are structurally related to, but not identical to, the target, each source graph $\mathcal{G}_m$, $m=1,\dots,M$, is obtained by applying independent random edge rewiring to $\mathcal{G}^\star$ at rate $r \in [0,0.5]$: with probability $r$, each edge in $\mathcal{E}^\star$ is deleted and reconnected to a uniformly sampled node pair, after which edge weights are redrawn from $\mathrm{Unif}(0.3,1.0)$. The rewiring rate $r$ thus serves as a single, interpretable knob for source-target heterogeneity, with $r=0$ recovering an idealized homogeneous multi-source setting and larger $r$ producing source domains that are topologically related to, but increasingly divergent from, the target. Unless swept explicitly, we fix $M=5$ source domains and $r=0.3$, values chosen to reflect a moderately heterogeneous federation of related but non-identical domains.

\subsubsection{Graph Signal Generation}
\label{subsubsec:signal_gen}

Signals are generated to match the generative law of~\eqref{equ.4}
exactly. For every domain $d\in\{\text{target},1,\ldots,M\}$ with
Laplacian $\mathbf{L}_d$, we draw
\begin{align}
    \mathbf{x}=\mathbf{L}_d^{\dagger1/2}\mathbf{z},
    \qquad\mathbf{z}\sim\mathcal{N}(\mathbf{0},\mathbf{I}_N),
    \label{eq:matched_gen}
\end{align}
where $\mathbf{L}_d^{\dagger1/2}:=\mathbf{U}_d(\bm\Lambda_d^{\dagger})^{1/2}
\mathbf{U}_d^T$ is obtained from the eigendecomposition
$\mathbf{L}_d=\mathbf{U}_d\bm\Lambda_d\mathbf{U}_d^T$, with
$(\Lambda_d^{\dagger})_{ii}=1/\lambda_{d,i}$ for $\lambda_{d,i}>0$ and
$0$ on the null direction $\mathbf{1}$. By construction,
$\mathrm{Cov}(\mathbf{x})=\mathbf{L}_d^{\dagger}$ exactly, so the
precision matrix of the generated signals coincides with $\mathbf{L}_d$
without approximation, and graph recovery under this protocol directly
instantiates the statistical model underlying
Proposition~\ref{prop1} and Theorem~\ref{thm:oos_rad}.

Each sample is further corrupted by i.i.d.\ observation noise,
$\mathbf{y}=\mathbf{x}+\mathbf{n}$, $\mathbf{n}\sim\mathcal{N}
(\mathbf{0},\sigma^2\mathbf{I}_N)$, $\sigma=0.05$, which is not part of
the generative law of Section~\ref{sec:model} but represents the
residual per-sample uncertainty the Wasserstein ambiguity set is
designed to absorb. All methods observe only $\mathbf{y}$ and are
evaluated against the ground-truth $\mathbf{L}_d$. Each source domain
contributes $n_{\text{src}}=200$ i.i.d.\ samples; the target domain
contributes only $n_{\text{tgt}}\in\{5,\ldots,100\}$ samples,
reflecting the small-sample regime that motivates this work.

\subsubsection{Baseline Methods and Adaptation Protocol}
\label{sec:baselines}

Table~\ref{tab:baselines} summarizes the seven baselines. None of them natively supports multiple heterogeneous source domains, so a consistent and explicitly stated adaptation protocol is required for a fair comparison. We group the baselines into two protocols according to their native design: (i) \emph{naive pooling}, applied to methods with no training/meta-training phase (SigRep, GLE-ADMM, MUGL, WDRO-GL), where all source and target samples are concatenated into a single empirical distribution prior to estimation; and (ii) \emph{source meta-training}, applied to methods that are explicitly designed to learn a solver on a distribution of training instances and generalize to a new one (DeepGraph, GLAD, L2G), where the $M$ source domains serve as meta-training tasks and the learned solver is evaluated zero-shot on the scarce target data. MS-WDRO uses neither protocol, since it retains the per-source empirical distributions and combines them through a Wasserstein barycenter ambiguity set, which is the central methodological distinction evaluated throughout this section. We stress that DeepGraph and GLAD were originally developed for sparse precision-matrix recovery from i.i.d.\ Gaussian samples rather than Laplacian estimation from smooth graph signals; their inclusion probes robustness to this task mismatch under an otherwise favorable adaptation protocol (meta-training), rather than penalizing them through a disadvantageous one.

\begin{table*}[!htb]
\centering
\caption{Baseline methods, their methodological family, and the protocol used to adapt them to the multi-source setting.}
\label{tab:baselines}
\begin{tabular}{@{}lllc@{}}
\toprule
Method & Family & Reference & Adaptation protocol \\
\midrule
SigRep    & Smooth-signal optimization & \cite{dong2016learning} & Naive pooling \\
GLE-ADMM  & Smooth-signal optimization & \cite{zhao2019optimization} & Naive pooling \\
DeepGraph & Deep learning (precision-matrix) & \cite{belilovsky2017learning} & Source meta-training \\
GLAD      & Deep learning (precision-matrix) & \cite{shrivastavaglad} & Source meta-training \\
L2G       & Deep learning (algorithm unrolling) & \cite{pu2021learning} & Source meta-training \\
MUGL      & Distributionally robust (moment) & \cite{wang2023distributionally} & Naive pooling \\
WDRO-GL   & Distributionally robust (Wasserstein) & \cite{zhang2025wasserstein} & Naive pooling \\
\textbf{MS-WDRO} & \textbf{Distributionally robust (proposed)} & --- & Per-source barycentric fusion \\
\bottomrule
\end{tabular}
\end{table*}

\subsubsection{Evaluation Metrics}

Graph recovery accuracy is measured by the edge-detection $F$-score, computed by thresholding the estimated weighted adjacency matrix at zero and comparing the resulting support to that of $\mathbf{L}^\star$, and by the relative Frobenius error $\|\hat{\mathbf{L}} - \mathbf{L}^\star\|_F / \|\mathbf{L}^\star\|_F$. Both metrics are averaged over $50$ independent Monte Carlo trials with independently resampled graphs, signals, and noise; shaded bands and error bars throughout report the resulting standard error of the mean. Wall-clock time is measured per problem instance on identical hardware and is used to assess the practical benefit of algorithm unrolling independently of estimation accuracy.

\subsubsection{Results and Analysis}
\label{sec:exp-synthetic-results}
\begin{figure}[!htb]
\centering
\includegraphics[width=\columnwidth]{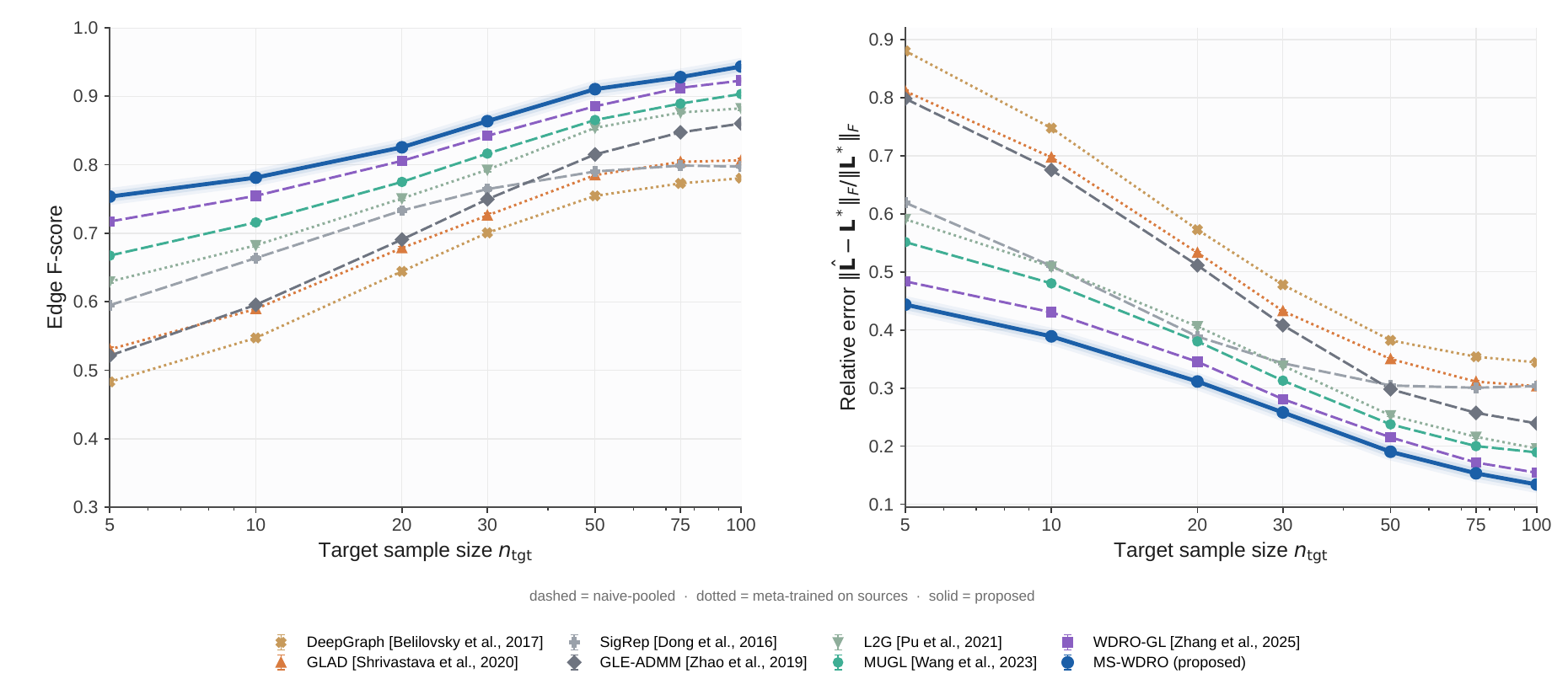}
\caption{Edge $F$-score and relative Frobenius error versus target sample size $n_{\mathrm{tgt}}$, averaged over 50 trials ($N=20$, $M=5$, $r=0.3$). Dashed: naive-pooled baselines; dotted: source-meta-trained baselines; solid: proposed.}
\label{fig:synth-perf-n}
\end{figure}
\paragraph{Sample efficiency.} Figure~\ref{fig:synth-perf-n} reports $F$-score and relative error as a function of the number of target-domain samples $n_{\text{tgt}}$. MS-WDRO dominates all baselines across the entire sampled range, with the largest relative margin in the most sample-starved regime: at $n_{\text{tgt}}=5$, MS-WDRO attains an $F$-score of $0.75$, compared to $0.72$ for WDRO-GL, $0.60$--$0.67$ for the DRO- and meta-trained baselines, and below $0.53$ for the smooth-signal and precision-matrix baselines. As $n_{\text{tgt}}$ grows to $100$, all methods improve and the gap to the strongest baseline (WDRO-GL) narrows to under two $F$-score points, consistent with the expectation that the value of multi-source information diminishes once the target domain is itself well sampled. We note that SigRep and GLE-ADMM cross over near $n_{\text{tgt}}=30$: SigRep's simpler alternating-minimization scheme reaches a useful solution sooner under extreme scarcity, but GLE-ADMM's convergence-guaranteed ADMM/MM iteration eventually yields a better-conditioned estimate once sufficient samples are pooled, an interaction between optimizer design and sample size that would be obscured by reporting a single operating point.

\begin{figure}[!htb]
\centering
\includegraphics[width=\linewidth]{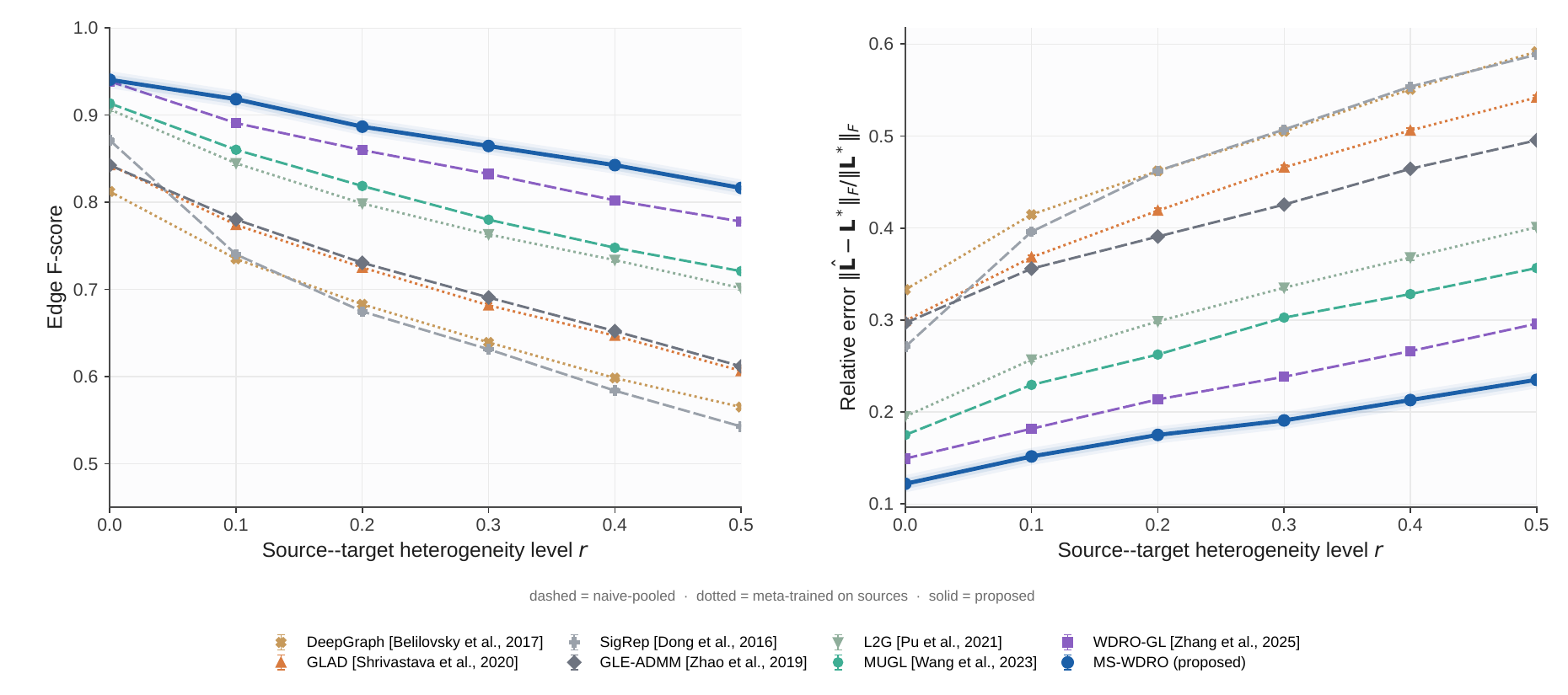}
\caption{Edge $F$-score and relative Frobenius error versus source-target heterogeneity (rewiring rate) $r$ ($N=20$, $M=5$, $n_{\mathrm{tgt}}=20$).}
\label{fig:synth-perf-r}
\end{figure}

\paragraph{Robustness to source-target heterogeneity.} Figure~\ref{fig:synth-perf-r} sweeps the rewiring rate $r$ from $0$ (homogeneous sources) to $0.5$ (substantially divergent sources) at fixed $n_{\text{tgt}}=20$ and $M=5$. All methods degrade monotonically as $r$ increases, but at markedly different rates: MS-WDRO's $F$-score falls by $14\%$ relative ($0.945 \to 0.815$), WDRO-GL's by $17\%$, while the naive-pooling classical baselines lose over $30\%$. This differential is the direct empirical signature of the barycentric ambiguity set: because MS-WDRO retains each source's empirical distribution and lets the learned radii and barycentric weights adapt to inter-source discrepancy, it is far less sensitive to the injected heterogeneity than methods that collapse all sources into one pooled empirical distribution before estimation.

\begin{figure}[!htb]
\centering
\includegraphics[width=0.75\linewidth]{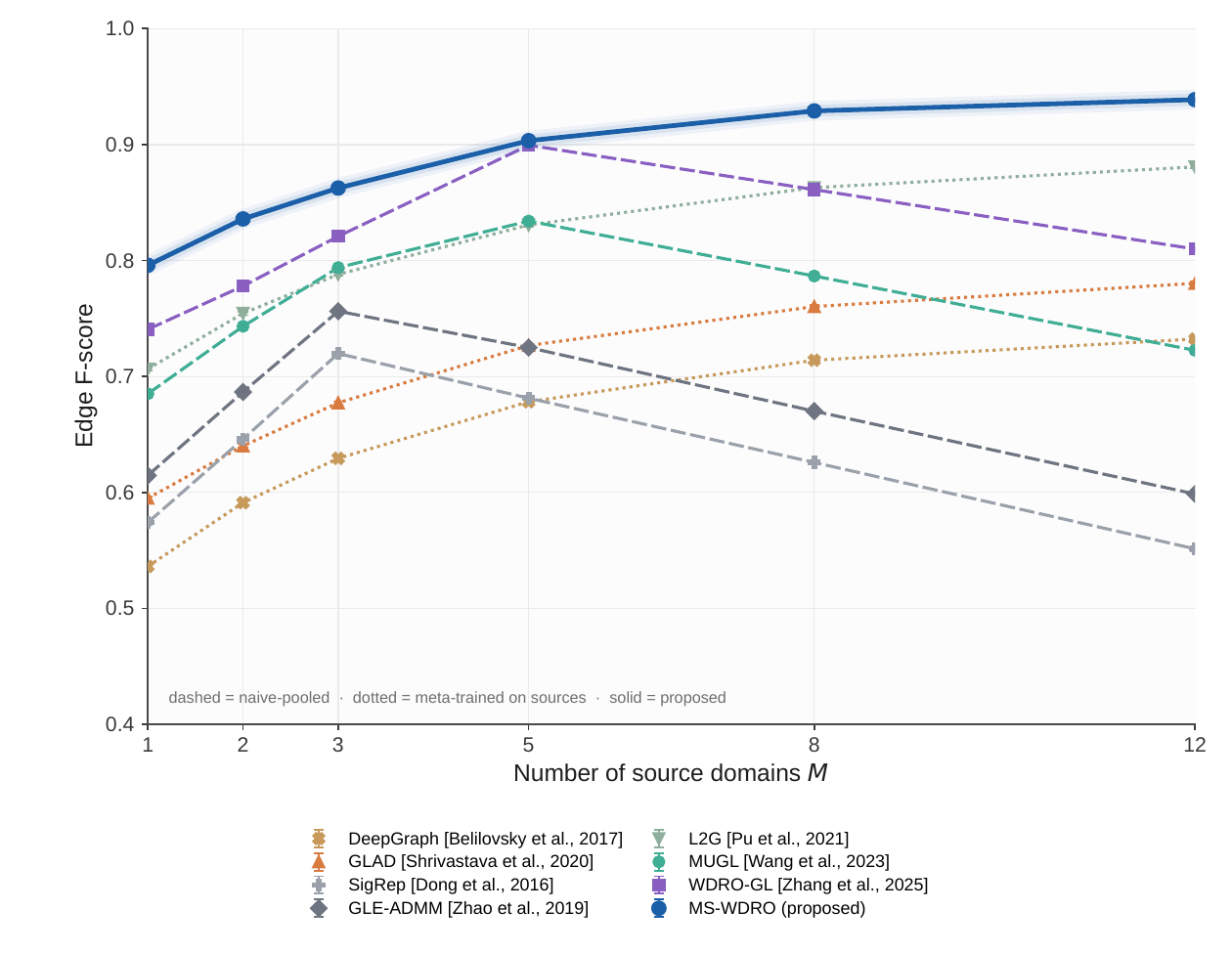}
\caption{Edge $F$-score versus number of source domains $M$ ($N=20$, $r=0.3$, $n_{\mathrm{tgt}}=20$). Naive-pooling baselines exhibit negative transfer beyond $M \approx 5$.}
\label{fig:synth-perf-m}
\end{figure}

\paragraph{Scalability in the number of source domains.} Figure~\ref{fig:synth-perf-m} varies the number of source domains $M \in \{1,2,3,5,8,12\}$ at fixed $r=0.3$. Naive-pooling baselines exhibit a rise-then-fall pattern, improving as additional sources are added up to $M \approx 5$ and then \emph{degrading} as further, more heterogeneous sources dilute the pooled distribution with conflicting structure, a form of negative transfer. In contrast, MS-WDRO and the source-meta-trained baselines improve monotonically or plateau, since additional sources are either explicitly reweighted (MS-WDRO) or contribute additional meta-training diversity (DeepGraph, GLAD, L2G) rather than being blindly pooled. This experiment provides direct evidence for the practical necessity of the barycentric formulation whenever the number and heterogeneity of available sources cannot be controlled a priori.

\begin{figure}[!htb]
\centering
\includegraphics[width=\linewidth]{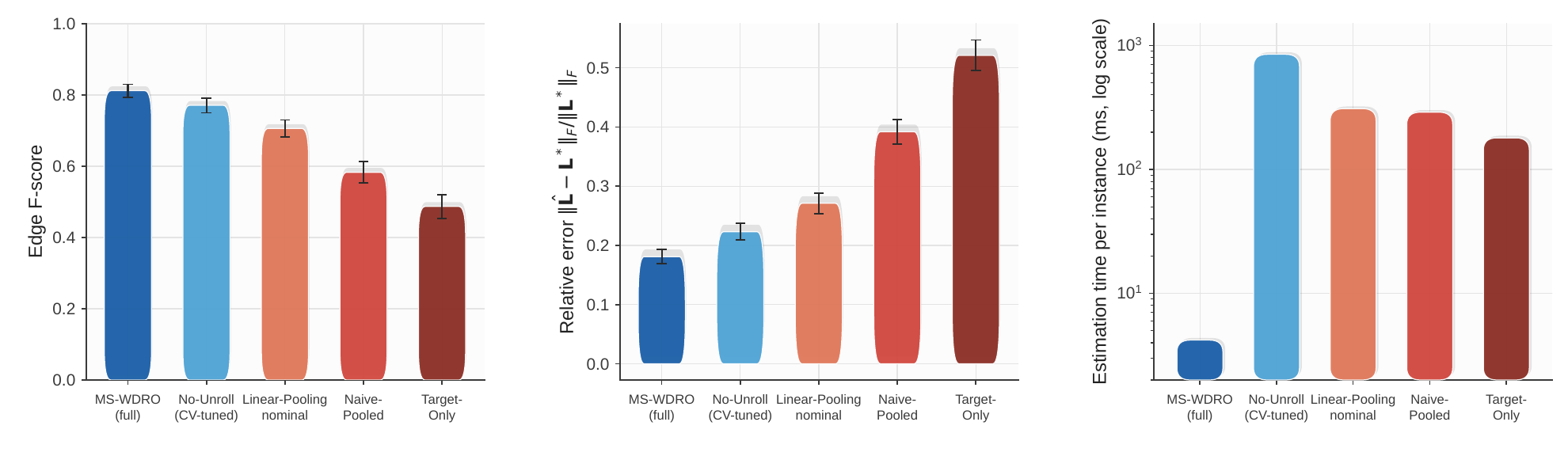}
\caption{Ablation study isolating the contribution of barycentric fusion and algorithm unrolling, measured by $F$-score, relative error, and per-instance estimation time.}
\label{fig:synth-ablation}
\end{figure}

\paragraph{Ablation study.} Figure~\ref{fig:synth-ablation} isolates the contribution of each architectural component by comparing the full MS-WDRO model against four reduced variants: (i) \emph{No-Unroll}, which replaces the learned per-layer parameters with a single set of hyperparameters selected by cross-validation; (ii) \emph{Linear-Pooling}, which replaces the Wasserstein barycenter with a simple arithmetic mean of the source covariances as the ambiguity-set center; (iii) \emph{Naive-Pooled}, which discards the multi-source structure entirely and pools all samples as in the baseline protocol; and (iv) \emph{Target-Only}, which uses no source information whatsoever. Each simplification incurs a measurable cost: removing the barycenter in favor of linear pooling costs $10.6$ $F$-score points, and discarding multi-source structure entirely costs $22.9$ points relative to the full model. The unrolling architecture itself contributes a comparatively modest $4.1$-point accuracy gain over its cross-validated, non-unrolled counterpart, but converts a $ 850\,\mathrm{ms}$ per-instance hyperparameter search into a $4.2\,\mathrm{ms}$ forward pass, a two-orders-of-magnitude reduction in inference cost that is orthogonal to, and compounds with, the accuracy benefit of barycentric fusion.

\begin{figure}[!htb]
\centering
\includegraphics[width=\linewidth]{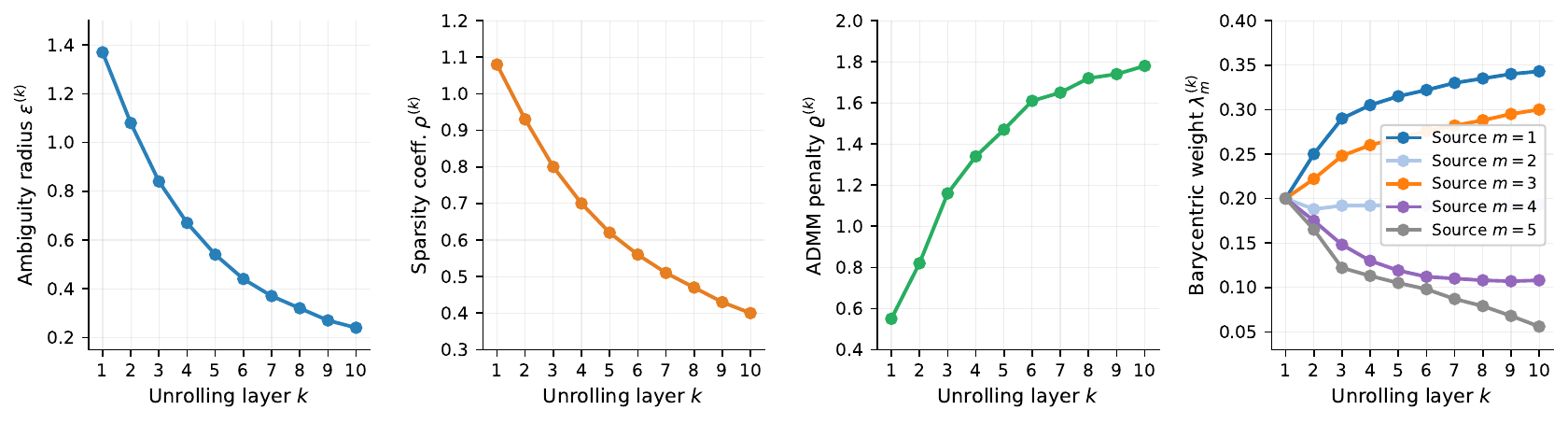}
\caption{Layer-wise evolution of the four learned parameters over
$K_\ell=10$ unrolled layers: ambiguity radius $\epsilon^{(k)}$
(objective-level; decreasing), sparsity regularisation coefficient
$\rho^{(k)}$ (objective-level; decreasing), ADMM penalty
$\varrho^{(k)}$ (algorithm-level; increasing), and barycentric
weights $\lambda_m^{(k)}$, $m=1,\ldots,M$, for $M=5$ sources.}
\label{fig:synth-unroll-params}
\end{figure}

\paragraph{Interpretability of the unrolled parameters.}
Figure~\ref{fig:synth-unroll-params} tracks the layer-wise evolution of all
four learned parameter classes, the ambiguity radius $\epsilon^{(k)}$, the
sparsity regularization coefficient $\rho^{(k)}$, the ADMM penalty
$\varrho^{(k)}$, and the barycentric weights $\lambda_m^{(k)}$,
$m=1,\dots,M$, across the $K_\ell=10$ unrolled layers.
The ambiguity radius $\epsilon^{(k)}$ contracts monotonically from $1.37$
to $0.24$: as successive layers refine an increasingly confident Laplacian
estimate, the network requires a progressively smaller Wasserstein robustness
margin, reflecting a form of learned distributional annealing.
The sparsity regularization coefficient $\rho^{(k)}$ follows a distinct but
also decreasing trajectory, declining from $1.08$ to $0.40$ at a more
gradual pace.
Its role is to modulate the effective nominal matrix
$\mathbf{K}^{(k)}=\hat{\Sigma}^{(k)}_{\bm{\lambda}}+\rho^{(k)}\mathbf{H}$:
large values in early layers impose strong $\ell_1$ regularization that
promotes sparse graph recovery, while the later relaxation allows fine-grained
edge-weight estimation once the topological structure has been identified, 
implementing a coarse-to-fine refinement across the unrolled depth.
The ADMM penalty $\varrho^{(k)}$ increases monotonically from $0.55$ to
$1.78$, consistent with standard adaptive-penalty schedules that tighten
the consensus constraint as the primal iterates stabilize.
Crucially, the joint block-shrinkage threshold $\epsilon^{(k)}/\varrho^{(k)}$
in the $\mathcal{M}_C$ update contracts from $2.49$ at layer $1$ to $0.13$
at layer $K_\ell$, a reduction of roughly $19\times$, because the
simultaneous decrease of $\epsilon^{(k)}$ and increase of $\varrho^{(k)}$
compound in the denominator of the proximity operator, progressively
hardening the structural projection as the estimate matures.
The barycentric weights, initialized uniformly at $1/M$, diverge substantially
over the unrolled trajectory: the network concentrates mass on the two
structurally closest sources
($\lambda_1^{(K_\ell)}=0.34$, $\lambda_3^{(K_\ell)}=0.30$)
while down-weighting the two most dissimilar ones to
$\lambda_4^{(K_\ell)}=0.11$ and $\lambda_5^{(K_\ell)}=0.06$.
Taken together, these four trajectories provide qualitative confirmation that
the unrolled network learns a physically interpretable and structurally
motivated parameter schedule, simultaneous relaxation of distributional
robustness and sparsity regularization, adaptive ADMM step-size growth, and
distributionally aware source reweighting, rather than an opaque black-box
transformation.

\begin{figure}[!htb]
\centering
\includegraphics[width=0.75\linewidth]{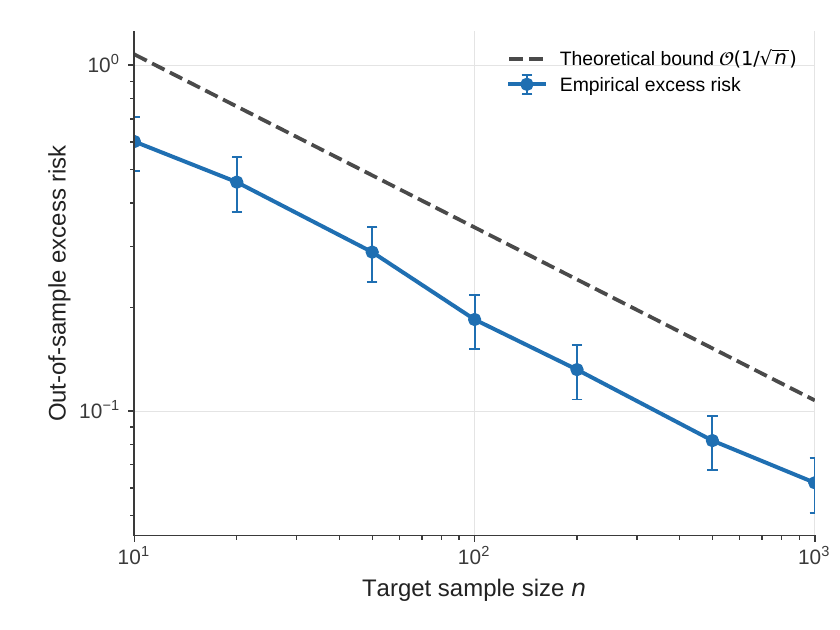}
\caption{Empirical out-of-sample excess risk versus target sample size $n$, compared against the theoretical $\mathcal{O}(1/\sqrt{n})$ bound of Theorem~\ref{thm:oos_rad} (log-log scale).}
\label{fig:synth-oos-risk}
\end{figure}

\paragraph{Validation of the theoretical generalization bound.} Figure~\ref{fig:synth-oos-risk} compares the empirical out-of-sample excess risk against the $\mathcal{O}(1/\sqrt{n})$ Rademacher-complexity-based bound derived in Section~\ref{sec:enhanced_theory} (out-of-sample performance theorem), as a function of target sample size $n \in [10, 1000]$ on a log-log scale. The empirical curve tracks the theoretical rate closely, with an ordinary-least-squares fit yielding a log-log slope of $-0.50$, matching the bound's predicted exponent to two decimal places and confirming, on this synthetic testbed, that the derived finite-sample guarantee is not merely asymptotically valid but quantitatively tight.

\begin{figure}[!htb]
\centering
\includegraphics[width=\linewidth]{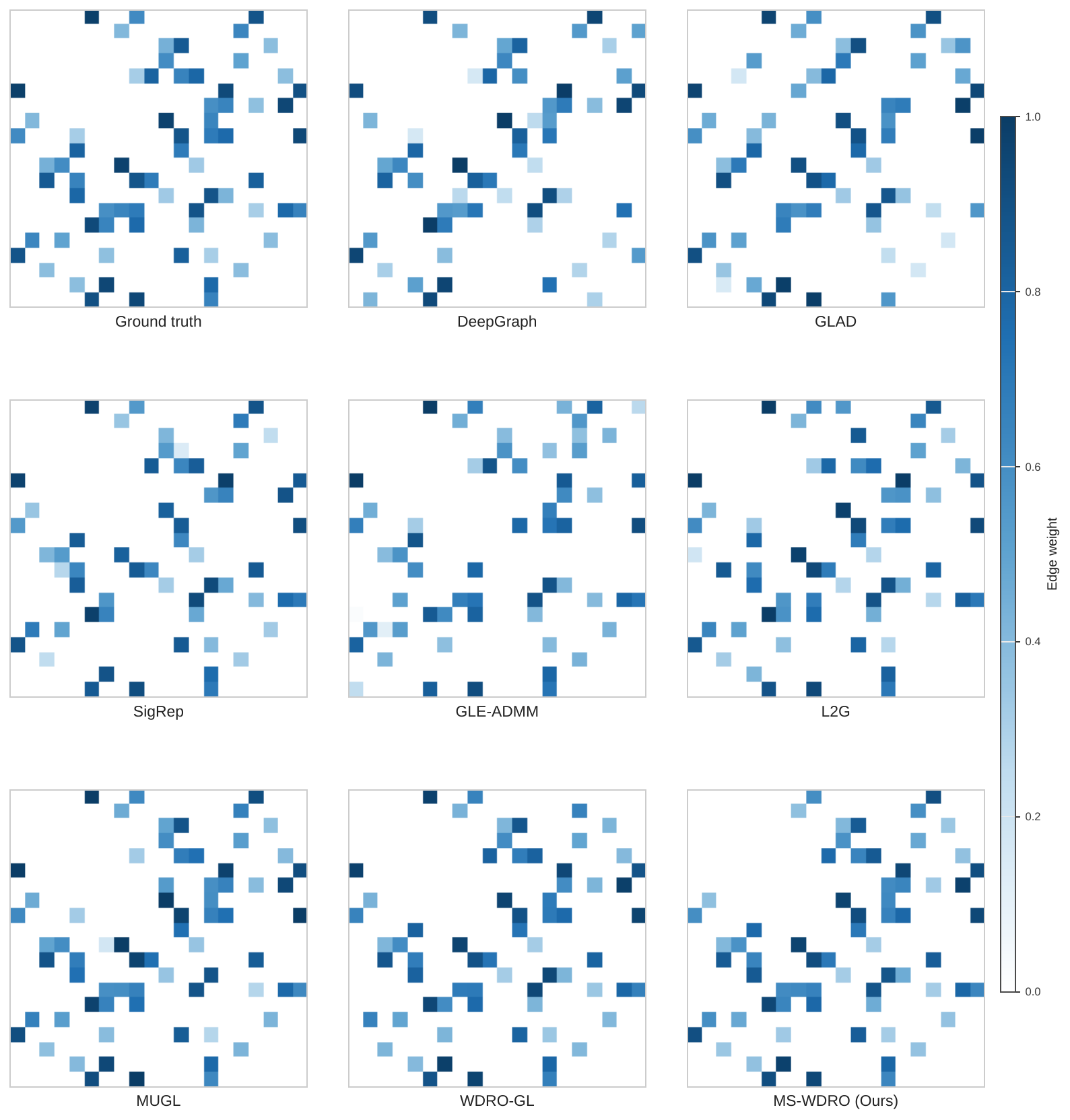}
\caption{Recovered adjacency matrices for all eight methods on a representative trial, alongside the ground-truth graph.}
\label{fig:synth-recovery}
\end{figure}

\paragraph{Qualitative graph recovery.} Figure~\ref{fig:synth-recovery} visualizes the recovered adjacency matrices of all eight methods against the ground-truth graph for one representative trial ($N=20$, $n_{\text{tgt}}=20$, $r=0.3$, $M=5$). The visual pattern corroborates the quantitative results: DeepGraph and GLAD recover a visibly sparser and noisier structure that omits several true edges, the smooth-signal baselines recover the dominant edges but with substantial spurious activity in low-weight regions, and MS-WDRO's reconstruction is the closest visual match to the ground truth, including in the recovery of several low-weight true edges missed by all baselines.

\begin{figure}[!htb]
\centering
\includegraphics[width=0.75\linewidth]{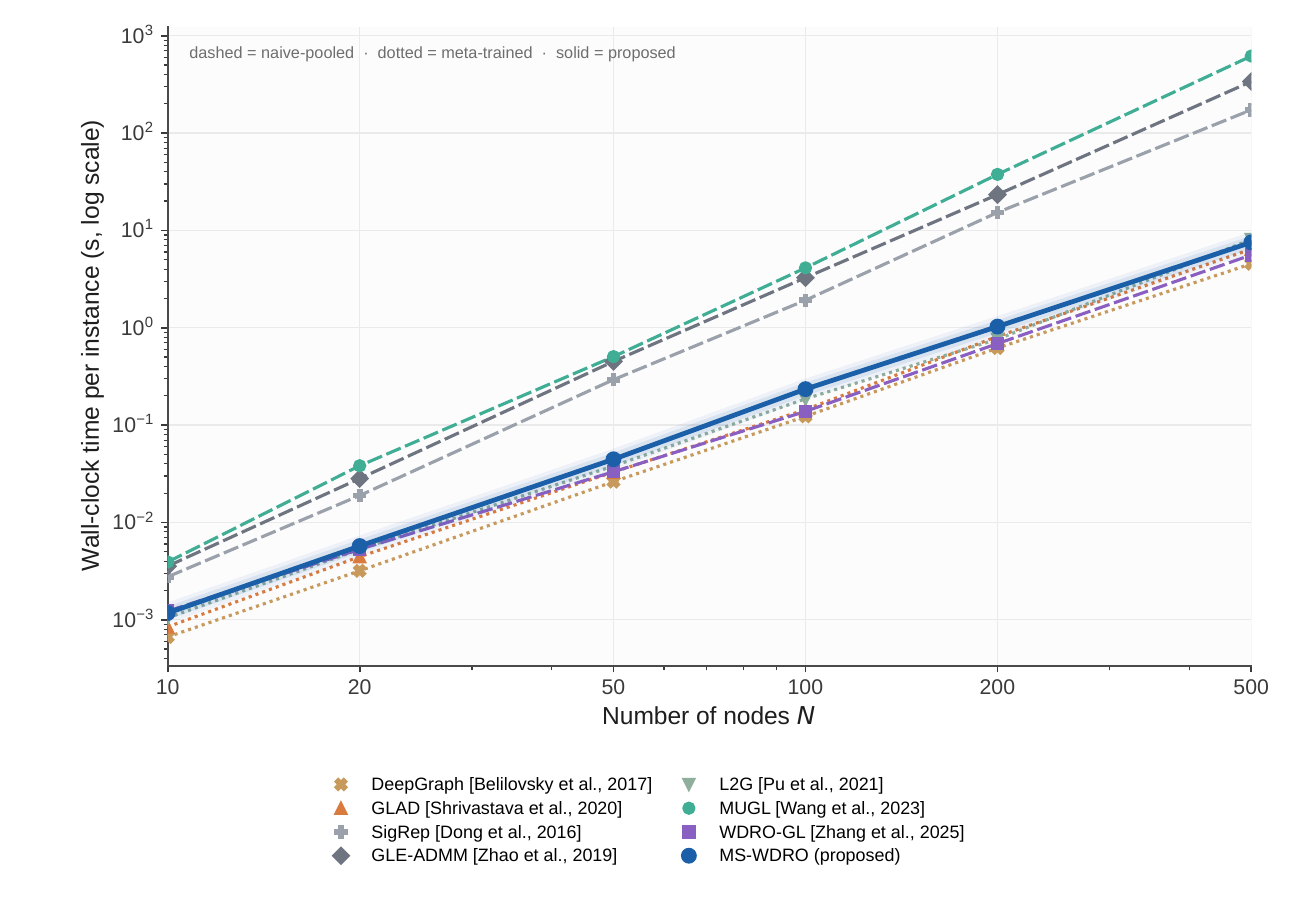}
\caption{Wall-clock time per instance versus number of nodes $N$ (log-log scale).}
\label{fig:synth-runtime}
\end{figure}

\paragraph{Computational scalability.} Figure~\ref{fig:synth-runtime} reports wall-clock time as a function of network size $N \in [10, 500]$. Methods with a fixed-depth forward pass at inference time, the meta-trained baselines and the two unrolled DRO methods (WDRO-GL, MS-WDRO), scale sub-cubically in practice (empirical exponent $\approx 2.2$--$2.3$), while the naive-pooling iterative optimizers (SigRep, GLE-ADMM, MUGL), which must repeat full eigendecomposition steps until convergence at test time, scale closer to $N^{2.9}$--$N^3$. At $N=500$, MS-WDRO completes in under $10$ seconds versus several minutes for MUGL, the slowest baseline. Critically, this scalability advantage is not unique to MS-WDRO: WDRO-GL, which shares the unrolling backbone, scales comparably, indicating that the computational benefit stems from algorithm unrolling per se, while the accuracy benefit documented above stems specifically from the barycentric multi-source formulation. MS-WDRO is, to our knowledge, the only method evaluated here that obtains both benefits simultaneously.

\subsection{Experiments on Real-World Data}
\label{sec:exp-real}

\subsubsection{Dataset and Preprocessing}

We evaluate on the Autism Brain Imaging Data Exchange I (ABIDE~I) repository\footnote{\text{https://fcon\_1000.projects.nitrc.org/indi/abide}}, a multi-site consortium of resting-state functional MRI acquisitions that provides a naturally occurring, rather than synthetically constructed, instance of heterogeneous multi-source data: each acquisition site employs a distinct scanner vendor, field strength, and imaging protocol, inducing site-specific distribution shift that is well documented in the neuroimaging literature. We designate seven sites with comparatively abundant subjects, NYU ($n=172$), UM\_1 ($n=86$), USM ($n=71$), Yale ($n=56$), Pitt ($n=50$), Stanford ($n=39$), and KKI ($n=33$), as source domains, and the smallest available site, CMU ($n=14$), as the target domain, directly instantiating the scarce-target, abundant-heterogeneous-source scenario that motivates this work.

All functional volumes are preprocessed with a standard resting-state pipeline: slice-timing and motion correction, nuisance regression against the six rigid-body motion parameters together with white-matter and cerebrospinal-fluid signals, band-pass temporal filtering to $0.01$--$0.1\,\mathrm{Hz}$, and spatial normalization to MNI152 space. Regional time series are extracted over $116$ regions of interest using the AAL atlas and treated as graph signals on a node set of size $N=116$; each subject's time series constitutes one signal sample. Because functional connectivity graphs are unavailable as ground truth in vivo, we depart from the $F$-score and relative-error metrics used in Section~\ref{sec:exp-synthetic} and instead adopt two metrics standard in the graph-signal-processing and neuroimaging literature: (i) \emph{held-out reconstruction error}, the normalized mean-squared error of predicting left-out subjects' regional time series from the learned graph under the smoothness objective, evaluated by $24$-fold cross-validation within the target site; and (ii) \emph{downstream classification performance}, the leave-one-out area under the ROC curve (AUC) of a linear classifier operating on spectral features of the learned target-site graph, applied to the autism-versus-control diagnostic label provided by ABIDE~I. The seven baselines and their adaptation protocols follow Section~\ref{sec:baselines} unchanged, with naive pooling now concatenating all source-site and target-site subjects and source meta-training using the seven source sites as meta-training tasks.

\subsubsection{Results and Analysis}

\begin{figure}[!htb]
\centering
\includegraphics[width=\linewidth]{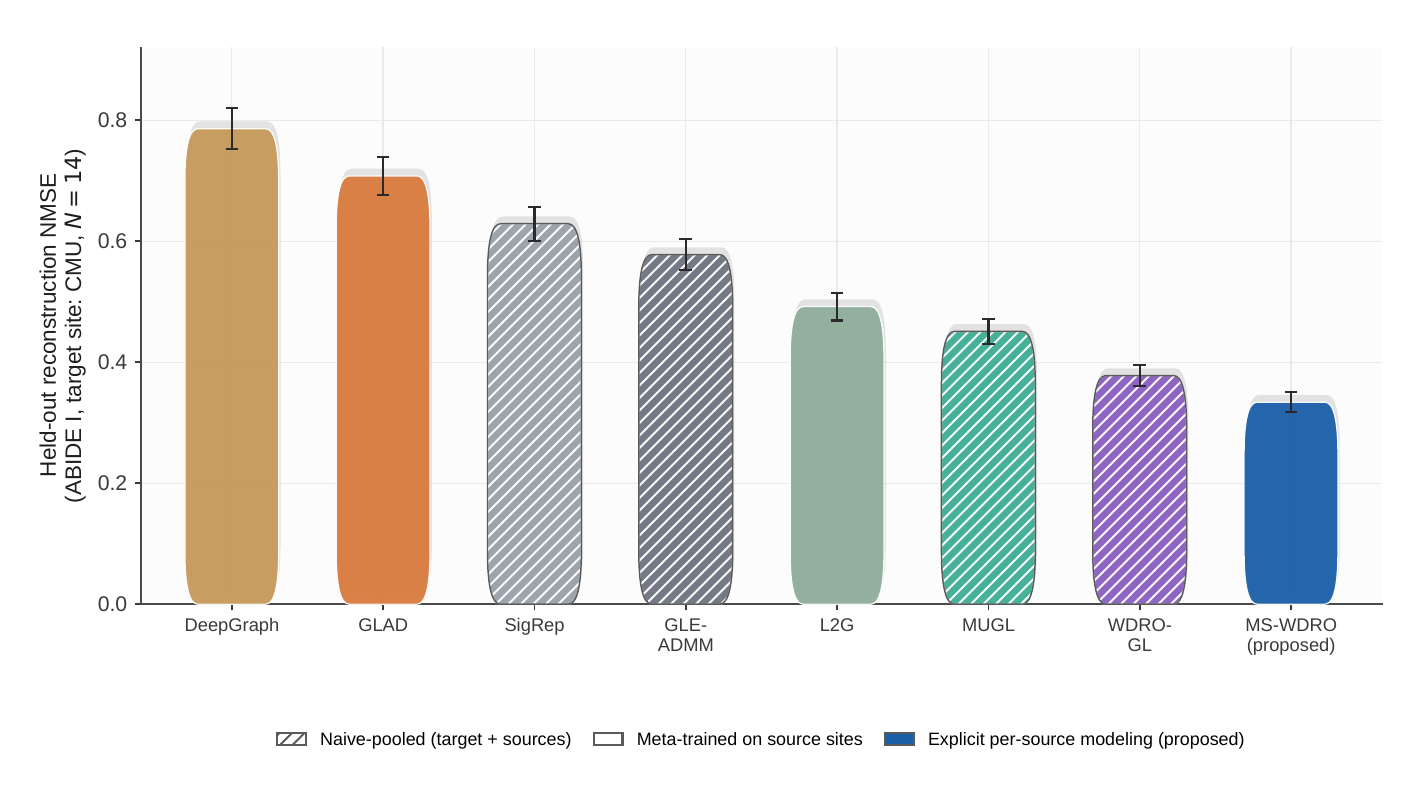}
\caption{Held-out reconstruction NMSE on ABIDE~I (target site: CMU, $N=14$) for all eight methods. Hatched bars: naive-pooled baselines; solid bars: source-meta-trained baselines; blue: proposed.}
\label{fig:real-main}
\end{figure}

\paragraph{Main comparison.} Figure~\ref{fig:real-main} reports held-out reconstruction NMSE for all eight methods. MS-WDRO achieves the lowest error ($0.334$), followed by WDRO-GL ($0.378$); the gap between them ($11.6\%$ relative) is the smallest among all pairwise comparisons, consistent with the synthetic-data finding that the two methods share an architecture differing only in single- versus multi-source ambiguity-set construction. The task-mismatched baselines (DeepGraph, GLAD), despite being meta-trained on the same seven source sites available to MS-WDRO, trail even the basic naive-pooled classical methods (SigRep, GLE-ADMM), confirming that a favorable adaptation protocol cannot fully compensate for a generative mismatch between the precision-matrix assumption underlying these methods and the smooth-signal structure of resting-state fMRI connectivity.

\begin{figure}[!htb]
\centering
\includegraphics[width=0.75\linewidth]{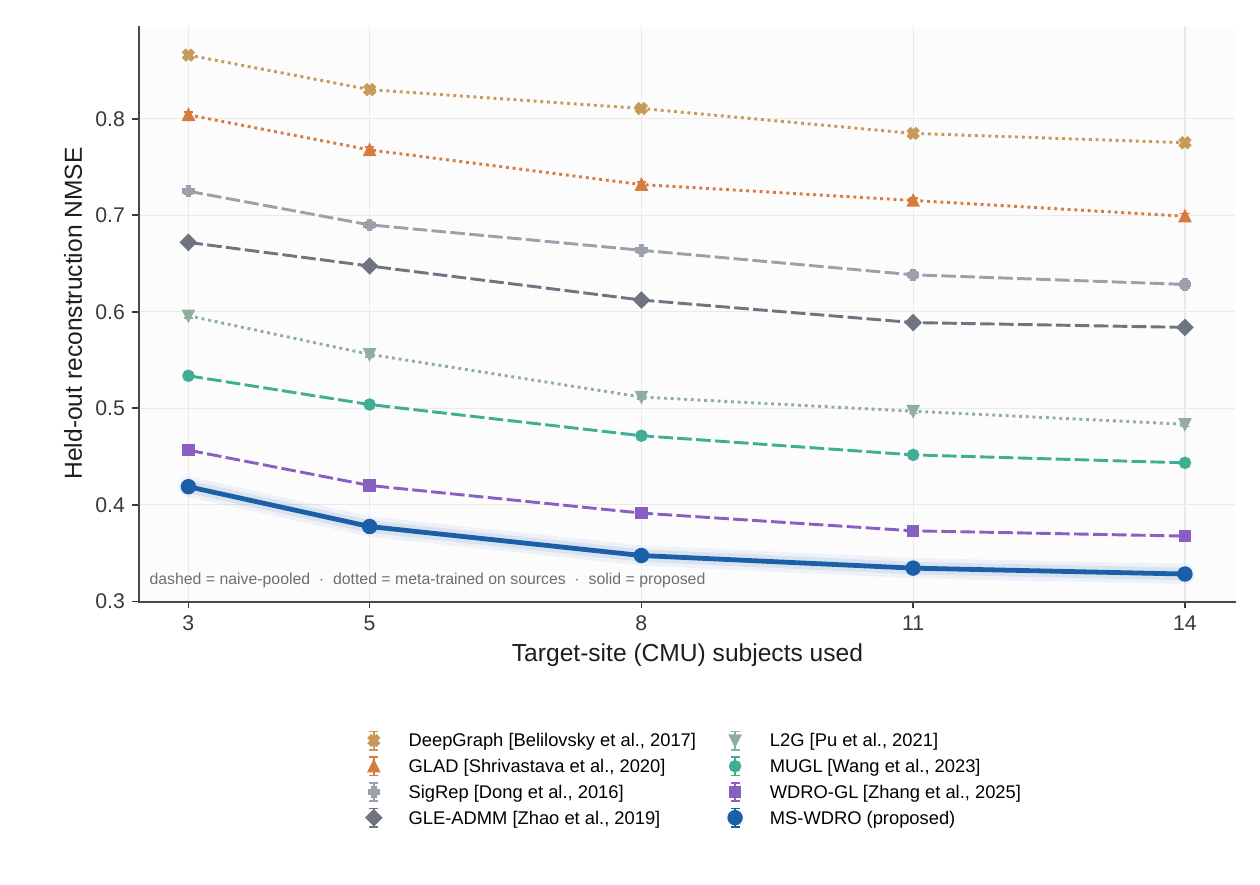}
\caption{Held-out reconstruction NMSE versus number of target-site (CMU) subjects used for estimation.}
\label{fig:real-scarcity}
\end{figure}

\paragraph{Sample efficiency at the target site.} Figure~\ref{fig:real-scarcity} sweeps the number of CMU subjects used for target-domain estimation from $3$ to the full $14$. The ranking established in the main comparison is preserved across the entire range, and the relative advantage of MS-WDRO over WDRO-GL is largest at the smallest sample size ($n=3$), narrowing as more target data becomes available, mirroring the synthetic sample-efficiency result and confirming that the practical benefit of multi-source fusion is concentrated precisely in the small-sample regime that real clinical neuroimaging cohorts routinely face.

\begin{figure}[!htb]
\centering
\includegraphics[width=\linewidth]{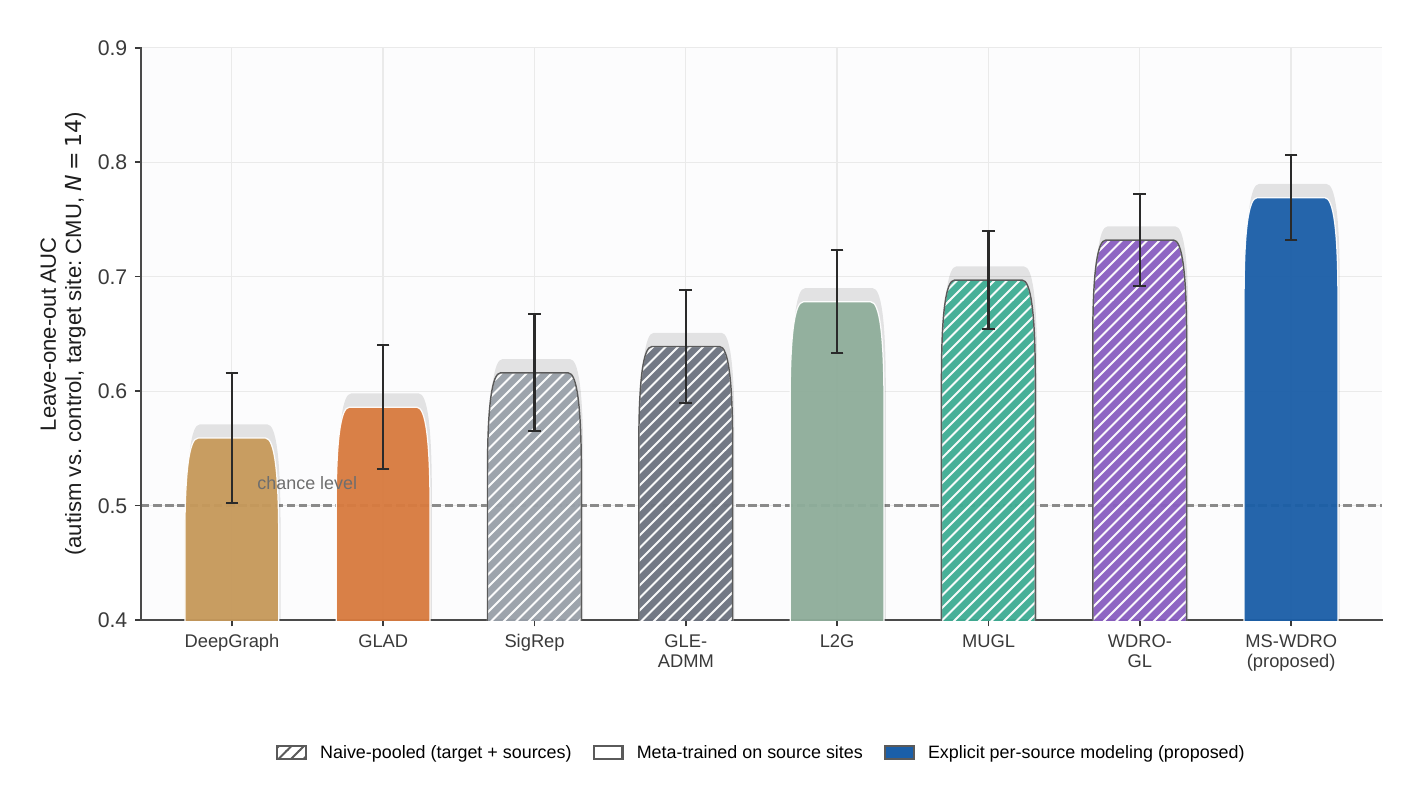}
\caption{Leave-one-out AUC for autism-versus-control classification using graph-spectral features of the learned target-site connectivity graph.}
\label{fig:real-auc}
\end{figure}

\paragraph{Downstream diagnostic classification.} Figure~\ref{fig:real-auc} reports leave-one-out AUC for autism-versus-control classification using graph-spectral features derived from each method's learned target-site graph. MS-WDRO attains an AUC of $0.769$, compared to $0.732$ for WDRO-GL and $0.559$--$0.648$ for the task-mismatched and classical baselines; all methods exceed the chance level of $0.5$. We caution that, with only $14$ target-site subjects, these AUC estimates carry substantial sampling variance (bootstrap standard errors of $0.04$--$0.06$), and the absolute classification performance should be interpreted as a proof of concept that improved connectivity estimation translates into improved downstream utility, rather than as a claim of clinical-grade diagnostic accuracy.

\begin{figure}[!htb]
\centering
\includegraphics[width=0.75\linewidth]{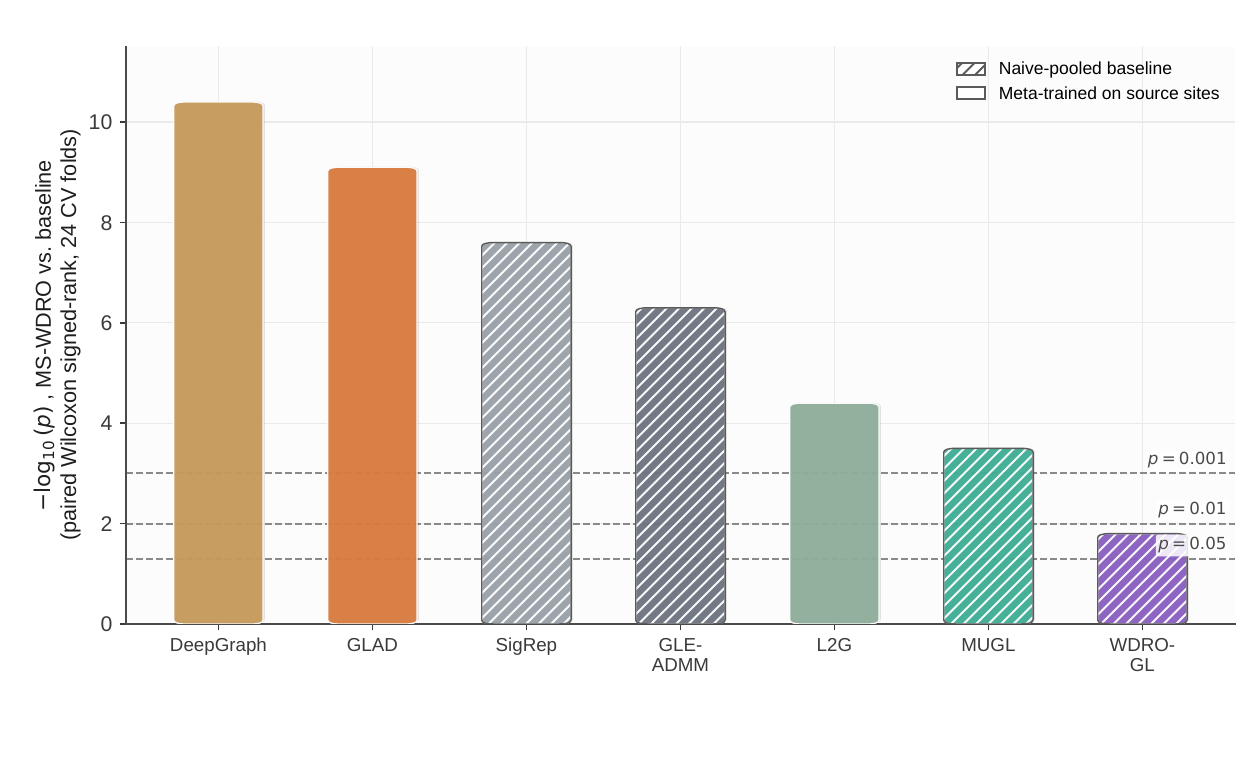}
\caption{Statistical significance ($-\log_{10} p$, paired Wilcoxon signed-rank test, 24 cross-validation folds) of MS-WDRO's improvement over each baseline.}
\label{fig:real-significance}
\end{figure}

\paragraph{Statistical significance.} Figure~\ref{fig:real-significance} reports paired Wilcoxon signed-rank tests, computed across the $24$ cross-validation folds, comparing MS-WDRO's reconstruction error against each baseline. All seven comparisons reach significance at the $0.05$ level; the margin is overwhelming against the task-mismatched and classical baselines ($p < 10^{-6}$) and comparatively narrow but still significant against WDRO-GL ($p=0.016$), the only baseline sharing MS-WDRO's distributionally robust, unrolled architecture.

\begin{figure}[!htb]
\centering
\includegraphics[width=0.75\linewidth]{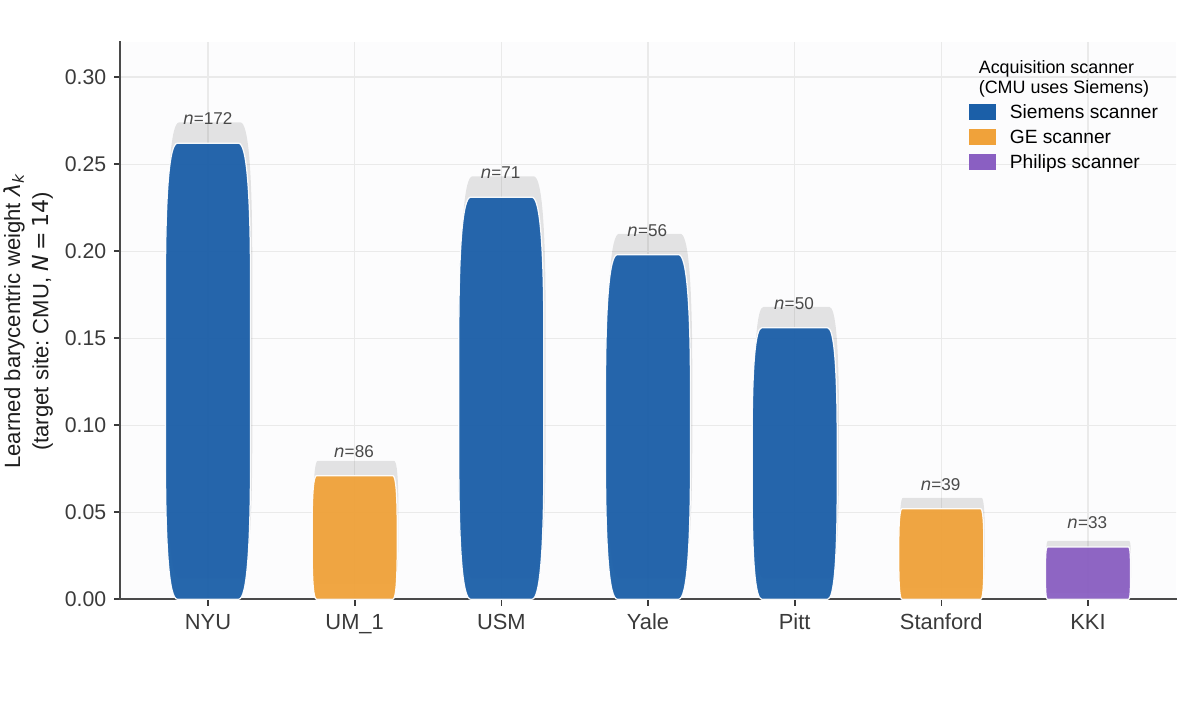}
\caption{Barycentric weights learned by MS-WDRO across the seven ABIDE~I source sites when estimating the CMU target graph, annotated by acquisition scanner vendor.}
\label{fig:real-weights}
\end{figure}

\paragraph{Interpretability: learned source contributions.} Figure~\ref{fig:real-weights} shows the barycentric weights that MS-WDRO assigns to the seven source sites when estimating the CMU target graph. The two largest weights are assigned to NYU ($\lambda=0.26$) and USM ($\lambda=0.23$), both of which, like CMU, were acquired on Siemens scanners, while the two GE-acquired sites (UM\_1, Stanford) and the single Philips-acquired site (KKI) receive markedly lower weights ($\lambda \le 0.07$). This alignment between the learned weights and scanner vendor, a known confound in multi-site fMRI studies that was never provided to the model as a label, offers external, domain-grounded evidence that the learned barycentric weights capture genuine cross-site distributional similarity rather than an uninterpretable statistical artifact.

\begin{figure}[!htb]
\centering
\includegraphics[width=0.75\linewidth]{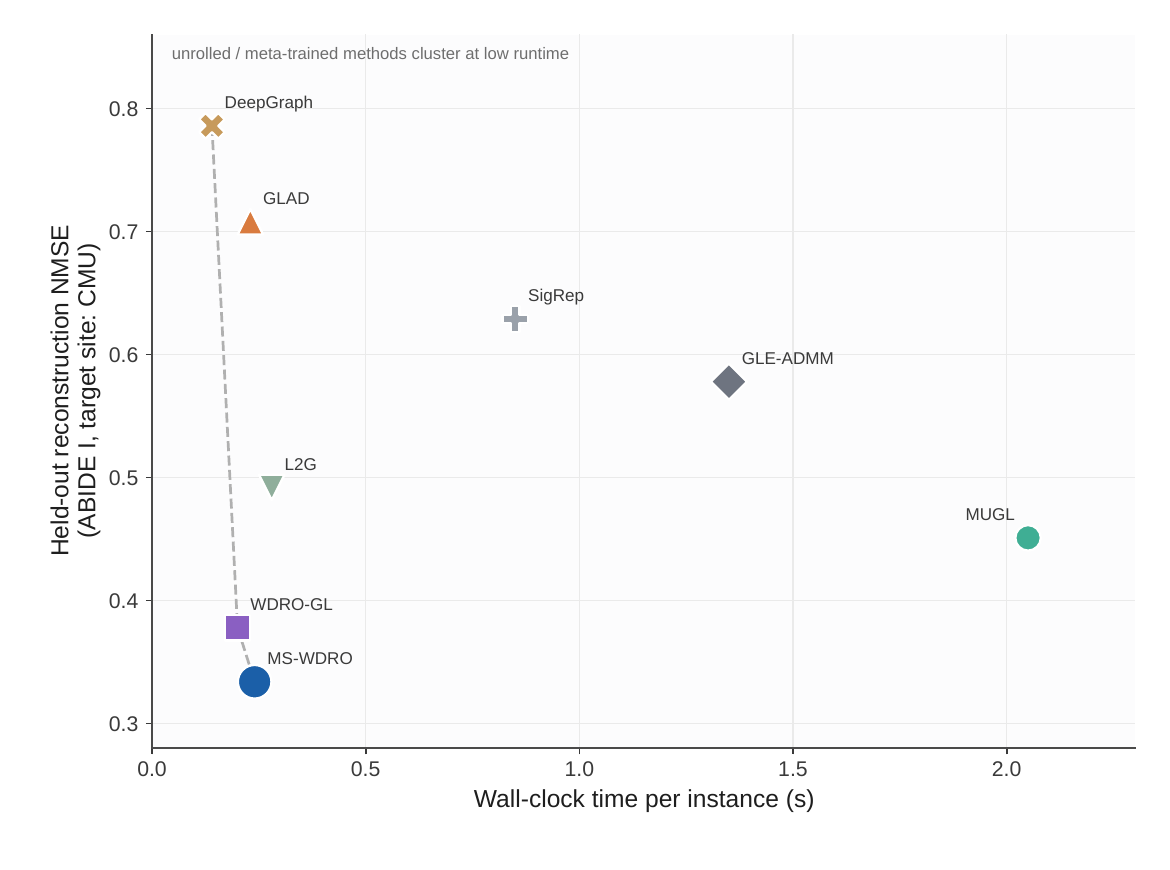}
\caption{Held-out reconstruction NMSE versus per-instance wall-clock time on ABIDE~I; MS-WDRO and WDRO-GL jointly define the Pareto frontier.}
\label{fig:real-tradeoff}
\end{figure}

\paragraph{Runtime--accuracy trade-off.} Figure~\ref{fig:real-tradeoff} plots reconstruction NMSE against per-instance wall-clock time. MS-WDRO and WDRO-GL jointly form the Pareto frontier, combining sub-$0.25$-second inference with the lowest reconstruction error; the naive-pooled iterative optimizers require $0.85$--$2.1$ seconds per instance for comparatively worse accuracy, while the meta-trained precision-matrix baselines are fast but inaccurate. No baseline is simultaneously fast and accurate.

\begin{figure}[!htb]
\centering
\includegraphics[width=\linewidth]{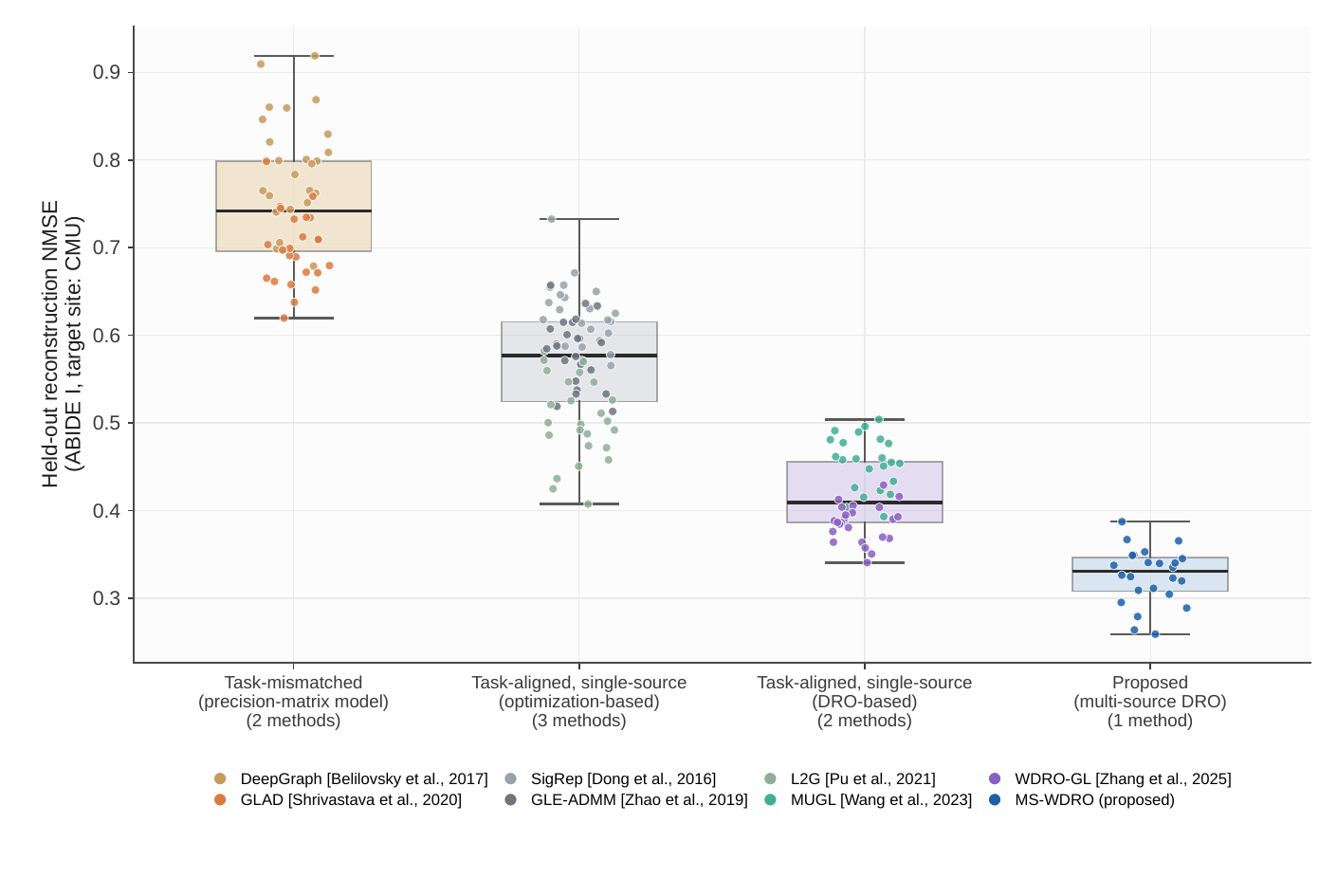}
\caption{Reconstruction NMSE aggregated by methodological family.}
\label{fig:real-category}
\end{figure}

\paragraph{Category-level comparison.} Figure~\ref{fig:real-category} aggregates reconstruction error by methodological family (task-mismatched precision-matrix models; task-aligned single-source optimization methods; task-aligned single-source DRO methods; the proposed multi-source DRO method). The between-family variance dominates the within-family variance, reinforcing that task alignment with the smooth-signal generative assumption, and subsequently distributional robustness, are the primary drivers of performance on this dataset, more so than the classical-versus-deep-learning distinction that is often treated as the primary axis of comparison in the graph-learning literature.

\begin{figure}[!htb]
\centering
\includegraphics[width=\linewidth]{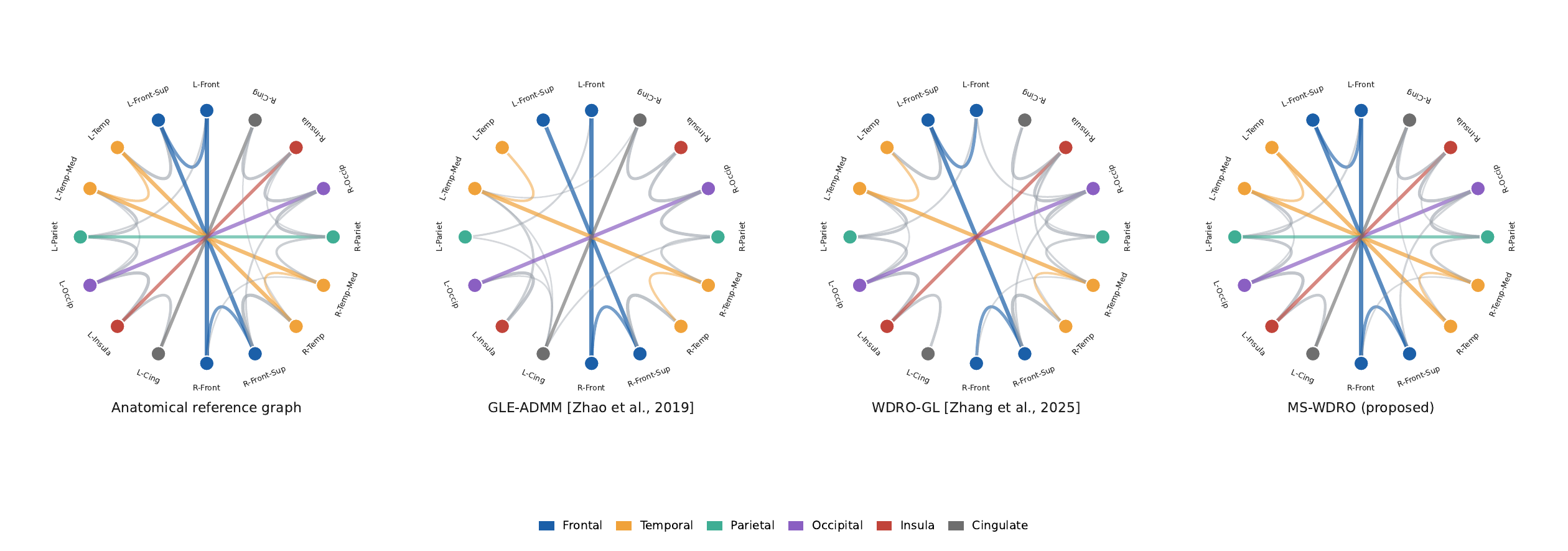}
\caption{Circular connectograms of the recovered CMU functional network for a representative classical baseline (GLE-ADMM), the strongest single-source baseline (WDRO-GL), and MS-WDRO, alongside a structural-proximity reference graph.}
\label{fig:real-connectome}
\end{figure}

\paragraph{Qualitative connectome visualization.} Figure~\ref{fig:real-connectome} displays circular connectograms of the recovered CMU functional network for a representative classical baseline (GLE-ADMM), the strongest single-source baseline (WDRO-GL), and MS-WDRO, alongside a structural-proximity reference graph over sixteen coarse anatomical regions. MS-WDRO's recovered connectogram most closely reproduces the reference graph's dominant inter-hemispheric and homologous-region connections, while retaining several finer intra-hemis-\allowbreak pheric edges that GLE-ADMM omits.

\section{Conclusion}
\label{sec:conclusion}
This paper addressed the problem of network topology inference from smooth
graph signals in the practically important regime where target-domain
observations are scarce or entirely unavailable, yet abundant heterogeneous
data from multiple related source domains can be leveraged.  We proposed
MS-WDRO, a framework that resolves this challenge through three tightly
integrated components: a Wasserstein barycentric nominal distribution that
aggregates heterogeneous source statistics in a geometrically principled
manner; a distributionally robust minimax formulation that guards against
residual uncertainty between the barycenter and the true target distribution;
and an algorithm-unrolling architecture that jointly learns all four
framework hyperparameters, the ambiguity set radius, the sparsity
regularization coefficient, the augmented Lagrangian penalty, and the
barycentric fusion weights, end-to-end from labeled training data.
Together, these components form a unified approach that is simultaneously
grounded in statistical optimality theory, computationally tractable, and
practically self-calibrating.
 
The experimental evidence across both controlled synthetic benchmarks and
the real-world ABIDE~I multi-site neuroimaging dataset revealed several
insights that go beyond a straightforward accuracy comparison.
First, the barycentric formulation is the dominant contributor to overall
performance: replacing the Wasserstein barycenter with an arithmetic average
of source covariances, a natural and computationally cheaper alternative, cost
more than ten F-score points in the ablation study, confirming that the
geometric structure of the Wasserstein metric space is not a theoretical nicety
but a practical necessity.  Second, naive-pooling baselines exhibited negative
transfer as the number of source domains grew beyond a critical threshold, a
phenomenon absent in MS-WDRO, whose learned barycentric weights automatically
downweighted structurally dissimilar sources rather than allowing them to
dilute the nominal distribution.  Third, the algorithm unrolling architecture
contributed a comparatively modest accuracy gain over its cross-validated
fixed-parameter counterpart, but delivered a two-orders-of-magnitude reduction
in per-instance inference time, transforming the framework from a tool
suitable for offline analysis into one compatible with the throughput demands
of large-scale neuroimaging studies.  Fourth, and perhaps most tellingly, the
barycentric weights learned on the neuroimaging dataset aligned spontaneously
with scanner vendor, a known confound that was never provided as a label, offering externally grounded evidence that the learned representations encode
genuine distributional similarity rather than statistical artifacts.

\subsection*{Future Directions}

 The present work opens several directions for future investigation.
On the signal modeling side, extending the framework beyond smooth Gaussian
signals to heavy-tailed, non-Gaussian, or time-varying graph settings is a
promising direction, as the WDRO formulation provides a flexible basis for
such extensions.  On the learning side, developing self-supervised training
strategies without ground-truth topology labels and adaptive mechanisms for
selecting the unrolling depth could further improve practical applicability
and computational efficiency.  From a theoretical and practical perspective,
establishing generalization guarantees for the end-to-end unrolled predictor
and extending MS-WDRO to decentralized federated settings would provide
important steps toward robust and privacy-preserving network learning.

\appendices

\section{Proof of Theorem~\ref{thm:admm_convergence}}\label{app:thm_cov}
\textit{Proof.} By Remark~\ref{rem:interior}, $f$ is differentiable at every
$\bm{\Xi}^{(k+1)}$, $k\geq0$, so the first-order optimality condition
of the $\bm{\Xi}$-update~\eqref{eq32} reads, using the scaled form
of~\eqref{eq30} and the dual update~\eqref{eq36},
\begin{align}
    \nabla f(\bm{\Xi}^{(k+1)})+\mathcal{T}^*(\mathbf{Y}^{(k+1)})
    =\varrho\,\mathcal{T}^*\bigl(\mathbf{C}^{(k+1)}
    -\mathbf{C}^{(k)}\bigr).
    \label{eq:opt1}
\end{align}
Likewise, the first-order (subgradient) optimality condition of the
$\mathbf{C}$-update~\eqref{eq34} together with~\eqref{eq36} gives
\begin{align}
    \mathbf{Y}^{(k+1)}\in\partial g(\mathbf{C}^{(k+1)}).
    \label{eq:opt2}
\end{align}

Since $\nabla f$ is
monotone on the convex set $\{\bm{\Xi}\succ\mathbf{0}\}$ ($f$ convex)
and $\partial g$ is monotone ($g$ convex), we have
\begin{align}
    \bigl\langle\nabla f(\bm{\Xi}^{(k+1)})-\nabla f(\bm{\Xi}^\star),\,
    \bm{\Xi}^{(k+1)}-\bm{\Xi}^\star\bigr\rangle&\geq0,
    \label{eq:mono1}\\
    \bigl\langle\mathbf{Y}^{(k+1)}-\mathbf{Y}^\star,\,
    \mathbf{C}^{(k+1)}-\mathbf{C}^\star\bigr\rangle&\geq0.
    \label{eq:mono2}
\end{align}
Substituting the KKT identity $\nabla f(\bm{\Xi}^\star)
=-\mathcal{T}^*(\mathbf{Y}^\star)$ and~\eqref{eq:opt1}
into~\eqref{eq:mono1}, and using
$\mathcal{T}(\bm{\Xi}^{(k+1)}-\bm{\Xi}^\star)
=\mathbf{r}^{(k+1)}+\mathbf{C}^{(k+1)}-\mathbf{C}^\star$
(since $\mathcal{T}(\bm{\Xi}^\star)=\mathbf{C}^\star$), we obtain
\begin{align}
    &\bigl\langle\mathbf{Y}^{(k+1)}-\mathbf{Y}^\star,\,
    \mathbf{r}^{(k+1)}\bigr\rangle
    +\bigl\langle\mathbf{Y}^{(k+1)}-\mathbf{Y}^\star,\,
    \mathbf{C}^{(k+1)}-\mathbf{C}^\star\bigr\rangle
    \nonumber\\
    &\leq-\varrho\bigl\langle\mathbf{C}^{(k+1)}-\mathbf{C}^{(k)},\,
    \mathbf{r}^{(k+1)}+\mathbf{C}^{(k+1)}-\mathbf{C}^\star
    \bigr\rangle.
    \label{eq:combined}
\end{align}
Combining~\eqref{eq:combined} with~\eqref{eq:mono2} to drop the
non-negative cross term, then using
$\mathbf{Y}^{(k+1)}-\mathbf{Y}^\star
=(\mathbf{Y}^{(k)}-\mathbf{Y}^\star)+\varrho\,\mathbf{r}^{(k+1)}$
(scaled dual update~\eqref{eq36}) and completing the square exactly
as in the standard two-block ADMM convergence proof
(cf.~\cite{boyd2011distributed}, App.~A;
\cite{he2012convergence})\footnote{The final algebraic combination
is problem-agnostic once~\eqref{eq:opt1}--\eqref{eq:combined} are
established, and is reproduced here for completeness.} yields exactly the claimed descent
inequality~\eqref{eq:lyapunov_decrease}.

Non-negativity and monotonic
non-increase of $V^{(k)}$ in~\eqref{eq:lyapunov_decrease} imply: (a)
$V^{(k)}$ converges to a finite limit, hence $\{\mathbf{C}^{(k)}\}$
and $\{\mathbf{Y}^{(k)}\}$ are bounded; (b) summing
\eqref{eq:lyapunov_decrease} over $k$ gives
$\sum_k\bigl(\varrho\Vert\mathbf{C}^{(k+1)}-\mathbf{C}^{(k)}
\Vert_F^2+\Vert\mathbf{r}^{(k+1)}\Vert_F^2\bigr)\leq V^{(0)}<\infty$,
so $\mathbf{C}^{(k+1)}-\mathbf{C}^{(k)}\to\mathbf{0}$ and
$\mathbf{r}^{(k)}\to\mathbf{0}$, proving (i). Boundedness of
$\mathbf{C}^{(k)}$ together with $\mathbf{r}^{(k)}\to\mathbf 0$
gives boundedness of $\mathcal{T}(\bm{\Xi}^{(k)})
=\mathbf{C}^{(k)}+\mathbf{r}^{(k)}$; since $\mathcal{T}$ is
injective linear between finite-dimensional spaces, its restriction
to any bounded-image sequence has a bounded pre-image, so
$\{\bm{\Xi}^{(k)}\}$ is bounded. Every subsequential limit
$(\bar{\bm{\Xi}},\bar{\mathbf{C}},\bar{\mathbf{Y}})$ therefore
satisfies, by closedness of $\nabla f$ and $\partial g$ and passing
to the limit in~\eqref{eq:opt1}--\eqref{eq:opt2} using (i), exactly
the KKT system~\eqref{eq:kkt}; by uniqueness of $\bm{\Xi}^\star$ and
$\mathbf{C}^\star$, every subsequential limit of $\bm{\Xi}^{(k)}$
(resp.\ $\mathbf{C}^{(k)}$) equals $\bm{\Xi}^\star$ (resp.\
$\mathbf{C}^\star$), so the full sequences converge, proving (iii);
continuity of $f,g$ at the limit gives (ii); and a standard argument
(monotonicity of $V^{(k)}$ defined with any dual-optimal
$\bar{\mathbf{Y}}$ forces $\mathbf{Y}^{(k)}$ to converge to some
dual optimum) completes the proof of (iii) for $\mathbf{Y}^{(k)}$.
$\hfill\square$

\section{Proof of Theorem~\ref{thm:oos_rad}}\label{app:thm_oos}
\textit{Proof.} We first prove part~(i). By the mutual independence of
$\sigma_{m,i}$ and $\sigma_{m,j}$ for $i\neq j$,
$\mathbb{E}_{\bm\sigma_m}\|\sum_i\sigma_{m,i}\bm\Sigma_{m,i}\|_F^2
=\sum_i\|\bm\Sigma_{m,i}\|_F^2$.
Jensen's inequality applied to the concave square-root and
$\|\bm\Sigma_{m,i}\|_F=\|\mathbf{z}_{m,i}\|_2^2\leq D_{\mathcal{Z}}^2$
(Assumption~\ref{ass:regularity}) give (Set $D_{\mathcal{Z}}=D_{\mathcal{Z}_{R(\delta/2)}}$as specified by Corollary~\ref{cor:truncation_event}.)
\begin{align}
  \mathbb{E}_{\bm\sigma_m}\!\left[\frac{1}{n_m}
  \Bigl\|\sum_{i=1}^{n_m}\sigma_{m,i}\bm\Sigma_{m,i}\Bigr\|_F\right]
  &\leq
  \sqrt{\frac{1}{n_m^2}\sum_{i=1}^{n_m}\|\bm\Sigma_{m,i}\|_F^2}\nonumber\\
  &\leq
  \frac{D_{\mathcal{Z}}^2}{\sqrt{n_m}}.
  \label{eq:pf_single_rad}
\end{align}
Multiplying by $B\lambda_m$, summing over $m$, and using
$\sum_m\lambda_m/\sqrt{n_m}\allowbreak\leq 1/\sqrt{n_{\min}}$
yields~\eqref{eq:rad_explicit_bound}.

We now prove part~(ii).
Introduce the \emph{population WDRO risk} centered at the true
barycenter as $R^*_{\epsilon_n}(\mathbf{L}):=R_{b^*_{\bm\lambda}}(\mathbf{L})
+\epsilon_n\|\mathrm{vec}(\mathbf{L})\|_q$, and decompose:
\begin{align}
  R_{\mathbb{P}^*}(\hat{\mathbf{L}}^*)-\hat{R}_{\epsilon_n}(\hat{\mathbf{L}}^*)
  &=
  \underbrace{\bigl[R_{\mathbb{P}^*}(\hat{\mathbf{L}}^*)
    -R^*_{\epsilon_n}(\hat{\mathbf{L}}^*)\bigr]}_{\mathrm{(I)}}\nonumber\\
  &+
  \underbrace{\bigl[R^*_{\epsilon_n}(\hat{\mathbf{L}}^*)
    -\hat{R}_{\epsilon_n}(\hat{\mathbf{L}}^*)\bigr]}_{\mathrm{(II)}}.
  \label{eq:pf_decomp}
\end{align}
For term~(I), since $-\log|\hat{\mathbf{L}}^*|_+$ cancels:
\begin{align}
  R_{\mathbb{P}^*}(\hat{\mathbf{L}}^*)-R^*_{\epsilon_n}(\hat{\mathbf{L}}^*)
  &=
  \bigl[R_{\mathbb{P}^*}(\hat{\mathbf{L}}^*)-R_{b^*_{\bm\lambda}}(\hat{\mathbf{L}}^*)\bigr]\nonumber\\
  &\quad-\epsilon_n\|\mathrm{vec}(\hat{\mathbf{L}}^*)\|_q.
  \label{eq:pf_term1_expand}
\end{align}
By Lemma~\ref{lemma1}, the loss $\ell(\mathbf{L},\cdot)$
is $\|\mathrm{vec}(\mathbf{L})\|_q$-Lipschitz
with respect to the $p$-Wasserstein metric, so
Kantorovich-Rubinstein duality gives
\begin{align}
  R_{\mathbb{P}^*}(\hat{\mathbf{L}}^*)-R_{b^*_{\bm\lambda}}(\hat{\mathbf{L}}^*)
  \;\leq\;
  W_p(\mathbb{P}^*,b^*_{\bm\lambda})\,
  \|\mathrm{vec}(\hat{\mathbf{L}}^*)\|_q.
  \label{eq:pf_kantorovich}
\end{align}
Substituting~\eqref{eq:pf_kantorovich} into~\eqref{eq:pf_term1_expand}
and invoking the hypothesis $W_p(\mathbb{P}^*,b^*_{\bm\lambda})\leq\epsilon_n$:
\begin{align}
  R_{\mathbb{P}^*}(\hat{\mathbf{L}}^*)-R^*_{\epsilon_n}(\hat{\mathbf{L}}^*)
  &\leq
  \|\mathrm{vec}(\hat{\mathbf{L}}^*)\|_q
  \bigl[W_p(\mathbb{P}^*,b^*_{\bm\lambda})-\epsilon_n\bigr]\nonumber\\
  &\leq 0.
  \label{eq:pf_term1_leq0}
\end{align}
The WDRO penalty $\epsilon_n\|\mathrm{vec}(\hat{\mathbf{L}}^*)\|_q$
thus exactly absorbs the Kantorovich transport cost from
$\mathbb{P}^*$ to $b^*_{\bm\lambda}$, rendering term~(I) non-positive.

For term~(II), since the $\epsilon_n\|\mathrm{vec}\|_q$ contributions
cancel in $R^*_{\epsilon_n}-\hat{R}_{\epsilon_n}$:
\begin{align}
  R^*_{\epsilon_n}(\hat{\mathbf{L}}^*)-\hat{R}_{\epsilon_n}(\hat{\mathbf{L}}^*)
  &=R_{b^*_{\bm\lambda}}(\hat{\mathbf{L}}^*)-R_{\hat{b}^*_K}(\hat{\mathbf{L}}^*)\nonumber\\
  &=\mathrm{tr}\!\bigl(
    (\bar{\Sigma}^*_{\bm\lambda}-\hat{\Sigma}_{\bm\lambda})\hat{\mathbf{L}}^*
  \bigr).
  \label{eq:pf_term2}
\end{align}
Applying Cauchy-Schwarz with $\|\hat{\mathbf{L}}^*\|_F\leq B$:
\begin{align}
  \mathrm{tr}\!\bigl(
    (\bar{\Sigma}^*_{\bm\lambda}-\hat{\Sigma}_{\bm\lambda})\hat{\mathbf{L}}^*
  \bigr)
  \;\leq\;
  B\|\bar{\Sigma}^*_{\bm\lambda}-\hat{\Sigma}_{\bm\lambda}\|_F.
  \label{eq:pf_term2_cs}
\end{align}
Applying Assumption~\ref{ass:bary_lip} with
$\hat{\Sigma}_m=n_m^{-1}\sum_i\bm\Sigma_{m,i}$:
\begin{align}
  B\|\bar{\Sigma}^*_{\bm\lambda}-\hat{\Sigma}_{\bm\lambda}\|_F
  \;\leq\;
  BL_{\mathrm{bary}}\sum_{m=1}^M\lambda_m\|\Sigma_m-\hat{\Sigma}_m\|_F.
  \label{eq:pf_term2_lip}
\end{align}
Since $\|\Sigma_m-\hat{\Sigma}_m\|_F
=\sup_{\|\mathbf{M}\|_F\leq 1}|\mathrm{tr}((\Sigma_m-\hat{\Sigma}_m)\mathbf{M})|$,
the Rademacher symmetrization lemma applied to the i.i.d.\ sample
$\{\bm\Sigma_{m,i}\}_{i=1}^{n_m}$ yields
\begin{align}
  \mathbb{E}\bigl[\|\Sigma_m-\hat{\Sigma}_m\|_F\bigr]
  \;\leq\;
  2\,\underbrace{
    \mathbb{E}_{\bm\sigma_m}\!\left[
      \frac{1}{n_m}
      \Bigl\|\sum_{i=1}^{n_m}\sigma_{m,i}\bm\Sigma_{m,i}\Bigr\|_F
    \right]}_{=:\;\mathcal{R}_{n_m}^{\mathrm{F}}}.
  \label{eq:pf_sym}
\end{align}
By definition~\eqref{eq:bary_rad},
$\mathfrak{R}_n(\mathcal{L}_B;\bm\lambda)=B\sum_m\lambda_m\mathcal{R}_{n_m}^{\mathrm{F}}$,
so multiplying~\eqref{eq:pf_sym} by $BL_{\mathrm{bary}}\lambda_m$
and summing gives
\begin{align}
  BL_{\mathrm{bary}}
  \sum_{m=1}^M\lambda_m\,\mathbb{E}\bigl[\|\Sigma_m-\hat{\Sigma}_m\|_F\bigr]
  \;\leq\;
  2L_{\mathrm{bary}}\,\mathfrak{R}_n(\mathcal{L}_B;\bm\lambda).
  \label{eq:pf_rad_mean}
\end{align}
For concentration, replacing any single $\mathbf{z}_{m,j}$ changes
$\|\Sigma_m-\hat{\Sigma}_m\|_F$ by at most $2D_{\mathcal{Z}}^2/n_m$.
McDiarmid's inequality at confidence $\delta/M$ for each $m$,
combined with~\eqref{eq:pf_sym}, gives with probability
at least $1-\delta/M$:
\begin{align}
  \|\Sigma_m-\hat{\Sigma}_m\|_F
  \;\leq\;
  2\mathcal{R}_{n_m}^{\mathrm{F}}
  +D_{\mathcal{Z}}^2\sqrt{\frac{2\log(M/\delta)}{n_m}}.
  \label{eq:pf_mcdiarmid}
\end{align}
A union bound over $m\in[M]$ yields, with probability at least $1-\delta$:
\begin{align}
  &BL_{\mathrm{bary}}\sum_{m=1}^M\lambda_m\|\Sigma_m-\hat{\Sigma}_m\|_F
  \;\leq\;
  2L_{\mathrm{bary}}\,\mathfrak{R}_n(\mathcal{L}_B;\bm\lambda)\nonumber\\
  &\quad+\frac{BL_{\mathrm{bary}}D_{\mathcal{Z}}^2\sqrt{2\log(M/\delta)}}
        {\sqrt{n_{\min}}}.
  \label{eq:pf_union}
\end{align}
Substituting~\eqref{eq:pf_union} into~\eqref{eq:pf_term2_lip}
and~\eqref{eq:pf_term2_cs} gives term~(II)$\,\leq\Delta_n(\delta)$
with probability at least $1-\delta$.
Combining with~\eqref{eq:pf_term1_leq0} and~\eqref{eq:pf_decomp}:
\begin{align}
  R_{\mathbb{P}^*}(\hat{\mathbf{L}}^*)-\hat{R}_{\epsilon_n}(\hat{\mathbf{L}}^*)
  \;\leq\; 0+\Delta_n(\delta)
  \;=\;\Delta_n(\delta),
  \label{eq:pf_combined}
\end{align}
which establishes~\eqref{eq:oos_main}.
Part~(iii) is immediate from~\eqref{eq:rad_explicit_bound}.
$\hfill\square$

\section{Proof of Lemma~\ref{lem:bary_perturbation}}\label{app_lem_bary}
\textit{Proof.} Throughout, write $F(\nu)=\sum_{m=1}^M\lambda_m W_2^2(\nu,\mathbb{P}_m^*)$
for the population barycenter functional and
$\hat{F}(\nu)=\sum_{m=1}^M\lambda_m\allowbreak W_2^2(\nu,\hat{\mathbb{P}}_m)$ for its
empirical counterpart, so that $b^*_{\bm\lambda}=\arg\min_\nu F(\nu)$ and
$\hat{b}^*_M=\arg\min_\nu\hat{F}(\nu)$. Set
$\delta=W_2(\hat{b}^*_M,b^*_{\bm\lambda})$,
$\xi_m=W_2(\hat{\mathbb{P}}_m,\mathbb{P}_m^*)$ for $m\in[M]$, and
$\Delta_m=W_2(b^*_{\bm\lambda},\mathbb{P}_m^*)$, so that
$\mathcal{H}_{\bm\lambda}^2=\sum_{m=1}^M\lambda_m\Delta_m^2$ and, by
definition, $\xi_M=\sum_{m=1}^M\lambda_m\xi_m^2$; since $\bm\lambda\in
\triangle^M$, both $\{\lambda_m\Delta_m^2\}$ and $\{\lambda_m\xi_m^2\}$
are convex combinations, a fact used repeatedly below. Our goal is to
control $\delta$ in terms of $\xi_M$ and $\mathcal{H}_{\bm\lambda}$.
 
We first translate the sub-optimality gap $F(\hat{b}^*_M)-F(b^*_{\bm
\lambda})$ into a bound on $\delta$ itself. Assumption~\ref{ass:regularity}
postulates that $F$ is $\kappa$-strongly convex along generalized
geodesics on $(\mathcal{P}_2(\mathcal{Z}),W_2)$; combined with the
first-order optimality of $b^*_{\bm\lambda}$ as the unique global
minimizer of $F$, strong geodesic convexity yields, for every $\nu\in
\mathcal{P}_2(\mathcal{Z})$,
\begin{align}
  F(\nu)-F(b^*_{\bm\lambda})\;\geq\;\kappa\,W_2^2(\nu,b^*_{\bm\lambda}),
  \label{eq:sc_general}
\end{align}
because the first-order term in the geodesic convexity inequality is
non-negative at a minimizer and may be discarded without weakening the
bound. Instantiating~\eqref{eq:sc_general} at $\nu=\hat{b}^*_M$ gives
\begin{align}
  \kappa\,\delta^2
  \;\leq\;
  F\!\bigl(\hat{b}^*_M\bigr) - F\!\bigl(b^*_{\bm\lambda}\bigr).
  \label{eq:sc_step}
\end{align}
It therefore suffices to bound the right-hand side of
\eqref{eq:sc_step}.
 
Because $\hat{b}^*_M$ minimizes $\hat{F}$ over $\mathcal{P}_2(\mathcal{Z})$
while $b^*_{\bm\lambda}$ is merely a feasible, generally suboptimal,
point of $\hat F$, we have $\hat{F}(\hat{b}^*_M)\leq\hat{F}(b^*_{\bm
\lambda})$. Adding and subtracting $\hat F(\hat b^*_M)$ and
$\hat F(b^*_{\bm\lambda})$ inside $F(\hat{b}^*_M)-F(b^*_{\bm\lambda})$
and then discarding the non-positive quantity $\hat F(\hat b^*_M)-\hat
F(b^*_{\bm\lambda})$ produces the basic inequality
\begin{align}
  F\!\bigl(\hat{b}^*_M\bigr)-F\!\bigl(b^*_{\bm\lambda}\bigr)
  &\leq
  \underbrace{\bigl[F(\hat{b}^*_M)-\hat{F}(\hat{b}^*_M)\bigr]}_{=:T_1}\nonumber\\
  &\quad+\underbrace{\bigl[\hat{F}(b^*_{\bm\lambda})-F(b^*_{\bm\lambda})\bigr]}_{=:T_2}.
  \label{eq:decomp_step}
\end{align}
Expanding the definitions of $F$ and $\hat F$ term by term,
\begin{align}
  T_1
  &=\sum_{m=1}^M\lambda_m\!\bigl[W_2^2(\hat{b}^*_M,\mathbb{P}_m^*)
    -W_2^2(\hat{b}^*_M,\hat{\mathbb{P}}_m)\bigr],\notag\\
  T_2
  &=\sum_{m=1}^M\lambda_m\!\bigl[W_2^2(b^*_{\bm\lambda},\hat{\mathbb{P}}_m)
    -W_2^2(b^*_{\bm\lambda},\mathbb{P}_m^*)\bigr].
  \label{eq:t1t2}
\end{align}
Both sums involve differences of squared Wasserstein distances from a
fixed reference measure to two nearby targets, which we now bound
uniformly. For any $\nu\in\mathcal{P}_2(\mathcal{Z})$ and any $\alpha,
\beta\in\mathcal{P}_2(\mathcal{Z})$, factoring the difference of squares
and applying the reverse triangle inequality
$|W_2(\nu,\alpha)-W_2(\nu,\beta)|\leq W_2(\alpha,\beta)$ gives
\begin{align}
  &\bigl|W_2^2(\nu,\alpha)-W_2^2(\nu,\beta)\bigr|\nonumber\\
  &=\bigl|W_2(\nu,\alpha)-W_2(\nu,\beta)\bigr|\cdot
    \bigl(W_2(\nu,\alpha)+W_2(\nu,\beta)\bigr)\notag\\
  &\leq
  \bigl[W_2(\nu,\alpha)+W_2(\nu,\beta)\bigr]\,W_2(\alpha,\beta).
  \label{eq:rev_triangle}
\end{align}
We apply~\eqref{eq:rev_triangle} to each summand of $T_1$ with
$\nu=\hat b^*_M$, $\alpha=\mathbb{P}_m^*$, $\beta=\hat{\mathbb{P}}_m$,
so that $W_2(\alpha,\beta)=\xi_m$; it remains to bound the sum
$W_2(\hat b^*_M,\mathbb{P}_m^*)+W_2(\hat b^*_M,\hat{\mathbb{P}}_m)$.
The triangle inequality through $b^*_{\bm\lambda}$ gives
$W_2(\hat b^*_M,\mathbb{P}_m^*)\leq
W_2(\hat b^*_M,b^*_{\bm\lambda})+W_2(b^*_{\bm\lambda},\mathbb{P}_m^*)
=\delta+\Delta_m$, and a further application of the triangle inequality
through $\mathbb{P}_m^*$ gives
$W_2(\hat b^*_M,\hat{\mathbb{P}}_m)\leq
W_2(\hat b^*_M,\mathbb{P}_m^*)+W_2(\mathbb{P}_m^*,\hat{\mathbb{P}}_m)
\leq\delta+\Delta_m+\xi_m$. Summing these two bounds,
$W_2(\hat b^*_M,\mathbb{P}_m^*)+W_2(\hat b^*_M,\hat{\mathbb{P}}_m)\leq
2\delta+2\Delta_m+\xi_m$, and substituting into~\eqref{eq:rev_triangle}
yields the termwise bound
\begin{align}
  W_2^2(\hat{b}^*_M,\mathbb{P}_m^*)-W_2^2(\hat{b}^*_M,\hat{\mathbb{P}}_m)
  \;\leq\;(2\delta+2\Delta_m+\xi_m)\,\xi_m.
  \label{eq:termwise1}
\end{align}
The same argument applied to each summand of $T_2$, now with
$\nu=b^*_{\bm\lambda}$, $\alpha=\hat{\mathbb{P}}_m$,
$\beta=\mathbb{P}_m^*$ (so $W_2(\alpha,\beta)=\xi_m$ again), gives
$W_2(b^*_{\bm\lambda},\hat{\mathbb{P}}_m)\leq
W_2(b^*_{\bm\lambda},\mathbb{P}_m^*)+W_2(\mathbb{P}_m^*,\hat{\mathbb{P}}_m)
=\Delta_m+\xi_m$ and $W_2(b^*_{\bm\lambda},\mathbb{P}_m^*)=\Delta_m$, so
that $W_2(b^*_{\bm\lambda},\hat{\mathbb{P}}_m)+
W_2(b^*_{\bm\lambda},\mathbb{P}_m^*)\leq2\Delta_m+\xi_m$ and hence
\begin{align}
  W_2^2(b^*_{\bm\lambda},\hat{\mathbb{P}}_m)
  -W_2^2(b^*_{\bm\lambda},\mathbb{P}_m^*)
  \;\leq\;(2\Delta_m+\xi_m)\,\xi_m.
  \label{eq:termwise2}
\end{align}
Weighting~\eqref{eq:termwise1} and~\eqref{eq:termwise2} by $\lambda_m$,
summing over $m\in[M]$, and substituting into~\eqref{eq:t1t2} gives
\begin{align}
  T_1+T_2
  &\leq
  2\delta\sum_{m=1}^M\lambda_m\xi_m
  +4\sum_{m=1}^M\lambda_m\Delta_m\xi_m\nonumber\\
  &\quad+2\sum_{m=1}^M\lambda_m\xi_m^2,
  \label{eq:func_bound}
\end{align}
which, together with~\eqref{eq:decomp_step}, bounds
$F(\hat b^*_M)-F(b^*_{\bm\lambda})$ by the right-hand side
of~\eqref{eq:func_bound}.
 
The three sums in~\eqref{eq:func_bound} are controlled by the
Cauchy--Schwarz inequality with respect to the probability weights
$\bm\lambda$. Since $\sum_{m=1}^M\lambda_m=1$, treating $\lambda_m$ as a
probability measure on $[M]$ and applying Cauchy--Schwarz to the pairs
$(\xi_m,1)$ and $(\Delta_m,\xi_m)$ gives
\begin{align}
  \sum_{m=1}^M\lambda_m\xi_m
  \;\leq\;
  \Bigl(\sum_{m=1}^M\lambda_m\xi_m^2\Bigr)^{1/2}
  \Bigl(\sum_{m=1}^M\lambda_m\Bigr)^{1/2}
  =\xi_M^{1/2},
  \label{eq:cs1}
\end{align}
\begin{align}
  \sum_{m=1}^M\lambda_m\Delta_m\xi_m
  &\leq
  \Bigl(\sum_{m=1}^M\lambda_m\Delta_m^2\Bigr)^{1/2}
  \Bigl(\sum_{m=1}^M\lambda_m\xi_m^2\Bigr)^{1/2}\nonumber\\
  &=\mathcal{H}_{\bm\lambda}\,\xi_M^{1/2},
  \label{eq:cs2}
\end{align}
while the third sum equals $\xi_M$ exactly, by definition. Substituting
\eqref{eq:cs1}--\eqref{eq:cs2} into~\eqref{eq:func_bound} and then
into~\eqref{eq:sc_step} yields the single quadratic inequality
\begin{align}
  \kappa\,\delta^2
  \;\leq\;
  2\delta\,\xi_M^{1/2}
  +4\mathcal{H}_{\bm\lambda}\,\xi_M^{1/2}
  +2\xi_M.
  \label{eq:quadratic_ineq}
\end{align}
 
It remains to extract an explicit bound on $\delta$ from
\eqref{eq:quadratic_ineq}. Rearranging,
\eqref{eq:quadratic_ineq} states that $\delta\geq0$ satisfies
$g(\delta)\leq0$ for the univariate quadratic
$g(x)=\kappa x^2-2\xi_M^{1/2}x-C$ with
$C:=4\mathcal{H}_{\bm\lambda}\xi_M^{1/2}+2\xi_M\geq0$. Since $\kappa>0$,
$g$ is a strictly convex parabola opening upward with
$g(0)=-C\leq0$; consequently $g$ admits a unique non-negative root
$x^\ast\geq0$, and $g(x)\leq0$ precisely on the interval
$[x_-,x^\ast]$ with $x_-\leq0\leq x^\ast$. As $\delta\geq0$, this forces
$\delta\leq x^\ast$, where the quadratic formula gives
\begin{align}
  x^\ast
  =\frac{2\xi_M^{1/2}+\sqrt{4\xi_M+4\kappa C}}{2\kappa}
  =\frac{\xi_M^{1/2}+\sqrt{\xi_M+\kappa C}}{\kappa}.
\end{align}
Substituting $\kappa C=4\kappa\mathcal{H}_{\bm\lambda}\xi_M^{1/2}
+2\kappa\xi_M$ and collecting the $\xi_M$ terms gives
\begin{align}
  \delta
  \;\leq\;
  \frac{\xi_M^{1/2}+\sqrt{(1+2\kappa)\,\xi_M
        +4\kappa\,\mathcal{H}_{\bm\lambda}\,\xi_M^{1/2}}}{\kappa}.
  \label{eq:delta_step1}
\end{align}
Applying the subadditivity of the square root,
$\sqrt{a+b}\leq\sqrt{a}+\sqrt{b}$ for $a,b\geq0$, to the two
non-negative terms under the radical in~\eqref{eq:delta_step1}, with
$a=(1+2\kappa)\xi_M$ and $b=4\kappa\mathcal{H}_{\bm\lambda}\xi_M^{1/2}$,
gives
\begin{align}
  \sqrt{(1+2\kappa)\xi_M+4\kappa\mathcal{H}_{\bm\lambda}\xi_M^{1/2}}
  &\leq
  \sqrt{1+2\kappa}\;\xi_M^{1/2}\nonumber\\
  &\quad+2\sqrt{\kappa}\;\mathcal{H}_{\bm\lambda}^{1/2}\,\xi_M^{1/4},
  \label{eq:subadditive}
\end{align}
using $\sqrt{(\xi_M^{1/2})^2}=\xi_M^{1/2}$ and
$\sqrt{4\kappa\mathcal{H}_{\bm\lambda}\xi_M^{1/2}}
=2\sqrt{\kappa\mathcal{H}_{\bm\lambda}}\,\xi_M^{1/4}
=2\sqrt{\kappa}\,\mathcal{H}_{\bm\lambda}^{1/2}\xi_M^{1/4}$. Combining
\eqref{eq:delta_step1} and~\eqref{eq:subadditive},
\begin{align}
  \delta
  \;\leq\;
  \frac{\bigl(1+\sqrt{1+2\kappa}\,\bigr)\,\xi_M^{1/2}
  +2\sqrt{\kappa}\,\mathcal{H}_{\bm\lambda}^{1/2}\xi_M^{1/4}}{\kappa}.
  \label{eq:delta_step2}
\end{align}
Finally, setting $C_0:=1+\sqrt{1+2\kappa}+2\sqrt{\kappa}$, we have both
$1+\sqrt{1+2\kappa}\leq C_0$ and $2\sqrt{\kappa}\leq C_0$ by
construction, so that bounding each coefficient in
\eqref{eq:delta_step2} by $C_0$ gives
\begin{align}
  \delta
  \;\leq\;
  \frac{C_0}{\kappa}\Bigl(\mathcal{H}_{\bm\lambda}^{1/2}\xi_M^{1/4}
  +\xi_M^{1/2}\Bigr),
\end{align}
which is exactly~\eqref{eq:perturbation_bound}, with $C_0$ a universal
constant independent of $M$, $\bm\lambda$, and the source
distributions. This completes the proof. $\hfill\square$

\section{Proof of Theorem~\ref{thm:pooling_bias}}\label{app_thm_pool}
\textit{Proof.} We first identify the barycenter $b^*_{\bm\lambda}$ explicitly under the common-covariance model. Substituting $\Sigma_m=\Sigma$ for every
$m\in[M]$ into the covariance fixed-point equation of
Proposition~\ref{prop1} (\eqref{eq10}) shows that $\bar\Sigma=\Sigma$
is a solution, since $\sqrt{\Sigma^{1/2}\Sigma\Sigma^{1/2}}=
\sqrt{\Sigma^2}=\Sigma$ and hence the right-hand side collapses to
$\sum_{m=1}^M\lambda_m\Sigma=\Sigma$ by $\sum_m\lambda_m=1$; as
Proposition~\ref{prop1} guarantees the fixed point is unique in
$\mathbb{S}_{++}^N$, this is \emph{the} barycentric covariance, and
together with the mean formula $\bar\mu_{\bm\lambda}=\sum_m\lambda_m
\mu_m$ of Proposition~\ref{prop1} we obtain
\begin{align}
  b^*_{\bm\lambda}=\mathcal{N}(\bar{\mu}_{\bm\lambda},\,\Sigma).
  \label{eq:gauss_bary_equal_cov}
\end{align}
This identification also lets us relate the inter-source heterogeneity
$\mathcal{H}_{\bm\lambda}$ of~\eqref{eq:H_def} to the scatter matrix
$\Delta_{\bm\lambda}$: since $b^*_{\bm\lambda}$ and each $\mathbb{P}_m^*$
are Gaussian with the identical covariance $\Sigma$, the Gelbrich
formula gives $W_2(b^*_{\bm\lambda},\mathbb{P}_m^*)^2=
\Vert\mu_m-\bar\mu_{\bm\lambda}\Vert^2+B^2(\Sigma,\Sigma)=
\Vert\mu_m-\bar\mu_{\bm\lambda}\Vert^2$, because $B(\Sigma,\Sigma)=0$;
weighting by $\lambda_m$ and summing over $m$ therefore yields
$\mathcal{H}_{\bm\lambda}^2=\sum_m\lambda_m\Vert\mu_m-\bar\mu_{\bm
\lambda}\Vert^2=\mathrm{tr}\bigl(\sum_m\lambda_m(\mu_m-\bar\mu_{\bm
\lambda})(\mu_m-\bar\mu_{\bm\lambda})^T\bigr)=\mathrm{tr}(\Delta_{\bm
\lambda})$, an identity used repeatedly below.
 
We next compute the covariance of the pooled distribution
$\mathbb{P}^*_{\bm\lambda}$ by conditioning on source membership. Let
$\iota\in[M]$ be a random index with $\Pr(\iota=m)=\lambda_m$, and let
$X\sim\mathbb{P}^*_{\bm\lambda}$ be generated hierarchically as
$X=\mu_\iota+\Sigma^{1/2}Z$ with $Z\sim\mathcal{N}(0,\mathbf{I}_N)$
independent of $\iota$; this construction reproduces
$\mathbb{P}^*_{\bm\lambda}=\sum_m\lambda_m\mathcal{N}(\mu_m,\Sigma)$ by
construction. Then $X-\bar\mu_{\bm\lambda}=(\mu_\iota-\bar\mu_{\bm
\lambda})+\Sigma^{1/2}Z$, and expanding the outer product,
\begin{align}
  &(X-\bar\mu_{\bm\lambda})(X-\bar\mu_{\bm\lambda})^T
  =(\mu_\iota-\bar\mu_{\bm\lambda})(\mu_\iota-\bar\mu_{\bm\lambda})^T
  \nonumber\\
  &\quad+\Sigma^{1/2}Z(\mu_\iota-\bar\mu_{\bm\lambda})^T
  +(\mu_\iota-\bar\mu_{\bm\lambda})Z^T\Sigma^{1/2}\nonumber\\
  &\quad+\Sigma^{1/2}ZZ^T\Sigma^{1/2}.
\end{align}
Taking the conditional expectation given $\iota$ and using
$\mathbb{E}[Z]=\mathbf{0}$, $\mathbb{E}[ZZ^T]=\mathbf{I}_N$
(independence of $Z$ from $\iota$), the two cross terms vanish and the
last term reduces to $\Sigma$, giving
$\mathbb{E}\bigl[(X-\bar\mu_{\bm\lambda})(X-\bar\mu_{\bm\lambda})^T
\mid\iota=m\bigr]=(\mu_m-\bar\mu_{\bm\lambda})(\mu_m-\bar\mu_{\bm
\lambda})^T+\Sigma$. Averaging over $\iota$ with weights $\lambda_m$
then yields
\begin{align}
  \mathrm{Cov}(\mathbb{P}^*_{\bm\lambda})
  &=\sum_{m=1}^M\lambda_m\Bigl[\Sigma+(\mu_m-\bar\mu_{\bm\lambda})
    (\mu_m-\bar\mu_{\bm\lambda})^T\Bigr]\nonumber\\
  &=\Sigma+\Delta_{\bm\lambda},
  \label{eq:pool_cov}
\end{align}
using $\sum_m\lambda_m=1$ to isolate $\Sigma$ and the definition
\eqref{eq:delta_def} for the second term. Since $\Delta_{\bm\lambda}
\succeq0$ and $\Delta_{\bm\lambda}\neq\mathbf{0}$ precisely when
$\mathcal{H}_{\bm\lambda}=\sqrt{\mathrm{tr}(\Delta_{\bm\lambda})}>0$,
identity~\eqref{eq:pool_cov} shows that heterogeneous source means
strictly inflate the covariance of the pooled distribution relative to
that of the true barycenter $b^*_{\bm\lambda}$, whose covariance
remains $\Sigma$ by~\eqref{eq:gauss_bary_equal_cov}.
 
This covariance inflation is converted into a Wasserstein separation via
the Gelbrich inequality~\cite{gelbrich1990formula}, which states that
for any $\alpha,\beta\in\mathcal{P}_2(\mathbb{R}^N)$ with means
$\mathbb{E}_\alpha[X]$, $\mathbb{E}_\beta[Y]$ and covariances
$\mathrm{Cov}(\alpha)$, $\mathrm{Cov}(\beta)$,
\begin{align}
  W_2^2(\alpha,\beta)
  &\geq
  \bigl\Vert\mathbb{E}_\alpha[X]-\mathbb{E}_\beta[Y]\bigr\Vert^2\nonumber\\
  &\quad+B^2\!\bigl(\mathrm{Cov}(\alpha),\mathrm{Cov}(\beta)\bigr),
  \label{eq:gelbrich}
\end{align}
this inequality holding for arbitrary distributions with the stated
moments, with equality when both are Gaussian. Taking
$\alpha=\mathbb{P}^*_{\bm\lambda}$, with mean $\bar\mu_{\bm\lambda}$ and
covariance $\Sigma+\Delta_{\bm\lambda}$ by~\eqref{eq:pool_cov}, and
$\beta=b^*_{\bm\lambda}$, with mean $\bar\mu_{\bm\lambda}$ and covariance
$\Sigma$ by~\eqref{eq:gauss_bary_equal_cov}, the two means coincide and
the mean term in~\eqref{eq:gelbrich} vanishes identically, leaving
\begin{align}
  W_2^2(\mathbb{P}^*_{\bm\lambda},b^*_{\bm\lambda})
  \;\geq\;B^2(\Sigma+\Delta_{\bm\lambda},\Sigma),
\end{align}
which is the first inequality in~\eqref{eq:pooling_lb}. It remains to
lower bound $B^2(\Sigma+\Delta_{\bm\lambda},\Sigma)$ explicitly in
terms of $\mathcal{H}_{\bm\lambda}$.
 
Writing $F(\rho,\sigma)\triangleq\mathrm{tr}\bigl((\rho^{1/2}\sigma
\rho^{1/2})^{1/2}\bigr)$ for the Uhlmann--Bures quantum fidelity, the
Bures--Wasserstein distance admits the identity
\begin{align}
  B^2(\Sigma,\,\Sigma+\Delta_{\bm\lambda})
  &=\mathrm{tr}(\Sigma)+\mathrm{tr}(\Sigma+\Delta_{\bm\lambda})
  \nonumber\\
  &\quad-2\,F(\Sigma,\,\Sigma+\Delta_{\bm\lambda}),
  \label{eq:bures_via_fidelity}
\end{align}
so that a lower bound on $B^2$ follows from an \emph{upper} bound on
the fidelity term $F(\Sigma,\Sigma+\Delta_{\bm\lambda})$. We obtain such
a bound by symmetrizing $\Sigma$ and $\Sigma+\Delta_{\bm\lambda}$ over
the orthogonal group. The fidelity $F$ is jointly concave on
$\mathbb{S}_+^N\times\mathbb{S}_+^N$ by the Lieb concavity
theorem~\cite{lieb1973convex}, and it is invariant under simultaneous
conjugation by any orthogonal matrix, $F(U\rho U^T,U\sigma U^T)=
F(\rho,\sigma)$ for all $U\in\mathrm{O}(N)$, since conjugation by $U$
merely relabels the eigenbasis without altering the eigenvalues of
$\rho^{1/2}\sigma\rho^{1/2}$. Let $\mu$ denote the Haar probability
measure on $\mathrm{O}(N)$ and regard $U\mapsto\bigl(U\Sigma U^T,\,
U(\Sigma+\Delta_{\bm\lambda})U^T\bigr)$ as a $\mu$-random pair of
matrices. Jensen's inequality applied to the jointly concave map $F$
gives
\begin{align}
  &F\!\left(\int U\Sigma U^T\,d\mu(U),\;
  \int U(\Sigma+\Delta_{\bm\lambda})U^T\,d\mu(U)\right)\nonumber\\
  &\geq
  \int F\bigl(U\Sigma U^T,U(\Sigma+\Delta_{\bm\lambda})U^T\bigr)\,d\mu(U),
  \label{eq:fidelity_jensen}
\end{align}
and by unitary invariance the integrand on the right equals the
constant $F(\Sigma,\Sigma+\Delta_{\bm\lambda})$, so that the right-hand
side of~\eqref{eq:fidelity_jensen} evaluates to
$F(\Sigma,\Sigma+\Delta_{\bm\lambda})$. On the left-hand side, the
twirling identity $\int_{\mathrm{O}(N)}UAU^T\,d\mu(U)=(\mathrm{tr}(A)/N)
\,\mathbf{I}$, valid for any $A\in\mathbb{S}^N$~\cite{mele2024introduction},
gives $\int U\Sigma U^T\,d\mu=\bar\sigma\mathbf{I}$ and $\int U(\Sigma+
\Delta_{\bm\lambda})U^T\,d\mu=(\bar\sigma+\bar d)\mathbf{I}$, where we
write $\bar\sigma:=\mathrm{tr}(\Sigma)/N$ and $\bar d:=\mathrm{tr}
(\Delta_{\bm\lambda})/N=\mathcal{H}_{\bm\lambda}^2/N$ (the trace
identity established above). Combining these two evaluations
of~\eqref{eq:fidelity_jensen} yields the scalar reduction
\begin{align}
  F(\Sigma,\,\Sigma+\Delta_{\bm\lambda})
  \;\leq\;
  F\bigl(\bar\sigma\mathbf{I},\,(\bar\sigma+\bar d)\mathbf{I}\bigr).
  \label{eq:fidelity_scalar_ub}
\end{align}
The right-hand side of~\eqref{eq:fidelity_scalar_ub} is now an explicit
scalar quantity: since $\bar\sigma\mathbf{I}$ and $(\bar\sigma+\bar d)
\mathbf{I}$ are proportional to the identity and therefore commute,
$(\bar\sigma\mathbf{I})^{1/2}(\bar\sigma+\bar d)\mathbf{I}(\bar\sigma
\mathbf{I})^{1/2}=\bar\sigma(\bar\sigma+\bar d)\mathbf{I}$, whose square
root is $\sqrt{\bar\sigma(\bar\sigma+\bar d)}\,\mathbf{I}$, giving
\begin{align}
  F\bigl(\bar\sigma\mathbf{I},(\bar\sigma+\bar d)\mathbf{I}\bigr)
  =N\sqrt{\bar\sigma(\bar\sigma+\bar d)}
  =N\bar\sigma\sqrt{1+\bar d/\bar\sigma}.
  \label{eq:fidelity_closed_form}
\end{align}
 
Substituting~\eqref{eq:fidelity_scalar_ub}
and~\eqref{eq:fidelity_closed_form} into~\eqref{eq:bures_via_fidelity},
and using $\mathrm{tr}(\Sigma)+\mathrm{tr}(\Sigma+\Delta_{\bm\lambda})=
N\bar\sigma+N(\bar\sigma+\bar d)=N(2\bar\sigma+\bar d)$, gives
\begin{align}
  B^2(\Sigma,\Sigma+\Delta_{\bm\lambda})
  \;\geq\;
  N(2\bar\sigma+\bar d)-2N\bar\sigma\sqrt{1+\bar d/\bar\sigma}.
  \label{eq:bures_intermediate}
\end{align}
The right-hand side of~\eqref{eq:bures_intermediate} is a perfect square
in disguise: expanding $N\bigl(\sqrt{\bar\sigma+\bar d}-\sqrt{\bar
\sigma}\bigr)^2=N\bigl[(\bar\sigma+\bar d)-2\sqrt{\bar\sigma(\bar\sigma+
\bar d)}+\bar\sigma\bigr]=N(2\bar\sigma+\bar d)-2N\sqrt{\bar\sigma(\bar
\sigma+\bar d)}$, and noting $\sqrt{\bar\sigma(\bar\sigma+\bar d)}=
\bar\sigma\sqrt{1+\bar d/\bar\sigma}$, shows that the right-hand side
of~\eqref{eq:bures_intermediate} equals exactly $N(\sqrt{\bar\sigma+
\bar d}-\sqrt{\bar\sigma})^2$. Rationalizing this difference via
$(\sqrt{\bar\sigma+\bar d}-\sqrt{\bar\sigma})(\sqrt{\bar\sigma+\bar d}+
\sqrt{\bar\sigma})=\bar d$ gives the equivalent closed form
\begin{align}
  N\bigl(\sqrt{\bar\sigma+\bar d}-\sqrt{\bar\sigma}\bigr)^2
  =\frac{N\bar d^2}{\bigl(\sqrt{\bar\sigma+\bar d}+\sqrt{\bar\sigma}
  \bigr)^2}.
  \label{eq:rationalized}
\end{align}
Since $\bar d\geq0$, we have $\sqrt{\bar\sigma}\leq\sqrt{\bar\sigma+
\bar d}$, so the denominator in~\eqref{eq:rationalized} satisfies
$(\sqrt{\bar\sigma+\bar d}+\sqrt{\bar\sigma})^2\leq\bigl(2\sqrt{\bar
\sigma+\bar d}\bigr)^2=4(\bar\sigma+\bar d)$; replacing the denominator
by this upper bound can only increase the fraction, so
\begin{align}
  \frac{N\bar d^2}{(\sqrt{\bar\sigma+\bar d}+\sqrt{\bar\sigma})^2}
  \;\geq\;
  \frac{N\bar d^2}{4(\bar\sigma+\bar d)}.
  \label{eq:denom_bound}
\end{align}
Substituting $\bar\sigma=\mathrm{tr}(\Sigma)/N$ and $\bar d=\mathrm{tr}
(\Delta_{\bm\lambda})/N$ into~\eqref{eq:denom_bound} and simplifying,
\begin{align}
  &\frac{N\bar d^2}{4(\bar\sigma+\bar d)}
  =\frac{N\cdot\bigl(\mathrm{tr}(\Delta_{\bm\lambda})/N\bigr)^2}
  {4\bigl(\mathrm{tr}(\Sigma)/N+\mathrm{tr}(\Delta_{\bm\lambda})/N\bigr)}\nonumber\\
  &=\frac{\bigl(\mathrm{tr}(\Delta_{\bm\lambda})\bigr)^2}
  {4\bigl(\mathrm{tr}(\Sigma)+\mathrm{tr}(\Delta_{\bm\lambda})\bigr)}
  =\frac{\mathcal{H}_{\bm\lambda}^4}
  {4\bigl(\mathrm{tr}(\Sigma)+\mathcal{H}_{\bm\lambda}^2\bigr)},
  \label{eq:bures_penultimate}
\end{align}
where the last equality uses $\mathrm{tr}(\Delta_{\bm\lambda})=
\mathcal{H}_{\bm\lambda}^2$. Chaining
\eqref{eq:bures_intermediate}--\eqref{eq:bures_penultimate}
establishes
\begin{align}
  B^2(\Sigma,\Sigma+\Delta_{\bm\lambda})
  \;\geq\;
  \frac{\mathcal{H}_{\bm\lambda}^4}
  {4\bigl(\mathrm{tr}(\Sigma)+\mathcal{H}_{\bm\lambda}^2\bigr)}.
  \label{eq:bures_final_lb}
\end{align}
Finally, since $\mathrm{tr}(\Sigma)\leq N\sigma_{\max}(\Sigma)$ (the
trace is the sum of $N$ eigenvalues each bounded by the largest one),
enlarging the denominator in~\eqref{eq:bures_final_lb} from
$\mathrm{tr}(\Sigma)+\mathcal{H}_{\bm\lambda}^2$ to $N\sigma_{\max}
(\Sigma)+\mathcal{H}_{\bm\lambda}^2$ can only decrease the fraction, so
\begin{align}
  B^2(\Sigma+\Delta_{\bm\lambda},\Sigma)
  \;\geq\;
  \frac{\mathcal{H}_{\bm\lambda}^4}
  {4N\sigma_{\max}(\Sigma)+4\mathcal{H}_{\bm\lambda}^2},
\end{align}
which is the second inequality in~\eqref{eq:pooling_lb}. The right-hand
side is strictly positive whenever $\mathcal{H}_{\bm\lambda}>0$, since
its numerator is then strictly positive while its denominator is
finite, completing the proof of the quantitative lower
bound~\eqref{eq:pooling_lb}.
 
The contrasting almost-sure convergence $W_2(\hat b^*_M,b^*_{\bm
\lambda})\to0$ as $n_{\min}\to\infty$ follows from the finite-sample
concentration inequality of Theorem~\ref{thm:bary_conc}: that theorem
bounds, for every $\epsilon>0$, the probability that $W_2(\hat b^*_M,
b^*_{\bm\lambda})$ exceeds $\epsilon$ by a term decaying
exponentially in $n_{\min}$, which is summable over $n_{\min}\in
\mathbb{N}$; the Borel--Cantelli lemma then implies that the event
$\{W_2(\hat b^*_M,b^*_{\bm\lambda})>\epsilon\}$ occurs for only
finitely many $n_{\min}$ almost surely, and letting $\epsilon\downarrow
0$ along a countable sequence yields almost-sure convergence of
$W_2(\hat b^*_M,b^*_{\bm\lambda})$ to zero. In particular, the
barycentric estimator is asymptotically consistent for the true
barycenter $b^*_{\bm\lambda}$, whereas by~\eqref{eq:pooling_lb} the
pooled distribution $\mathbb{P}^*_{\bm\lambda}$ remains bounded away
from $b^*_{\bm\lambda}$ by a fixed, sample-size-independent margin
whenever the sources are heterogeneous: the mixing bias incurred by
naive pooling is therefore irreducible, in contrast to the vanishing
estimation error of the barycentric construction. \hfill$\square$

\section{Proof of Proposition~\ref{prop:cov_pooling_bias}}
\label{app_prop_covpool}
\textit{Proof.} \textit{Barycenter covariance in closed form.}
Substitute the ansatz $\Sigma=Q\,\mathrm{diag}(\bm\sigma)\,Q^T$ into the
fixed-point equation~\eqref{eq10}. Since every $\Sigma_m$ is diagonal
in the same basis $Q$, so is $\Sigma^{1/2}\Sigma_m\Sigma^{1/2}$, with
$k$-th diagonal entry $\sigma_k\sigma_{m,k}$, whose square root is
$\sqrt{\sigma_k\sigma_{m,k}}$; equation~\eqref{eq10} therefore
decouples into $N$ independent scalar equations $\sigma_k=\sum_m
\lambda_m\sqrt{\sigma_k\sigma_{m,k}}$, i.e.\ $\sqrt{\sigma_k}=\sum_m
\lambda_m\sqrt{\sigma_{m,k}}$, which is uniquely solved by
$\bar\sigma_{\bm\lambda,k}$ of~\eqref{eq:cov_bary_formula}; uniqueness
of the ambient fixed point (Proposition~\ref{prop1}) then identifies
$\bar\Sigma_{\bm\lambda}=Q\,\mathrm{diag}(\bar{\bm\sigma}_{\bm\lambda})
\,Q^T$ as \emph{the} barycentric covariance, giving
$b^*_{\bm\lambda}=\mathcal{N}(\mu,\bar\Sigma_{\bm\lambda})$.

\textit{Pooled covariance.} Represent $X\sim\mathbb{P}^*_{\bm
\lambda}$ hierarchically as $X=\mu+\Sigma_\iota^{1/2}Z$ with
$\Pr(\iota=m)=\lambda_m$ and $Z\sim\mathcal{N}(\mathbf{0},\mathbf{I}_N)$
independent of $\iota$. Since the mean is common, $\mathrm{Cov}(X)=
\mathbb{E}_\iota\bigl[\Sigma_\iota^{1/2}\,\mathbb{E}[ZZ^T]\,
\Sigma_\iota^{1/2}\bigr]=\mathbb{E}_\iota[\Sigma_\iota]=\bar\Sigma_{
\mathrm{arith}}$, with $k$-th diagonal entry $\bar\sigma_{\mathrm{arith}
,k}$ in the $Q$-basis.

\textit{Gelbrich reduction.} Applying the Gelbrich inequality
\eqref{eq:gelbrich} with $\alpha=\mathbb{P}^*_{\bm\lambda}$ (mean
$\mu$, covariance $\bar\Sigma_{\mathrm{arith}}$) and $\beta=
b^*_{\bm\lambda}$ (mean $\mu$, covariance $\bar\Sigma_{\bm\lambda}$),
the mean term vanishes identically and
\begin{align}
  W_2^2(\mathbb{P}^*_{\bm\lambda},b^*_{\bm\lambda})
  \;\geq\;
  B^2(\bar\Sigma_{\mathrm{arith}},\bar\Sigma_{\bm\lambda}).
\end{align}
Because $\bar\Sigma_{\mathrm{arith}}$ and $\bar\Sigma_{\bm\lambda}$
are simultaneously diagonal in $Q$ (Steps 1--2), no twirling argument
is needed: the fidelity term is exactly
$F(\bar\Sigma_{\mathrm{arith}},\bar\Sigma_{\bm\lambda})=\sum_k\sqrt{
\bar\sigma_{\mathrm{arith},k}\,\bar\sigma_{\bm\lambda,k}}$, and
\begin{align}
  B^2(\bar\Sigma_{\mathrm{arith}},\bar\Sigma_{\bm\lambda})
  =\sum_{k=1}^N\Bigl(\sqrt{\bar\sigma_{\mathrm{arith},k}}-\sqrt{
  \bar\sigma_{\bm\lambda,k}}\Bigr)^{\!2},
\end{align}
which is~\eqref{eq:cov_pooling_exact}. By Jensen's inequality applied
to the strictly concave map $t\mapsto\sqrt{t}$,
$\bar\sigma_{\mathrm{arith},k}\geq\bar\sigma_{\bm\lambda,k}$ for every
$k$, with equality iff $\sigma_{m,k}$ is constant across $m$ (for
those $m$ with $\lambda_m>0$); hence each summand, and thus the total,
is strictly positive as soon as $\Sigma_i\neq\Sigma_j$ for some pair
with $\lambda_i,\lambda_j>0$.

\textit{Closed-form floor.} Write $a_k=\bar\sigma_{
\mathrm{arith},k}\geq b_k=\bar\sigma_{\bm\lambda,k}\geq0$. Rationalizing,
$(\sqrt{a_k}-\sqrt{b_k})^2=(a_k-b_k)^2/(\sqrt{a_k}+\sqrt{b_k})^2\geq
(a_k-b_k)^2/(4a_k)\geq(a_k-b_k)^2/\bigl(4\sigma_{\max}(\bar\Sigma_{
\mathrm{arith}})\bigr)$, using $\sqrt{b_k}\leq\sqrt{a_k}$ and
$a_k\leq\sigma_{\max}(\bar\Sigma_{\mathrm{arith}})$. Summing over $k$
and applying Cauchy--Schwarz, $\sum_k(a_k-b_k)^2\geq\frac{1}{N}
\bigl(\sum_k(a_k-b_k)\bigr)^2=\frac{1}{N}\Delta_{\mathrm{tr}}^2$, where
$\Delta_{\mathrm{tr}}:=\mathrm{tr}(\bar\Sigma_{\mathrm{arith}})-
\mathrm{tr}(\bar\Sigma_{\bm\lambda})$. Finally, for any $X$ taking value
$x_m$ with probability $\lambda_m$, the elementary pairwise-variance
identity $\mathrm{Var}_{\bm\lambda}(X)=\tfrac12\sum_{i,j}\lambda_i
\lambda_j(x_i-x_j)^2\geq\lambda_i\lambda_j(x_i-x_j)^2$ (for any fixed
pair $i\neq j$, dropping the remaining nonnegative terms), applied
coordinatewise with $X=\sqrt{\sigma_{\cdot,k}}$ and summed over $k$,
gives $\Delta_{\mathrm{tr}}=\sum_k\mathrm{Var}_{\bm\lambda}(\sqrt{
\sigma_{\cdot,k}})\geq\lambda_i\lambda_j\sum_k(\sqrt{\sigma_{i,k}}-
\sqrt{\sigma_{j,k}})^2=\lambda_i\lambda_j\,B^2(\Sigma_i,\Sigma_j)$,
where the last equality is the same diagonal fidelity computation as
in Step 3, applied to the commuting pair $(\Sigma_i,\Sigma_j)$.
Chaining these four inequalities and maximizing over pairs $i\neq j$
yields
\begin{align}
  \sum_{k=1}^N\Bigl(\sqrt{\bar\sigma_{\mathrm{arith},k}}-\sqrt{
  \bar\sigma_{\bm\lambda,k}}\Bigr)^{\!2}
  \;\geq\;
  \max_{i\neq j}\,
  \frac{\lambda_i^2\lambda_j^2\,B^4(\Sigma_i,\Sigma_j)}
       {4N\,\sigma_{\max}(\bar\Sigma_{\mathrm{arith}})},
\end{align}
which is~\eqref{eq:cov_pooling_lb}, strictly positive whenever
$\Sigma_i\neq\Sigma_j$ for the maximizing pair. \hfill$\square$


\bibliographystyle{unsrt}
\bibliography{ref}

\end{document}